\documentclass[pdflatex,sn-mathphys-ay]{sn-jnl}

\usepackage{graphicx}
\usepackage{multirow}
\usepackage{amsmath,amssymb,amsfonts}
\usepackage{amsthm}
\usepackage{mathrsfs}
\usepackage[title]{appendix}
\usepackage{xcolor}
\usepackage{textcomp}
\usepackage{manyfoot}
\usepackage{booktabs}
\usepackage{algorithm}
\usepackage{algorithmicx}
\usepackage{algpseudocode}
\usepackage{listings}

\usepackage{hyperref}
\usepackage{colortbl}
\usepackage{subcaption}
\usepackage[algo2e, ruled, linesnumbered]{algorithm2e}
\usepackage{bm}
\usepackage[dash,dot]{dashundergaps}
\usepackage{multirow}

\theoremstyle{thmstyleone}
\newtheorem{theorem}{Theorem}

\theoremstyle{thmstyletwo}

\theoremstyle{thmstylethree}

\newtheorem{problem}{Problem}

\newcommand{\Method}{3CPO\xspace}

\begin{document}

\title[Poisson Subspace Clustering: Focusing on the Essentials in Count Data]{Poisson Subspace Clustering: Focusing on the Essentials in Count Data}

\author*[1,2]{\fnm{Collin} \sur{Leiber}}\email{collin.leiber@aalto.fi}

\author[2]{\fnm{Kai} \sur{Puolamäki}}\email{kai.puolamaki@helsinki.fi}

\author[1]{\fnm{Heikki} \sur{Mannila}}\email{heikki.mannila@aalto.fi}

\affil*[1]{\orgname{Aalto University}, \orgaddress{\city{Espoo}, \country{Finland}}}

\affil[2]{\orgname{University of Helsinki}, \orgaddress{\city{Helsinki}, \country{Finland}}}

\abstract{Count data represented as a matrix of non-negative integer values, such as contingency tables, are prevalent across diverse domains. When clustering such data sets, specific methods are required, as generic algorithms often fail to consider their unique distributional properties, leading to unreliable outputs. An effective strategy is to use well-established statistical models such as the Poisson and negative binomial distributions. We present \Method, a clustering algorithm based on statistically solid modeling of count data. In addition to the cluster labels, it identifies a subset of relevant columns, enhancing the interpretability of the results. We propose a simple iterative algorithm that maximizes the posterior probability to find good clustering solutions and discuss its properties. Extensive experiments demonstrate its ability to define high-quality clusters within associated subspaces for various data domains, ranging from gene expressions and texts to economics. Our findings suggest that \Method is a robust solution for clustering count data in a statistically sound and interpretable manner. Our code is available at \url{https://github.com/collinleiber/3CPO}.
}


\maketitle

\section{Introduction}\label{sec:intro}

Count data are a natural and commonly used type of data, e.g., in the form of contingency tables. The idea is that each entry in the data matrix represents the number of occurrences of some event. This type of data is used in a wide variety of research fields, such as biology~\citep{snow2007survey}, ecology~\citep{farnsworth2005statistical}, mobility studies~\citep{ceder1984bus}, environmental studies~\citep{hargesheimer1998particle}, publication studies~\citep{waltman2012leiden}, or economics~\citep{rouwendal2009assessing}. Due to its definition, count data have some unique properties: (1) all values are non-negative natural numbers, (2) all rows share a common domain, and (3) all columns share a common domain. Property (3) contrasts typical tabular data sets, where each column usually describes a different trait. Consider a data set regarding personal information consisting of columns for year of birth, place of birth, height, and gender. These columns are hard to compare as they do not even share a common data type. In contrast, an exemplary count data set could comprise various options, where each entry in the data matrix counts how often a person has chosen a particular option. This information could be used to determine whether someone prefers \textit{option A} or \textit{option B}. 

If one would like to explore certain groups within count data, clustering is a viable choice. Clustering algorithms try to group the rows of a data matrix, using some notion of similarity, so that all rows within a group share a natural relationship, while those in different groups are dissimilar \citep{dubes76clustering}. As an unsupervised machine-learning task, no prior knowledge of the data is required. In general, clustering has already been extensively studied \citep{xu2005survey}, especially in the context of tabular data. 
However, it has been shown that due to the unique characteristics of count data, transformations often perform worse than sophisticated models, such as Poisson and negative binomial distributions \citep{ohara2010not}. This limits the applicability of traditional clustering algorithms that often require pre-processed data to return satisfactory results. Therefore, we argue that specialized approaches are a valuable tool in the data mining toolbox.

As count data is often high-dimensional, e.g., gene expression data or bag-of-words representations of texts, it is beneficial to cluster relevant rows and automatically select columns relevant for the clustering task. In this way, a subsequent analysis of the identified structures can be simplified, reducing the manual work a domain expert has to invest. For tabular data sets, such subspace clustering algorithms have already been studied extensively \citep{sim2013survey}. The selection of relevant columns can be carried out in two different ways: specifically for each cluster or globally for all clusters. In the first variant, individual columns are selected for each cluster so that their characteristics are best described. This is a well-studied field of research \citep{kriegel2012subspace} and is used, e.g, in block-diagonal co-clustering \citep{battaglia2024co}. A disadvantage is that the comparability of individual clusters is more difficult as the clusters are located in different subspaces \citep{goebl2014finding}. For this reason, so-called common subspace clustering has gained momentum in recent years \citep{ding2007adaptive,goebl2014finding,mautz2017towards,bauer2023extension}. Here, a single subspace is defined that contains all clusters.

In this paper, we propose \Method, \textbf{C}lustering and \textbf{C}olumn selection of \textbf{C}ount Data using a \textbf{P}oisson-based \textbf{O}ptimization. \Method uses concepts based on the Poisson distribution to cluster the rows of a data matrix and to simultaneously select the columns that are most relevant for this clustering result. Extensive experiments verify the good performance of \Method with respect to the created clusters. Furthermore, we show that our method can significantly reduce the number of relevant columns. This property simplifies a subsequent interpretation of the clustering results, which is beneficial in real-world unsupervised scenarios. 

Our main contributions are as follows:
\begin{enumerate}
    \item We propose a sophisticated and statistically sound model to describe the entries within a count data matrix using the Poisson distribution.
    \item We introduce \Method, a common subspace clustering algorithm that iteratively maximizes the posterior probability with respect to our proposed model.
    \item Extensive experiments using data sets from diverse domains verify that \Method outperforms its competitors in many scenarios and simplifies interpreting the results by automatically selecting the most informative columns.
\end{enumerate}

\section{Related Work}\label{sec:related}

We see two research areas that are relevant for our proposal: clustering of count data and column selection for clustering.

\subsection{Clustering of Count Data}
\noindent\textbf{Baselines.} An easy way to take into account certain characteristics of count data is to pre-process the data. For example, to handle different magnitudes of row sums, one can consider the relative frequencies of the columns, i.e., divide each value by the row sum. Another option is to compute the revealed comparative advantage (RCA) \citep{balassa1965Trade}, which is often used in economics. An issue with pre-processing count data is discussed by \citet{ohara2010not}, where transformations are shown to be often unsuitable for count data. Therefore, it would be beneficial if the clustering algorithm operated on the raw data. Here, an option is Spherical $k$-Means~\citep{dhillon2001concept} (SKM), a variant of $k$-Means \citep{lloyd82least} that optimizes toward the cosine instead of the Euclidean distance. It follows that the length of a vector does not influence the result, making SKM more suitable for sparse data such as text representations.
Our proposal follows a similar idea by implicitly incorporating pre-processing within the model and therefore simplifying the clustering process.

\noindent{\bf Poisson-based Clustering.} Examples of Poisson-based clustering approaches are PoissonC and PoissonL~\citep{cai2004clustering}, which model the expected values of the Poisson distribution as products of row- and column-specific factors, where each cluster receives its own column values (details in Appendix \ref{appendix:poissonl}). This formulation is optimized using an Expectation Maximization (EM)~\citep{dempster77Maximum} approach, where cluster assignments are updated by considering probabilities (PoissonL) or a Chi-squared test (PoissonC). Note that columns not following the clustering model can heavily influence the final result. To receive more appropriate distances, TransChisq~\citep{kim2007measuring} adds a feature transformation to this formulation by considering all pairs of features. In \citep{rau2011Clustering}, a method based on a Poisson-Mixture-Model is proposed that considers replicates in gene expression data. Using a Poisson-based dissimilarity function, the proposal in \citep{witten2011classification} applies complete linkage to perform hierarchical clustering.

\noindent\textbf{Dirichlet-based Clustering.} Another family of approaches is based on Dirichlet distributions, often used for topic modeling. Latent Dirichlet Allocation~\citep{blei03} adapts a three-level hierarchical Bayesian model that is optimized using an EM-based approach. 
EDCM~\citep{elkan2006clustering}, an approximation of the Dirichlet compound multinomial distribution, improves on traditional multinomial distributions by considering \textit{burstiness}, i.e., the phenomenon that if an entry occurs once, it is likely that it occurs again.
Often, multiple distributional families are combined to better capture the unique characteristics of the data. Examples are the Multinomial Generalized Dirichlet Distribution~\citep{bouguila2008clustering}, a composition of the generalized Dirichlet and multinomial distributions, and Multinomial Beta-Liouville Mixture~\citep{bouguila2010count}, a combination of the multinomial and the Beta-Liouville distributions. Exponential approximation to the Multinomial Beta-Liouville \citep{zamzami2020high}~is another case that can simultaneously fit a model and select the optimal number of clusters by applying an extended EM procedure. 
In contrast to these methods, our approach only considers a simple Poisson distribution to yield a hard clustering solution.

\subsection{Column Selection for Clustering}

While some research has already been done that combines clustering of count data with column selection, to the best of our knowledge, this has been focused on \textbf{co-clustering}, i.e., simultaneously clustering the rows and columns of a data matrix. This can, for example, be achieved by using information-theoretic approaches such as CROINFO~\citep{govaert2018mutual} which optimizes a loss function based on mutual information. Algorithms like CoclustMod~\citep{ailem2015co} and CoclustSpecMod~\citep{labiod2011co} interpret co-clustering as a graph modularity problem. Other methods, e.g., \citep{govaert2010latent,ailem2017sparse,ailem2017modelbased}, use the Poisson distribution in the form of Poisson Latent Block Models \citep{nadif05comparison}. This idea has been generalized by ELBM and SELBM~\citep{hoseinipour2025sparse} to consider (sparse) block models from various distributions of the exponential family. TauCC~\citep{battaglia2024fast} simplifies the complexity of choosing hyperparameters in co-clustering by automatically identifying an appropriate number of clusters. In our work, the goal is not to analyze multiple column sets containing different clustering characterizations but to filter out columns not helpful for the clustering task and, thus, identify a single set of columns that is relevant for all clusters. This allows not only the analysis of the intra-cluster but also the inter-cluster relationship~\citep{goebl2014finding}. 

\textbf{Common subspace} algorithms pursue a similar goal, by defining a single relevant subspace, often by combining a feature transformation with column selection, e.g., by applying Singular Value or QR Decomposition~\citep{ding2002adaptive}, Linear Discriminant Analysis~\citep{ding2007adaptive}, Givens rotations~\citep{goebl2014finding}, specialized eigenvalue decompositions~\citep{mautz2017towards} or modality-based projection pursuit~\citep{bauer2023extension}. The actual clustering is then performed in the resulting subspaces, often using variants of $k$-Means or the EM algorithm, and influences the feature transformation in an iterative manner. Yet, if one wants to receive a lower-dimensional and interpretable representation of count data, these approaches are less practical, as the applied rotations lead to mixtures of features that cannot be easily interpreted. Furthermore, we have no guarantee that the resulting features are integers, so statistical foundations based on count data like the Poisson distribution cannot be applied. A similar problem is described by \cite{klein2023ksubmix} for clustering mixed-type data in a common subspace setting. Their proposal identifies relevant features by comparing the cost function with respect to a so-called clustered space to a noise space that contains a single cluster. \cite{zamzami2023novel} identify the number of clusters and the most important features for count data by combining the generalized Dirichlet multinomial with the minimum message length criterion.

\section{Definitions, Theory, and Algorithms}\label{sec:theory}

\begin{figure}[t]
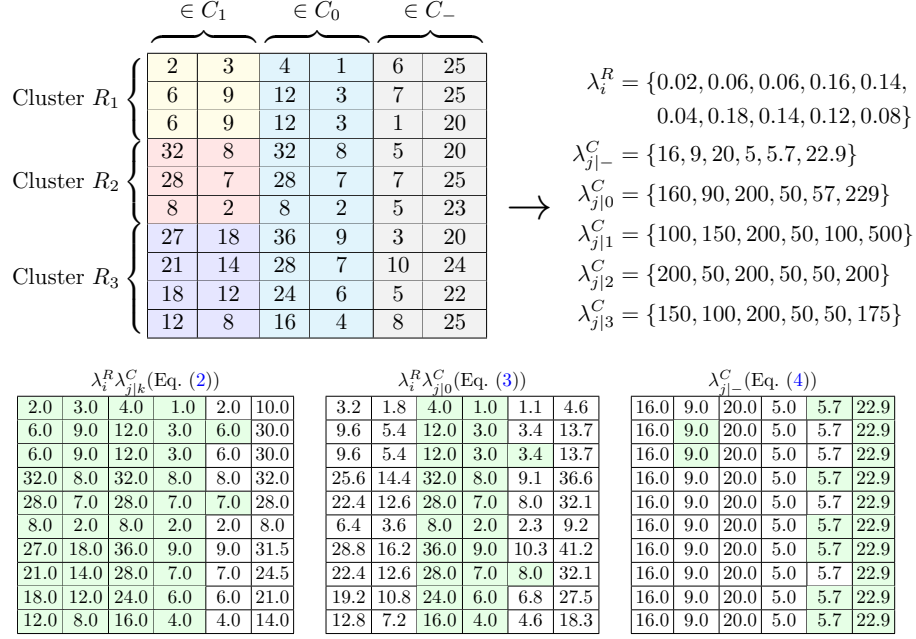

\centering
\resizebox{0.92\textwidth}{!}{
$\displaystyle
\begin{array}{@{} r @{}}
    \\\\
    \text{Cluster $R_1$}\left\{\begin{array}{@{}c@{}}\null\\\null\\\null\end{array}\right.\\
    \text{Cluster $R_2$}\left\{\begin{array}{@{}c@{}}\null\\\null\\\null\end{array}\right.\\
    \text{Cluster $R_3$}\left\{\begin{array}{@{}c@{}}\null\\\null\\\null\\\null\end{array}\right.\\
\end{array}
\begin{array}{| c | c | c | c | c | c |}
\multicolumn{2}{c}{\in C_1}&\multicolumn{2}{c}{\in C_0}&\multicolumn{2}{c}{\in C_{-}}\\
\multicolumn{2}{c}{\overbrace{\rule{1.6cm}{0pt}}} & \multicolumn{2}{c}{\overbrace{\rule{1.6cm}{0pt}}} & \multicolumn{2}{c}{\overbrace{\rule{1.6cm}{0pt}}}\\
\hline
~\hspace{0.1cm}\cellcolor{yellow!10} 2 \hspace{0.1cm}~&~\hspace{0.1cm} \cellcolor{yellow!10} 3 \hspace{0.1cm}~&~\hspace{0.1cm} \cellcolor{cyan!10} 4 \hspace{0.1cm}~&~\hspace{0.1cm} \cellcolor{cyan!10} 1 \hspace{0.1cm}~&~\hspace{0.1cm} \cellcolor{gray!10} 6 \hspace{0.1cm}~&~ \cellcolor{gray!10} 25~\\
\hline
\cellcolor{yellow!10} 6 & \cellcolor{yellow!10} 9 & \cellcolor{cyan!10} 12 & \cellcolor{cyan!10} 3 & \cellcolor{gray!10} 7 & \cellcolor{gray!10} 25\\
\hline
\cellcolor{yellow!10} 6 & \cellcolor{yellow!10} 9 & \cellcolor{cyan!10} 12 & \cellcolor{cyan!10} 3 & \cellcolor{gray!10} 1 & \cellcolor{gray!10} 20\\
\hline
\cellcolor{red!10} 32 & \cellcolor{red!10} 8 & \cellcolor{cyan!10} 32 & \cellcolor{cyan!10} 8 & \cellcolor{gray!10} 5 & \cellcolor{gray!10} 20\\
\hline
\cellcolor{red!10} 28 & \cellcolor{red!10} 7 & \cellcolor{cyan!10} 28 & \cellcolor{cyan!10} 7 & \cellcolor{gray!10} 7 & \cellcolor{gray!10} 25\\
\hline
\cellcolor{red!10} 8 & \cellcolor{red!10} 2 & \cellcolor{cyan!10} 8 & \cellcolor{cyan!10} 2 & \cellcolor{gray!10} 5 & \cellcolor{gray!10} 23\\
\hline
\cellcolor{blue!10} 27 & \cellcolor{blue!10} 18 & \cellcolor{cyan!10} 36 & \cellcolor{cyan!10} 9 & \cellcolor{gray!10} 3 & \cellcolor{gray!10} 20\\
\hline
\cellcolor{blue!10} 21 & \cellcolor{blue!10} 14 & \cellcolor{cyan!10} 28 & \cellcolor{cyan!10} 7 & \cellcolor{gray!10} 10 & \cellcolor{gray!10} 24\\
\hline
\cellcolor{blue!10} 18 & \cellcolor{blue!10} 12 & \cellcolor{cyan!10} 24 & \cellcolor{cyan!10} 6 & \cellcolor{gray!10} 5 & \cellcolor{gray!10} 22\\
\hline
\cellcolor{blue!10} 12 & \cellcolor{blue!10} 8 & \cellcolor{cyan!10} 16 & \cellcolor{cyan!10} 4 & \cellcolor{gray!10} 8 & \cellcolor{gray!10} 25\\
\hline
\end{array}
\begin{minipage}{0.1\textwidth}
\vspace{1.3cm}
\begin{center}
    \scalebox{2}{$\rightarrow$}
\end{center}
\end{minipage}
\begin{minipage}{0.2\textwidth}
\vspace{1cm}
    \begin{align*}
        \lambda_i^R = \{&0.02, 0.06, 0.06, 0.16, 0.14, \\
        &0.04, 0.18, 0.14, 0.12, 0.08\}\\
        \lambda^C_{j|-} = \{&16, 9, 20, 5, 5.7, 22.9\}\\
        \lambda^C_{j|0} = \{&160, 90, 200, 50, 57, 229\}\\
        \lambda^C_{j|1}=\{&100, 150, 200, 50, 100, 500\}\\
        \lambda^C_{j|2}=\{&200, 50, 200, 50, 50, 200\}\\
        \lambda^C_{j|3}=\{&150, 100, 200, 50, 50, 175\}\\
    \end{align*}
\end{minipage}
\vspace{-0.2cm}
$}
\resizebox{!}{0.093\textheight}{$
\begin{array}{| c | c | c | c | c | c |}
\multicolumn{6}{c}{\lambda_i^R\lambda_{j|k}^C $(Eq. \eqref{eq:mui})$}\\
\hline
~\cellcolor{green!10} 2.0 ~&~ \cellcolor{green!10} 3.0 ~&~ \cellcolor{green!10} 4.0 ~&~~ \cellcolor{green!10} 1.0 ~&~ 2.0 ~& 10.0 \\
\hline
\cellcolor{green!10} 6.0 & \cellcolor{green!10} 9.0 & \cellcolor{green!10} 12.0 & \cellcolor{green!10} 3.0 & \cellcolor{green!10}6.0 & 30.0 \\
\hline
\cellcolor{green!10} 6.0 & \cellcolor{green!10} 9.0 & \cellcolor{green!10} 12.0 & \cellcolor{green!10} 3.0 & 6.0 & 30.0 \\
\hline
\cellcolor{green!10} 32.0 & \cellcolor{green!10} 8.0 & \cellcolor{green!10} 32.0 & \cellcolor{green!10} 8.0 & 8.0 & 32.0 \\
\hline
\cellcolor{green!10} 28.0 & \cellcolor{green!10} 7.0 & \cellcolor{green!10} 28.0 & \cellcolor{green!10}7.0 & \cellcolor{green!10}7.0 & 28.0 \\
\hline
\cellcolor{green!10} 8.0 & \cellcolor{green!10} 2.0 & \cellcolor{green!10} 8.0 & \cellcolor{green!10} 2.0 & 2.0 & 8.0 \\
\hline
\cellcolor{green!10} 27.0 & \cellcolor{green!10} 18.0 & \cellcolor{green!10} 36.0 & \cellcolor{green!10} 9.0 & 9.0 & 31.5 \\
\hline
\cellcolor{green!10} 21.0 & \cellcolor{green!10} 14.0 & \cellcolor{green!10} 28.0 & \cellcolor{green!10} 7.0 & 7.0 & 24.5 \\
\hline
\cellcolor{green!10} 18.0 & \cellcolor{green!10} 12.0 & \cellcolor{green!10} 24.0 & \cellcolor{green!10} 6.0 & 6.0 & 21.0 \\
\hline
\cellcolor{green!10} 12.0 & \cellcolor{green!10} 8.0 & \cellcolor{green!10} 16.0 & \cellcolor{green!10} 4.0 & 4.0 & 14.0 \\
\hline
\end{array}$
}
\hspace{0.1cm}
\resizebox{!}{0.093\textheight}{$
\begin{array}{| c | c | c | c | c | c |}
\multicolumn{6}{c}{\lambda^R_i \lambda^C_{j|0} $(Eq. \eqref{eq:mu0})$}\\
\hline
~3.2~ & ~1.8~ & ~\cellcolor{green!10} 4.0~ &~ \cellcolor{green!10}1.0~ & ~1.1~ & ~4.6~ \\
\hline
9.6 & 5.4 & \cellcolor{green!10} 12.0 & \cellcolor{green!10} 3.0 & 3.4 & 13.7 \\
\hline
9.6 & 5.4 & \cellcolor{green!10} 12.0 & \cellcolor{green!10} 3.0 & \cellcolor{green!10}3.4 & 13.7 \\
\hline
25.6 & 14.4 & \cellcolor{green!10} 32.0 & \cellcolor{green!10} 8.0 & 9.1 & 36.6 \\
\hline
22.4 & 12.6 & \cellcolor{green!10} 28.0 & \cellcolor{green!10} 7.0 & 8.0 & 32.1 \\
\hline
6.4 & 3.6 & \cellcolor{green!10} 8.0 & \cellcolor{green!10} 2.0 & 2.3 & 9.2 \\
\hline
28.8 & 16.2 & \cellcolor{green!10} 36.0 & \cellcolor{green!10} 9.0 & 10.3 & 41.2 \\
\hline
22.4 & 12.6 & \cellcolor{green!10} 28.0 & \cellcolor{green!10} 7.0 & \cellcolor{green!10}8.0 & 32.1 \\
\hline
19.2 & 10.8 & \cellcolor{green!10} 24.0 & \cellcolor{green!10} 6.0 & 6.8 & 27.5 \\
\hline
12.8 & 7.2 & \cellcolor{green!10} 16.0 & \cellcolor{green!10} 4.0 & 4.6 & 18.3 \\
\hline
\end{array}$
}
\hspace{0.1cm}
\resizebox{!}{0.093\textheight}{$
\begin{array}{| c | c | c | c | c | c |}
\multicolumn{6}{c}{\lambda^C_{j|-} $(Eq. \eqref{eq:muminus})$}\\
\hline
16.0 &~9.0~&20.0&~5.0~&~\cellcolor{green!10}5.7~&\cellcolor{green!10}22.9 \\
\hline
16.0&\cellcolor{green!10}9.0&20.0&5.0&5.7&\cellcolor{green!10}22.9\\
\hline
16.0&\cellcolor{green!10}9.0&20.0&5.0&5.7&\cellcolor{green!10}22.9\\
\hline
16.0&9.0&20.0&5.0&\cellcolor{green!10}5.7&\cellcolor{green!10}22.9\\
\hline
16.0&9.0&20.0&5.0&5.7&\cellcolor{green!10}22.9\\
\hline
16.0&9.0&20.0&5.0&\cellcolor{green!10}5.7&\cellcolor{green!10}22.9\\
\hline
16.0&9.0&20.0&5.0&\cellcolor{green!10}5.7&\cellcolor{green!10}22.9\\
\hline
16.0&9.0& 20.0&5.0&5.7&\cellcolor{green!10}22.9\\
\hline
16.0&9.0&20.0&5.0&\cellcolor{green!10}5.7&\cellcolor{green!10}22.9\\
\hline
16.0&9.0&20.0&5.0&\cellcolor{green!10}5.7&\cellcolor{green!10}22.9\\
\hline
\end{array}$
}

\caption{\textbf{Motivation.} The upper matrix shows a data set ${\bf X} \in {\mathbb{N}}^{n\times m}$, where $n=10, m=6$, containing three clusters. Only the first two columns are relevant for clustering (highlighted in yellow, red and dark blue) and therefore belong to $C_1$. Columns three and four (light blue) 
equally contribute to all clusters and belong to $C_0$. Columns five and six (gray) contain (uniformly distributed) noise and belong to $C_{-}$. The lower matrices show the expected values when using $\lambda^C_{j|-}, \lambda^C_{j|0}$ and $\lambda^C_{j|k}$, where the green color highlights the closest expected value compared to ${\bf X}$ among those three matrices.}\label{fig:exampleMatrix}
\end{figure}

In this work, the main goal is to partition the $n$ rows of a matrix ${\bf Y}\in{\mathbb{N}}_0^{n\times m}$ consisting of non-negative integer counts, into $K \in \mathbb{N}_{>0}$ clusters $R_k\subseteq [n]=\{1,\ldots,n\}$, where $k\in[K]$, such that $\cup_{k=1}^K{R_k}=[n]$ and $R_k\cap R_l=\emptyset$ if $k\ne l$.

More formally, we define the probability of ${\bf Y}_{ij}$ by using the Poisson distribution as
\begin{equation}
    p({\bf Y}_{ij},\mu_{ij})
    =\text{Pois}({\bf Y}_{ij}\mid\mu_{ij})
   \propto \mu_{ij}^{{\bf Y}_{ij}}e^{-\mu_{ij}}.
\end{equation}
In our experiments, we use ${\bf X}_{ij}={\bf Y}_{ij}+10^{-3}$ for numerical stability in place of ${\bf Y}_{ij}$, which can be interpreted as a weak Gamma prior.\footnote{
The full expression with the prior would be: $p({\bf Y}_{ij},\mu_{ij})
  =\text{Pois}({\bf Y}_{ij}\mid\mu_{ij})\times\text{Gamma}(\mu_{ij}\mid\alpha,\beta)    \propto \mu_{ij}^{{\bf Y}_{ij}+\alpha-1}e^{-\mu_{ij}}
    \propto \mu_{ij}^{{\bf X}_{ij}}e^{-\mu_{ij}}$ with $\alpha=1+10^{-3}$ and the second Gamma parameter being $\beta=0$.
}

Following the proposal in, e.g., \citep{cai2004clustering}, we model the expected number of occurrences $\mu_{ij}$ as a product of row- and column-specific factors. This leads to two main assumptions for the data matrix $\bf X$: (i) entries within row $i$ are scaled by the row-specific value $\lambda^R_i$ and (ii) entries in cluster $k$ follow the proportions as defined by the column and cluster-specific values $\lambda^C_{j|k}$ that play a role similar to ``cluster centroids'' and are ideally different from those of other clusters.

These conditions are quite strict and are often not satisfied in practice. Let us consider a hypothetical matrix, where each row contains the word counts of a letter. In most cases, greeting phrases are only contained once, regardless of the length of the letter. This contradicts Assumption~(i). In contrast, accompanying words such as ‘the’ usually occur more frequently in longer documents and fulfill Assumption~(i). However, their number is probably less dependent on the content (i.e., cluster), which contradicts Assumption~(ii).

To address these problems, we divide the $m$ columns of the data matrix into three non-overlapping subsets $C_-$, $C_0$, and $C_1$, with $C_- \cup C_0 \cup C_1=[m]$; see Fig.~\ref{fig:exampleMatrix}. 
The subset $C_1$ contains those columns that are important for clustering and for which the overall scale of the row matters. For columns $j \in C_1$ the Poisson parameter reads
\begin{equation}\label{eq:mui}
\mu_{ij} = \lambda^R_i\lambda^C_{j\mid k},~~~{\rm if}~~~j\in C_1 ~{\rm and}~i \in R_k,
\end{equation}
where $\lambda^R_i\in{\mathbb{R}}_{>0}$ are the row-specific and $\lambda^C_{j\mid k}\in{\mathbb{R}}_{>0}$ for $k\in[K]$ the column-specific parameters.
The subset $C_0$ contains those columns where the overall scale does have an effect, but the clusters behave in a similar way (violating Assumption (ii)). For columns $j \in C_0$ we have 
\begin{equation}\label{eq:mu0}
\mu_{ij} = \lambda^R_i\lambda^C_{j\mid 0},~~~{\rm if}~~~j\in C_0.
\end{equation}
Finally, the subset $C_-$ contains those columns where neither the scale of the row nor the cluster is informative (violating Assumptions (i) and (ii)). The values in columns $j\in C_-$ are modeled by a column-specific parameter:
\begin{equation}\label{eq:muminus}
\mu_{ij} = \lambda^C_{j\mid -},~~~{\rm if}~~~j\in C_-.
\end{equation}

Since we assume that all entries within a matrix ${\bf X}$ have been drawn independently, the log-loss of ${\bf X}$ is given by
\begin{align}
\label{eq:likelihood}
{\cal L}({\bf X})=
-\sum\nolimits_{i = 1}^n{\sum\nolimits_{j=1}^m{\log{p({\bf X}_{ij},\mu_{ij})}}}+\sum_{j\in C_1}{b({\bf X}_{\cdot j}|K)},
\end{align}
where $b({\bf X}_{\cdot j}|K)\in{\mathbb{R}}_{\ge 0}$ is a penalty term related to the column selection, defined in Sect. \ref{sec:column}. This loss function is scale invariant if $b({\bf X}_{\cdot j}|K)=0$, i.e., for all $\alpha \in {\mathbb{R}}_{> 0}$ we have ${\cal L}({\bf X})={\cal L}({\bf \alpha X})$. For the details, see Appendix \ref{sec:scale_invariance_appendix}.

\noindent\fbox{\begin{minipage}{\textwidth}
\noindent \textbf{Our main computational problem is as follows}.
\vspace{-0.35cm}
\begin{problem}\label{prob:main}
Given a $n\times m$ matrix of counts ${\bf X}$, the number of clusters $K$, and the definitions given in Eq.~\eqref{eq:mui}-\eqref{eq:muminus}, find
the parameters $\{\lambda^R_i,\lambda^C_{j|\cdot}\}$,
the clustering $R_1,\ldots,R_K$ and the column partitions $C_-, C_0, C_1$ that minimize the log-loss given by Eq. \eqref{eq:likelihood}.
\vspace{0.1cm}
\end{problem}
\end{minipage}}\\

We solve Problem \ref{prob:main} using an Expectation Maximization (EM)-type algorithm, where we iterate finding (i) the optimal expected values $\mu_{ij}$ (Sect. \ref{sec:em1}), (ii) the clustering $R_1,\ldots,R_K$ (Sect. \ref{sec:em2}), and (iii) the column partition $C_-,C_0,C_1$ (Sect. \ref{sec:column}), that at each step maximize the likelihood. This is repeated until convergence. We describe the initialization of $C_-,C_0,C_1$ and $R_1,\ldots,R_K$ in Sect. \ref{sec:initial}. A pseudocode version of our method is given in Algorithm \ref{algo:algorithm}. Detailed derivations can be found in Appendix \ref{app:a}.

\subsection{Finding the Expected Values $\mu_{ij}$}\label{sec:em1}
Following the definitions in Eq.~\eqref{eq:mui}--\eqref{eq:muminus}, we can find the expected values $\mu_{ij}$ by determining the parameters $\lambda_i^R, \lambda_{j|-}^C, \lambda_{j|0}^C$ and $\lambda_{j\mid k}^C$. 
The update rules are obtained by solving the equations $\partial{{\cal L}({\bf X})}/\partial \lambda^R_i=0$ and $\partial{{\cal L}({\bf X})}/\partial \lambda^C_{j\mid l}=0$ for $l \in \{-, 0\} \cup [K]$. 
Notice that for any constant $\kappa\in{\mathbb{R}}_{>0}$, we can multiply all $\lambda^R_i\leftarrow \kappa \lambda^R_i $ and divide all $\lambda^C_{j\mid l}\leftarrow \lambda^C_{j\mid l}/\kappa $ for $l\in\{0\}\cup [K]$ without affecting the parameters $\mu_{ij}$. For this reason, we will set, without loss of generality, $\sum\nolimits_{i=1}^n{\lambda^R_i}=1$. 
Given a clustering $R_1,\dots,R_K$ and a column partition $C_1, C_0, C_-$, the solutions for the row-specific parameters are
\begin{align}\label{eq:lambdaR}
\lambda^R_i=\begin{cases}
\frac{(\sum_{j\in C_0 \cup C_1} {\bf X}_{ij}) (\sum_{i \in R_k} \sum_{j\in C_0} {\bf X}_{ij})}{(\sum_{i=1}^n \sum_{j \in C_0} {\bf X}_{ij})(\sum_{i \in R_k} \sum_{j \in C_1 \cup C_0} {\bf X}_{ij})},&~~~{\rm if}~~~ i \in R_k~{\rm and}~C_0 \neq \emptyset\\
\frac{\sum_{j \in C_1} {\bf X}_{ij}}{\sum_{i=1}^n \sum_{j \in C_1} {\bf X}_{ij}},&~~~{\rm if}~~~ i \in R_k~{\rm and}~C_0 = \emptyset
\end{cases}
\end{align}
(for details, see Appendix \ref{sec:3cpo_appendix}). For the column parameters, we obtain: 
\begin{align}
\lambda^C_{j\mid -}&=\sum\nolimits_{i=1}^n{{\bf X}_{ij}/n},\label{eq:meanminus}\\
\lambda^C_{j\mid 0}&=\sum\nolimits_{i=1}^n{{\bf X}_{ij}},\label{eq:mean0}\\
\lambda^C_{j\mid k}&=\begin{cases}
    \frac{(\sum_{i \in R_k} {\bf X}_{ij}) (\sum_{i=1}^n \sum_{j \in C_0} {\bf X}_{ij})}{\sum_{i \in R_k} \sum_{j\in C_0} {\bf X}_{ij}},&~~~{\rm if}~~~ C_0 \neq \emptyset\\
    \frac{(\sum_{i \in R_k} {\bf X}_{ij}) (\sum_{i=1}^n \sum_{j \in C_1} {\bf X}_{ij})}{\sum_{i \in R_k} \sum_{j\in C_1} {\bf X}_{ij}},&~~~{\rm if}~~~ C_0 = \emptyset.
\end{cases}\label{eq:mean1}
\end{align}
We will compute and store the column parameters for all columns to update the column partitions later in Sect. \ref{sec:column}.

\subsection{Finding the Clustering $R_1,\ldots,R_K$}\label{sec:em2}

Similarly, the update rule for the clustering $R_1,\ldots,R_K$ that maximizes the likelihood when all other parameters are kept fixed is given by computing the score 
\begin{equation}\label{eq:Sik}
S_{i}^k=\sum\nolimits_{j\in C_1}{\left(
{\bf X}_{ij}\log{\hat{\mu}_{ij}^k}-\hat{\mu}_{ij}^k
\right)},
\end{equation}
where $\hat{\mu}_{ij}^k=\frac{(\sum_{i \in \hat{R}_k} {\bf X}_{ij})(\sum_{j\in C_1 \cup C_0} {\bf X}_{ij})}{\sum_{i \in \hat{R}_k} \sum_{j \in C_1 \cup C_0} {\bf X}_{ij}}$ and $\hat{R}_k$ are the cluster assignments from the last iteration. 
Finally, we assign each row $i\in [n]$ to the cluster $k\in[K]$ with the highest score $S_{i}^k$, i.e., 
$R_k=\left\{
i\in[n]\mid
k=\arg\max\nolimits_{k'\in [K]}{S_{ik'}}
\right\}$.

\subsection{Finding the Column Partitions $C_-,C_0,C_1$}\label{sec:column}

To find the column partitions, we again maximize the likelihood, keeping all other parameters fixed. To this end, we compute the contributions to the log-likelihood $L_j(g)$ for each column $j\in[m]$ and partition $g\in\{-,0,1\}$. We obtain:
\begin{align}
L_j(-)=&\sum\nolimits_{i=1}^n{\left({\bf X}_{ij}\log{\lambda^C_{j\mid -}}-\lambda^C_{j\mid -}\right)},\label{eq:Lminus}\\
L_j(0)=&\sum\nolimits_{i=1}^n{\left({\bf X}_{ij}\log{\left(\lambda^R_i\lambda^C_{j\mid 0}\right)}-\lambda^R_i\lambda^C_{j\mid 0}\right)},\label{eq:L0}\\
L_j(1)=&\sum\nolimits_{k=1}^K{\sum\nolimits_{i\in R_k}{\left({\bf X}_{ij}\log{\left(\lambda^R_i\lambda^C_{j\mid k}\right)}-\lambda^R_i\lambda^C_{j\mid k}\right)}} -b({\bf X}_{\cdot j}|K).\label{eq:L1}
\end{align}
Here, we added a penalty term motivated by the Minimum Description Length (MDL) \citep{risannen1978modeling} (see Eq. \eqref{eq:likelihood}) defined by
\begin{equation}\label{eq:Bx}
b({\bf x}|K)=\text{max}\left(0,
(K-1)\log{\sum\nolimits_{i=1}^n{\bf x}_i}\right).
\end{equation}
The intuition is that the additional information to describe parameters $\lambda^C_{j\mid k}$ requires $K-1$ cluster-specific column sums (knowing the total column sum gives the $K$-th sum), each of which can be described by $\log{\sum\nolimits_{i=1}^n{\bf X}_{ij}}$ bits or less. Other options for the penalty term, such as the Bayesian Information Criterion (BIC) \citep{schwarz1978estimating}, i.e., $b({\bf x}|K)=\frac{K}{2}\log n$, are also applicable.
To solve Problem \ref{prob:main} we simply assign column $j$ to group $C_{g}$ such that $g=\arg\max\nolimits_{g'\in\{-,0,1\}}{L_j(g')}$.

\subsection{Distance Between the Rows}\label{sec:distance}

In many practical applications, it may be helpful to have a distance measure between the rows. In particular, we use such a distance measure to obtain a better initialization of the clusters (see Sect. \ref{sec:initial}).

A natural way to define a distance between items $a,b\in[n]$ is to consider a $2\times m$ submatrix containing only these two rows. The distance can then be computed as the difference in losses of Eq. \eqref{eq:likelihood} (assuming $C_1 = [m]$, $C_-=C_0=\emptyset$) between both rows being in the same cluster ($R_1=\{a,b\}$) and in individual clusters ($R_1=\{a\}$, $R_2=\{b\}$). Doing the computations using Eqs. \eqref{eq:lambdaR}--\eqref{eq:mean1} results in the distance
\begin{align}\label{eq:distance}
d(a,b)=
\sum\nolimits_{i\in\{a,b\}}{\sum\nolimits_{j=1}^m{
{\bf X}_{ij}\log{\left({\bf X}_{ij}N_{ab}/\left(r_ic_{j}\right)\right)}}},
\end{align}
where we used $r_i=\sum\nolimits_{j=1}^m{{\bf X}_{ij}}$, $c_j={\bf X}_{aj}+{\bf X}_{bj}$, and $N_{ab}=r_a+r_b$ (details are provided in Appendix \ref{sec:distance_function_appendix}). The distance is, to a constant factor, the classical likelihood–ratio test statistic for independence in $2\times m$ contingency tables. It is always non-negative and zero only if ${\bf X}_{ij}=r_ic_j/N_{ab}$ for all $i\in \{a,b\}$ and $j\in[m]$; in particular, the distance is zero if there exists $\kappa\in{\mathbb{R}}_{>0}$ such that ${\bf X}_{aj}=\kappa\,{\bf X}_{bj}$ for all $j\in[m]$.

\begin{algorithm2e}[t!]
    \scriptsize
	\SetAlgoVlined
	\DontPrintSemicolon
	\KwIn{data set ${\bf X}$, number of clusters $K$}
	\KwOut{cluster assignments $R_k$, column values $\lambda^C_{j|k}$, column partitions $C_1, C_0, C_-$}
    // Initialization (Sect. \ref{sec:initial})\;
    $\lambda^C_{j|-} \gets \sum_{i=1}^n{\bf X}_{ij}/n \quad$ and $ \quad \lambda^C_{j|0} \gets \sum_{i=1}^n{\bf X}_{ij}$\;
    $L_j(-) \gets \sum_{i=1}^n({\bf X}_{ij} \log \lambda^C_{j|-}-\lambda^C_{j|-})$\;
    $C_1 \gets$ assign the $\lfloor \frac{m}{2} \rfloor$ columns with the hightest $h(j)$ values (Eq. \eqref{eq:columnInfo}) to $C_1$\;
    $C_- \gets [m]\setminus C_1 \quad $ and $\quad C_0 \gets \emptyset$\;
    $R_1, \dots, R_K \gets $ initial cluster assignments by $k$-Means++ with $d(a,b)$ (Eq. \eqref{eq:distance}) in $C_1$\;
    \While{$C_1$, $C_0$ or any $R_k$ changed in last iteration}{
        $\lambda_i^R, \lambda_{j|k}^C \gets$ update lambdas (Sect. \ref{sec:em1})\;
        // Update column partitions (Sect. \ref{sec:column})\;
        \For{$j \in [m]$}{
            $L_j(0) \gets \sum_{i=1}^n({\bf X}_{ij} \log(\lambda^R_i\lambda^C_{j|0})-\lambda^R_i\lambda^C_{j|0})$\;
            $L_j(1) \gets \sum_{k=1}^K \sum_{i \in R_k}({\bf X}_{ij} \log (\lambda^R_i\lambda^C_{j|k})-\lambda^R_i\lambda^C_{j|k})-b({\bf X}_{\cdot j}|K)$\;
        }
        $C_- \gets \{j \mid j \in [m] \wedge L_j(-) \ge \text{max}(L_j(0), L_j(1))\}$\;
        $C_0 \gets \{j \mid j \in [m] \wedge j \notin C_- \wedge L_j(0) \ge L_j(1)\}$\;
        $C_1 \gets [m] \setminus (C_- \cup C_0)$\;
        // Update cluster assignments (Sect. \ref{sec:em2})\;
        $R_1, \dots, R_K \gets \emptyset$\;
        \For{$i \in [n]$}{
            $k \gets \underset{k'\in [K]}{\mathrm{argmax}}~ {\sum_{j \in C_1}{ {\bf X}_{ij} \log (\hat{\mu}_{ij}^k)  - \hat{\mu}_{ij}^k}}$ (Eq. \eqref{eq:Sik})\;
            $R_k \gets R_k \cup \{i\}$\;
        }
    }
	\Return{$R_k, \lambda_{j|k}^C, C_1,C_0,C_-$}
	\caption{The \Method algorithm (without outlier detection)}
 \label{algo:algorithm}
\end{algorithm2e}

\subsection{Finding the Initial Values}\label{sec:initial}

Good initial values result in faster convergence to a better (local) optimum.

\textbf{Initial $C_-,C_0,C_1$.}
As an initial guess, we set $C_0=\emptyset$. 
Next, we use the following heuristic, motivated by Eq.~\eqref{eq:distance}, to measure the potential clustering information $h(j)$ within each column
\begin{align}\label{eq:columnInfo}
h(j)=
\left(\sum\nolimits_{i=1}^n{ 
{\bf X}_{ij}  \log{\left({\bf X}_{ij}N/\left(r_ic_{j}\right)\right)}} \right)/c_j,
\end{align}
where  $c_j=\sum\nolimits_{i=1}^n{\bf X}_{ij}$ and $N=\sum\nolimits_{j=1}^m c_j$\footnote{This formulation is equivalent to the Kullback-Leibler divergence with $p={\bf X}_{ij}/c_j$ and $q=r_i/N$.}. $C_1$ is then defined as the $\lfloor \frac{m}{2} \rfloor$ columns with the highest $h(j)$ values. We divide by $c_j$ to avoid placing disproportionate weight on columns containing high values. Lastly, we set $C_- = [m] \setminus C_1$. Alternatively, one can randomly assign the columns in $[m]$ to $C_-$ and $C_1$, so that $C_- \cap C_1 = \emptyset$.

\textbf{Initial $R_1,\ldots,R_K$.}
Using the $k$-Means++ \citep{Arthur07kmeans} seeding mechanism within $C_1$ with $d(a,b)$ as defined by Eq. \eqref{eq:distance} in place of the Euclidean distance gives us the $K$ rows ${\bf Z}\in {\mathbb{R}_{>0}}^{K\times |C_1|}$. Afterward, we receive the initial cluster assignments $R_1,\ldots,R_K$ by assigning each row in ${\bf X}$ to its best matching row within ${\bf Z}$ according to Eq.~\eqref{eq:Sik} with $\hat{R_k}=\{{\bf Z}_k\} \Rightarrow \hat{\mu}_{ij}^k= ( \sum_{j\in C_1} {\bf X} _{ij} ) ({\bf Z}_{kj} / \sum_{j \in C_1} {\bf Z}_{kj} ) $.

\subsection{Identifying Outliers}
\label{sec:outlier_detection}

In many real-world applications, data sets contain rows that deviate from the clustering model assumptions, potentially degrading clustering quality. To address this, we propose an outlier detection mechanism integrated within the \Method framework that intends to demonstrate the expandability of our method. Other outlier detection approaches are also conceivable and can be easily integrated. 

The core idea is to model outliers using a data set-wide (``background'') column parameter rather than cluster-specific values. Using such a noise distribution is a well-studied strategy in clustering, e.g., in \citep{banfield1993model}. Specifically, rows assigned to the outlier set~\(O \subseteq [n]\) have their Poisson parameters defined by:
\begin{equation}
\mu_{ij} = 
\lambda^R_i \lambda_{j\mid 0}^C,~~~{\rm if}~~~i \in O ~\wedge~ j \in C_0 \cup C_1,
\label{eq:outlier:mu}
\end{equation}
where $\bigcup_{k \in K} R_k \cup O = [n]$ and $\forall_{k \in K}~ R_k \cap O = \emptyset$.
At each iteration (or as a final step, after convergence), rows are assigned to clusters or to the outlier set based on likelihood comparison. For each row~\(i\), we compute the score for cluster assignment:
\begin{equation}
S_i^k = \sum\nolimits_{j \in C_1} \left( X_{ij} \log \hat{\mu}_{ij}^k - \hat{\mu}_{ij}^k \right) - \sum\nolimits_{j \in C_1} \left( X_{ij} \log \hat{\mu}_{ij}^0 - \hat{\mu}_{ij}^0\right),
\label{eq:outlier:score}
\end{equation}
where $\hat{\mu}_{ij}^k$ is defined in Sect.~\ref{sec:em2} and $\hat{\mu}_{ij}^0 = \frac{(\sum_{i = 1}^n {\bf X}_{ij})(\sum_{j\in C_1 \cup C_0} {\bf X}_{ij})}{\sum_{i = 1}^n \sum_{j \in C_1 \cup C_0} {\bf X}_{ij}}$.
The score \(S_i^k\) quantifies the log-likelihood difference between assigning row~\(i\) to cluster~\(k\) versus the background model. We then assign $i\in R_k$ if $\max_k S_i^k \ge 0$, otherwise $i\in O$. This criterion ensures that rows fitting any cluster better than the background model remain within a cluster. See Appendix~\ref{app:outlier} for detailed derivations and update rules.

\subsection{Convergence and Complexity}\label{sec:observations}

\textbf{Convergence.} We see that the steps described in Sects. \ref{sec:em1}, \ref{sec:em2}, and \ref{sec:column} always lead to a local optimum. Since there is a finite number of possible clustering solutions and each step is guaranteed not to increase the loss function, \Method has to converge. To empirically verify this claim, Fig. \ref{fig:convergence} shows the loss after each iteration for three data sets. It is evident that the loss is steadily decreasing. The dashes on the bottom also show that it is beneficial to integrate the selection of columns into the optimization instead of performing a simpler pre-/post-processing, as changes can happen at any time due to updates in $R_k$ which often lead to noticeable decreases in the loss.

\noindent\textbf{Complexity.} As for other clustering algorithms, except for simple methods such as $k$-Means, for which theoretical results exist \citep{Arthur07kmeans,arthur2006How}, it is challenging to prove approximability or runtime guarantees. However, we can derive the asymptotic computational complexity of our method. Calculating the column parameters $\lambda_{j|-}^C$ and $\lambda_{j|0}^C$ as well as the bias $b({\bf X}_{\cdot j}|K)$ has to be done only once, and the complexity is ${\mathcal O}(nm)$. Updating the row parameters $\lambda_i^R$ has a complexity of ${\mathcal O}(n\,|C_0\cup C_1|)$, and updating the cluster-specific column parameters $\lambda_{j|k}^C$ has a complexity of ${\mathcal O}(nm)$. 
The complexity of updating the cluster assignments $R_k$ is ${\mathcal O}(n\,|C_1|\,K)$, and for the column partitions $C_-$, $C_0$ and $C_1$ it is ${\mathcal O}(nm)$. Assuming $I$ iterations until the process converges gives the worst-case complexity of ${\mathcal O}(InmK)$, indicating that the runtime grows linearly with $I$, $n$, $m$, and $K$.

\begin{figure}[t]
    \centering
    \begin{subfigure}{0.31\textwidth}
        \includegraphics[width=\textwidth]{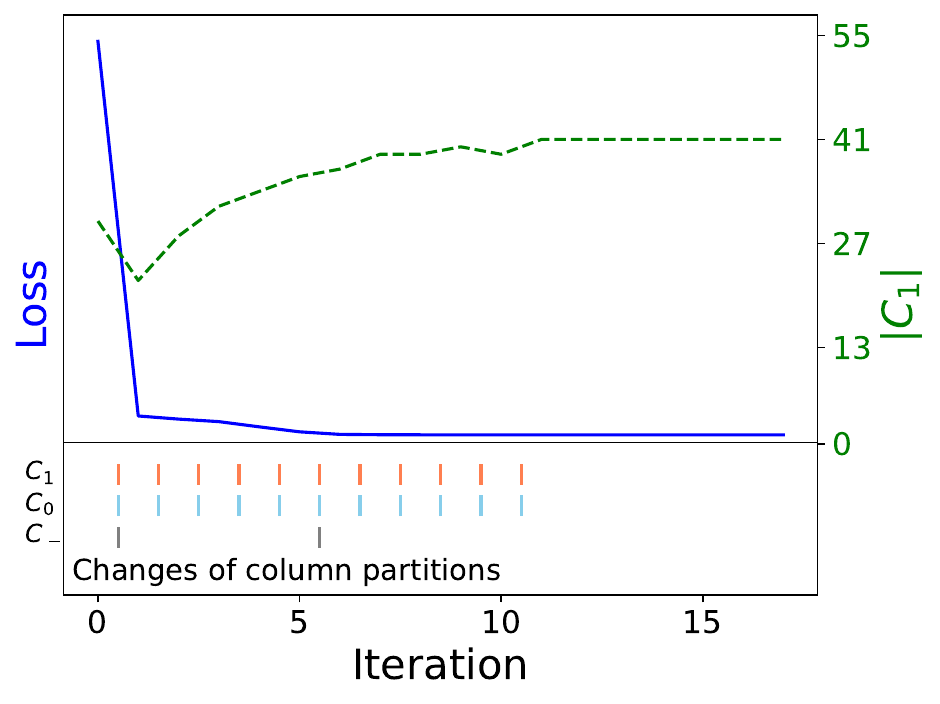}
        \caption{Data set: SportA.}
    \label{fig:loss_sportarticles}
    \end{subfigure}
    \begin{subfigure}{0.31\textwidth}
        \includegraphics[width=\textwidth]{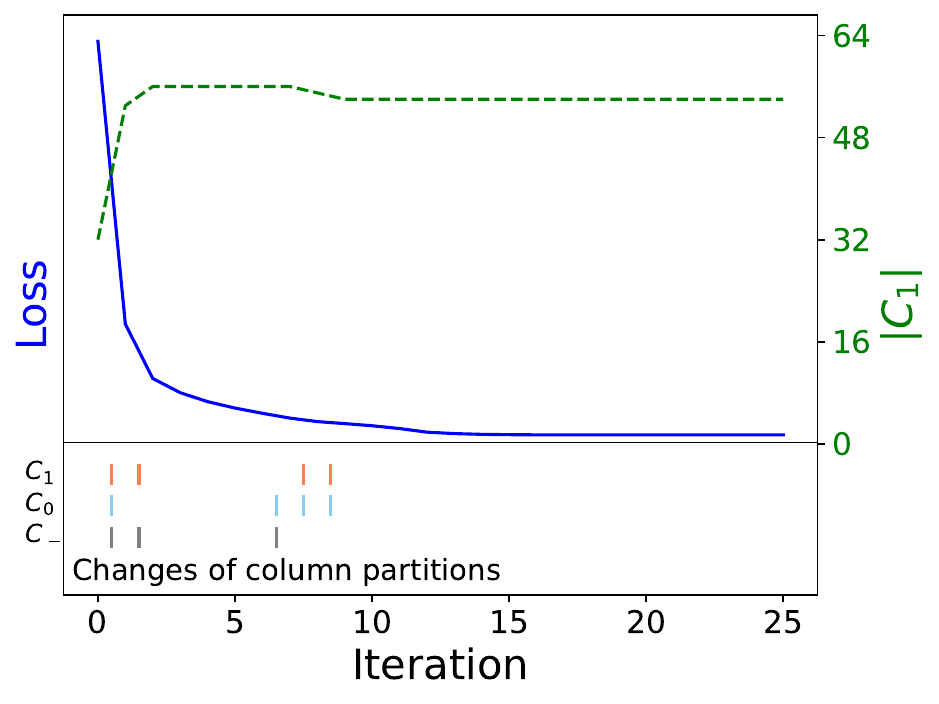}
        \caption{Data set: Optdigits.}
        \label{fig:loss_optdigits}
    \end{subfigure}
    \begin{subfigure}{0.31\textwidth}
        \includegraphics[width=\textwidth]{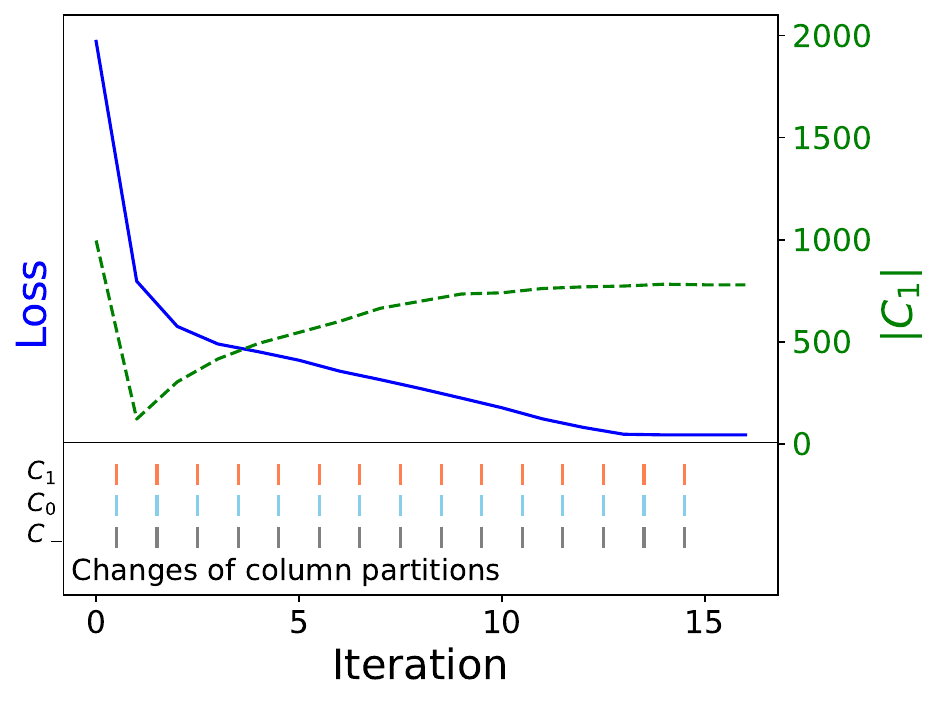}
        \caption{Data set: BBCSports.}
        \label{fig:loss_newsgroups}
    \end{subfigure}
    \caption{Loss of \Method (blue line) and number of columns in $C_1$ (green line) after each iteration. The dashes below indicate changes of $C_1$ (red), $C_0$ (blue) and $C_-$ (gray).}
    \label{fig:convergence}
\end{figure} 

\section{Experiments}\label{sec:exp}

We show the applicability of our method by evaluating its performance in various scenarios. In addition, we validate essential design decisions through ablation studies.

\subsection{Evaluation Setup}

We compare \Method with various comparison algorithms using multiple evaluation metrics and data sets from different domains.

\noindent\textbf{Comparisons.}
For comparison, we choose multiple clustering algorithms that return hard cluster labels, including the Poisson-based clustering algorithms PoissonL and PoissonC (using our proposed initialization strategy for $R_k$ - see Sect.~\ref{sec:initial}) as well as SKM. Furthermore, we consider $k$-Means (KM) in combination with various transformation functions to account for different normalizations. These are row-wise z-normalization (STD, ${\bf X}_{ij}^{STD} = ({\bf X}_{ij} - \text{mean}({\bf X}_{i \cdot}))/ \text{std}({\bf X}_{i \cdot})$), row-wise min-max normalization (MM, ${\bf X}_{ij}^{MM} = ({\bf X}_{ij} - \text{min}({\bf X}_{i \cdot}))/(\text{max}({\bf X}_{i \cdot})-\text{min}({\bf X}_{i \cdot}))$), relative frequency (RF, ${\bf X}_{ij}^{RF} = {\bf X}_{ij}/r_i$), and revealed comparative advantage (RCA, ${\bf X}_{ij}^{RCA} = {\bf X}_{ij}N/(r_i c_j)$).
Here, $r_i = \sum_{j=1}^m {\bf X}_{ij}$, $c_j = \sum_{i=1}^n {\bf X}_{ij}$, and $N=\sum_{j=1}^m c_j$. 

In addition, we consider the co-clustering algorithms CROINFO, CoclustMod (CCMod), CoclustSpecMod (CCSMod), ELBM, SELBM, and TauCC. The number of row and column clusters is set according to the ground truth and ELBM and SELBM use a Poisson model. Other parameters follow the defaults in the implementations.

All algorithms except TauCC receive the number of ground truth clusters as input parameter. As the comparison algorithms are unable to detect outliers, we use \Method without outlier detection if not explicitly mentioned otherwise to ensure a fair comparison. Outlier detection is examined separately in the end of Sect.~\ref{sec:experiment_outlier}.

Information regarding the used implementations is given in Appendix~\ref{sec:implementations}.

\noindent\textbf{Data sets.} We consider one synthetic and $11$ real-world count data sets from various domains: synthetic (Synth), \textit{Wholesales}~\citep{uciRepository}, \textit{Sport Articles} (SportA)~\citep{uciRepository}, \textit{Optdigits}~\citep{uciRepository}, \textit{BBCSports}~\citep{greene2006practical}, \textit{BBCNews}~\citep{greene2006practical}, \textit{WebKB}\footnote{\url{https://www.cs.cmu.edu/~webkb/} (accessed 07.01.2026)}, \textit{Reuters21578} (Reuters)~\citep{uciRepository}, \textit{20Newsgroups} (20NewsG)~\citep{uciRepository}, \textit{Mouse Cell Atlas} (MouseAtlas)~\citep{tran2020benchmark}, \textit{Gene Expression} (GeneExp)~\citep{uciRepository}, and \textit{Human Dendritic Cells} (HDendritic)~\citep{tran2020benchmark}. Detailed descriptions of these data sets, including potential pre-processing steps, are provided in Appendix~\ref{sec:dataset_appendix}, and the specific characteristics, i.e., the number of clusters $K$, rows $n$, columns $m$, data ranges, sparsity information, and imbalance ratios, are summarized in Table~\ref{tab:dataset_characteristics}.

\noindent\textbf{Metrics.} To evaluate our results, we consider the \textit{Unsupervised Clustering Accuracy} (ACC)~\citep{acc}, \textit{Normalized Mutual Information} (NMI)~\citep{nmi} and \textit{Adjusted Rand Index} (ARI)~\citep{ari}. These metrics compare the predicted clustering labels with ground truth labels, where $1$ indicates a perfect match and values close to $0$ a random assignment. All results are reported in \%.

\subsection{Evaluation of the Clustering Performance}

\noindent\textbf{Comparison to traditional Clustering Algorithms.}
The upper part of Table~\ref{tab:experiments_ari} shows the ARI results of \Method and the traditional clustering algorithms, i.e., those not performing a column selection (NMI and ACC results can be found in Appendix~\ref{sec:additional_results}).
Each entry corresponds to the average and standard deviation ($\pm$) of ten executions. Note that each execution itself consists of ten runs, where only the result with the best internal score (e.g., Eq.~\eqref{eq:likelihood} for \Method and inertia for $k$-Means) is considered. 

\begin{table*}[t]
\caption{ARI results (in \%) of traditional clustering (top) and co-clustering algorithms (bottom). Entries correspond to the mean of ten executions $\pm$ the standard deviation. The best result for each data set and group of algorithms (traditional/co-clustering) is marked in \textbf{bold}, the runner-up is \underline{underlined}, and the third place is \dashuline{dashed-underlined}.}
\label{tab:experiments_ari}
\resizebox{1\textwidth}{!}{
\begin{tabular}{l|ccccccccc}
\toprule
\textbf{Data set}  & 3CPO & PoissonL & PoissonC & SKM & KM & STD+KM & MM+KM & RF+KM & RCA+KM\\
\midrule
Synth & \bm{$95.5 \pm 0.0$} & $33.3 \pm 0.1$ & $29.9 \pm 0.0$ & \dashuline{$42.7 \pm 0.0$} & $0.2 \pm 0.0$ & $42.1 \pm 0.0$ & \underline{$42.9 \pm 0.1$} & $42.5 \pm 0.0$ & $0.0 \pm 0.0$\\
Wholesales & $30.1 \pm 0.0$ & \dashuline{$31.3 \pm 0.0$} & \underline{$32.8 \pm 0.0$} & $26.5 \pm 0.0$ & $-3.1 \pm 0.0$ & $28.2 \pm 0.0$ & $27.9 \pm 0.2$ & $21.3 \pm 0.0$ & \bm{$47.0 \pm 0.3$}\\
SportA & \underline{$32.3 \pm 0.5$} & \bm{$32.5 \pm 0.2$} & $21.9 \pm 1.5$ & $29.5 \pm 0.0$ & $21.3 \pm 0.0$ & \dashuline{$31.9 \pm 0.0$} & $25.9 \pm 0.1$ & $7.5 \pm 0.0$ & $0.1 \pm 0.0$\\
Optdigits & $66.7 \pm 2.2$ & $67.0 \pm 1.8$ & $45.0 \pm 3.1$ & \underline{$67.5 \pm 0.1$} & $67.1 \pm 0.2$ & \bm{$67.6 \pm 0.1$} & \dashuline{$67.3 \pm 0.3$} & $66.5 \pm 0.2$ & $0.0 \pm 0.0$\\
BBCSports & \bm{$90.1 \pm 2.9$} & \underline{$51.6 \pm 8.1$} & \dashuline{$27.7 \pm 4.0$} & $10.4 \pm 1.7$ & $0.3 \pm 0.2$ & $9.5 \pm 0.9$ & $6.9 \pm 3.6$ & $9.8 \pm 1.7$ & $3.7 \pm 2.6$\\
BBCNews & \bm{$89.9 \pm 0.9$} & \underline{$89.2 \pm 0.8$} & \dashuline{$63.4 \pm 6.7$} & $31.5 \pm 4.0$ & $6.0 \pm 0.3$ & $24.6 \pm 1.6$ & $16.4 \pm 1.5$ & $25.6 \pm 1.8$ & $15.4 \pm 12.0$\\
WebKB & \bm{$31.5 \pm 3.2$} & \underline{$25.4 \pm 0.9$} & \dashuline{$20.2 \pm 3.8$} & $11.4 \pm 0.2$ & $3.6 \pm 0.0$ & $12.0 \pm 0.2$ & $10.9 \pm 0.4$ & $11.0 \pm 2.7$ & $-0.0 \pm 0.1$\\
Reuters & \bm{$67.8 \pm 1.6$} & \underline{$65.2 \pm 5.0$} & \dashuline{$44.1 \pm 4.0$} & $21.9 \pm 0.2$ & $15.2 \pm 0.6$ & $32.5 \pm 0.1$ & $31.9 \pm 0.0$ & $3.7 \pm 4.2$ & $0.3 \pm 0.3$\\
20NewsG & \bm{$23.0 \pm 1.3$} & \underline{$22.2 \pm 1.9$} & \dashuline{$14.6 \pm 1.5$} & $2.0 \pm 0.2$ & $0.4 \pm 0.0$ & $2.1 \pm 0.1$ & $1.5 \pm 0.1$ & $1.1 \pm 0.2$ & $0.0 \pm 0.0$\\
MouseAtlas & \bm{$56.3 \pm 6.2$} & \underline{$55.4 \pm 5.0$} & $18.5 \pm 9.5$ & \dashuline{$39.3 \pm 2.6$} & $1.6 \pm 0.1$ & $38.4 \pm 1.4$ & $36.3 \pm 2.0$ & $28.6 \pm 2.6$ & $0.1 \pm 0.0$\\
GeneExp & \bm{$99.1 \pm 0.1$} & \underline{$98.7 \pm 0.1$} & $86.7 \pm 11.0$ & \dashuline{$98.5 \pm 0.0$} & $98.2 \pm 0.1$ & \bm{$99.1 \pm 0.1$} & $98.3 \pm 0.0$ & \dashuline{$98.5 \pm 0.1$} & $3.7 \pm 9.8$\\
HDendritic & $79.6 \pm 5.8$ & \bm{$83.4 \pm 5.2$} & $74.8 \pm 6.1$ & \dashuline{$79.9 \pm 0.2$} & $37.5 \pm 13.3$ & \underline{$80.2 \pm 0.4$} & $75.5 \pm 0.2$ & $37.5 \pm 13.3$ & $-0.1 \pm 0.1$\\
\bottomrule
\addlinespace[0.5pt]
\cmidrule[\heavyrulewidth]{1-8}
\textbf{Data set} & 3CPO & CROINFO & CCMod & CCSMod & ELBM & SELBM & TauCC \\
\cmidrule{1-8}
Synth & \bm{$95.5 \pm 0.0$} & \underline{$32.8 \pm 0.3$} & \dashuline{$25.2 \pm 0.1$} & $17.8 \pm 20.8$ & $9.2 \pm 8.6$ & $3.5 \pm 1.6$ & $25.1 \pm 0.6$\\
Wholesales & \dashuline{$30.1 \pm 0.0$} & \bm{$33.6 \pm 0.5$} & \underline{$32.0 \pm 0.0$} & $3.2 \pm 9.1$ & $2.0 \pm 5.0$ & $1.9 \pm 0.6$ & $27.6 \pm 0.0$\\
SportA & \bm{$32.3 \pm 0.5$} & $25.4 \pm 0.8$ & \dashuline{$30.3 \pm 0.0$} & \underline{$31.4 \pm 0.1$} & $7.1 \pm 0.0$ & $6.5 \pm 0.1$ & $15.8 \pm 15.1$\\
Optdigits & \bm{$66.7 \pm 2.2$} & \underline{$49.7 \pm 2.1$} & $21.4 \pm 1.5$ & $33.2 \pm 2.9$ & \dashuline{$45.3 \pm 3.0$} & $23.1 \pm 4.3$ & $10.6 \pm 7.5$\\
BBCSports & \bm{$90.1 \pm 2.9$} & \underline{$48.3 \pm 5.2$} & \dashuline{$41.0 \pm 4.3$} & $38.1 \pm 0.1$ & $0.3 \pm 0.1$ & $0.3 \pm 0.0$ & $26.2 \pm 6.3$\\
BBCNews & \bm{$89.9 \pm 0.9$} & $62.1 \pm 2.3$ & \underline{$71.9 \pm 1.8$} & \dashuline{$62.3 \pm 0.3$} & $4.5 \pm 0.2$ & $2.9 \pm 0.6$ & $41.8 \pm 10.5$\\
WebKB & \bm{$31.5 \pm 3.2$} & $20.1 \pm 1.4$ & \dashuline{$20.4 \pm 0.7$} & \underline{$25.8 \pm 0.1$} & $3.7 \pm 0.1$ & $0.0 \pm 0.0$ & $11.3 \pm 2.0$\\
Reuters & \bm{$67.8 \pm 1.6$} & \dashuline{$59.9 \pm 1.2$} & \underline{$64.6 \pm 4.1$} & $46.1 \pm 0.0$ & $34.6 \pm 0.3$ & $10.1 \pm 1.0$ & $43.2 \pm 0.5$\\
20NewsG & \bm{$23.0 \pm 1.3$} & \underline{$16.2 \pm 0.7$} & \dashuline{$6.9 \pm 0.6$} & $6.2 \pm 2.3$ & $0.5 \pm 0.1$ & $0.0 \pm 0.0$ & $2.1 \pm 1.5$\\
MouseAtlas & \bm{$56.3 \pm 6.2$} & \underline{$49.7 \pm 1.8$} & \dashuline{$38.4 \pm 1.8$} & $27.4 \pm 0.1$ & $1.4 \pm 0.3$ & $0.5 \pm 0.1$ & $34.4 \pm 4.6$\\
GeneExp & \bm{$99.1 \pm 0.1$} & $59.0 \pm 3.1$ & \dashuline{$72.4 \pm 2.9$} & \underline{$93.2 \pm 0.0$} & $4.3 \pm 1.4$ & $0.0 \pm 0.0$ & $53.1 \pm 9.4$\\
HDendritic & \underline{$79.6 \pm 5.8$} & \bm{$81.9 \pm 1.3$} & \dashuline{$72.2 \pm 5.0$} & $0.0 \pm 0.0$ & $45.1 \pm 4.3$ & $4.0 \pm 7.8$ & $42.3 \pm 23.0$\\
\cmidrule[\heavyrulewidth]{1-8}
\end{tabular}}
\end{table*}

We see that \Method is the top performer in eight out of twelve experiments. Considering Synth and BBCSport, it outperforms all competitors by more than $38\%$ and by more than $6\%$ in the case of WebKB. Furthermore, if \Method does not perform best, the gap to the best performer is often marginal with differences below $1\%$ in the case of SportA and Optdigits. An exception is Wholesales, where RCA+KM is by far the best algorithm, indicating why RCA is a common technique in economics. However, RCA+KM performs much worse than its competitors in all other scenarios. The only other case in which the gap to the best-performing algorithm exceeds $1\%$ is HDendritic. Here, the difference between \Method ($79.6 \pm 5.8$) and the best competitor, PoissonL ($83.4 \pm 5.2$), lies within one standard deviation, indicating that \Method often attains results comparable to the top performer. Overall, the strongest competitor is PoissonL, which can be interpreted as a special version of \Method with $C_-=C_0=\emptyset$. Nevertheless, \Method yields superior results on most high-dimensional text and biology data sets, underscoring the benefit of explicit column selection. Considering the average across all tested data sets, \Method is a good choice in unsupervised scenarios.

\noindent\textbf{Comparison to Co-Clustering Algorithms.}
The lower part of Table~\ref{tab:experiments_ari} shows the ARI results of \Method and the co-clustering clustering algorithms (NMI and ACC results can be found in Appendix~\ref{sec:additional_results}). The setting is the same as for the traditional approaches. The only exception is TauCC, where each execution corresponds to a single run as the implementation does not include an internal scoring function.

The results show that \Method ranks among the top three across all data sets. The strongest competitor is CROINFO, which outperforms \Method on Wholesales and HDendritic while maintaining a low standard deviation in almost all experiments. Considering that TauCC has to determine the number of clusters and only does a single run per execution, its performance is noteworthy as it is often in a similar region as CROINFO, CCMod, and CCSMod. In contrast, ELBM and SELBM appear to struggle with raw features, suggesting they may require additional pre-processing.

\begin{table*}[t]
\centering
\caption{ARI results (in \%) regarding text data sets pre-processed by TF-IDF and BM25. Entries correspond to the mean of ten executions $\pm$ the standard deviation. The best performance is highlighted in \textbf{bold}.}
\label{tab:tfidf_ari}
\resizebox{0.7\textwidth}{!}{
\begin{tabular}{l|ccccc}
\toprule
\textbf{Data set} & \Method & TF-IDF+KM & BM25+KM & TF-IDF+SKM & BM25+SKM\\
\midrule
BBCSports & \bm{$90.1 \pm 2.9$} & $71.0 \pm 11.4$ & $43.7 \pm 17.0$ & $76.5 \pm 7.1$ & $81.1 \pm 6.8$\\
BBCNews & $89.9 \pm 0.9$ & $84.2 \pm 2.8$ & $87.7 \pm 6.8$ & $86.4 \pm 0.8$ & \bm{$90.7 \pm 0.3$}\\
WebKB & \bm{$31.5 \pm 3.2$} & $18.8 \pm 1.0$ & $14.6 \pm 1.3$ & $22.4 \pm 1.1$ & $29.0 \pm 1.9$\\
Reuters & \bm{$67.8 \pm 1.6$}  & $19.3 \pm 0.2$ & $34.8 \pm 5.0$ & $34.0 \pm 0.2$ & $55.8 \pm 1.2$\\
20NewsG &\bm{$23.0 \pm 1.3$} & $5.6 \pm 0.6$ & $5.5 \pm 0.5$ & $8.2 \pm 0.4$ & $21.1 \pm 1.2$\\
\bottomrule
\end{tabular}
}
\end{table*}

\noindent\textbf{A detailed look at text data sets: Comparison to TF-IDF.}
To better assess the performance with respect to the text data sets, we compare \Method to $k$-Means and SKM in combination with term frequency–inverse document frequency (TF-IDF) and BM25~\citep{robertson09probabilistic} with its free parameters set to $k_\text{bm25}=1.5$ and $b_\text{bm25}=0.75$. The remaining experimental setting corresponds to Table~\ref{tab:experiments_ari}. The ARI results are shown in Table~\ref{tab:tfidf_ari} (NMI and ACC results can be found in Appendix~\ref{sec:additional_results}).
We see that \Method performs best in four out of five scenarios. It outperforms the runner-up, BM25+SKM, by a substantial margin on BBCSports and Reuters. 

In general, the experiments indicate the applicability of \Method for text data. Furthermore, as it uses raw word counts and discovers a subset of relevant columns, its results are more interpretable than those of more sophisticated word embeddings like Word2Vec~\citep{mikolov2013efficient} or BERT~\citep{devlin2019bert}. This can benefit a subsequent analysis of the results, allowing domain experts to identify key terms driving cluster definitions. Our main experiments further demonstrate that \Method can be applied to data from various domains and is not restricted to texts. 

\noindent\textbf{Performance with Outlier Detection.}\label{sec:experiment_outlier}
We further evaluate \Method when incorporating the optional ``background'' component for outlier detection. Table~\ref{tab:outlier_result} shows the results where metrics (denoted by $*$) are calculated only on non-outlier samples. Using the outlier detection method, \Method is able to improve the clustering scores for all data sets except Synth, indicating that it successfully identifies rows that are hard to assign to any specific cluster. This behavior is practically valuable since declaring uncertain samples as outliers can prevent wrong cluster assignments leading to false conclusions. 

The proportion of identified outliers varies considerably across data sets, exceeding $10\%$ for Synth, Wholesales, SportA, and MouseAtlas. While this reduces the clustering task's complexity, it also reflects genuine data characteristics where many samples lack clear cluster membership.
We want to highlight the result regarding GeneExp, where \Method achieves a perfect clustering result while only identifying $\sim14$ outliers.

Appendix~\ref{sec:outliers_optdigits} provides visual examples of outliers identified within Optdigits.

\begin{table*}
\centering
\caption{Clustering results (in \%) and the number of outliers $|O|$ when using the outlier detection feature. An asterisk $*$ indicates that only non-outliers are used for the evaluation. Entries correspond to the mean of ten executions $\pm$ the standard deviation. Improvements compared to results without outlier detection are marked in \textbf{bold}.}
\label{tab:outlier_result}
\resizebox{1\textwidth}{!}{
\begin{tabular}{l|c|c|c|c||l|c|c|c|c}
\toprule
\textbf{Data set} & $\text{ACC}^*$ & $\text{NMI}^*$ & $\text{ARI}^*$ & $|O|$ & \textbf{Data set} & $\text{ACC}^*$ & $\text{NMI}^*$ & $\text{ARI}^*$ & $|O|$ \\
\midrule
Synth & $91.9 \pm 0.0$ & $81.7 \pm 0.0$ & $85.4 \pm 0.0$ & $336 \pm 0$ & WebKB & $56.8 \pm 3.7$ & \bm{$39.3 \pm 1.1$} & \bm{$33.1 \pm 3.0$} & $431 \pm 38$ \\
Wholesales & \bm{$86.2 \pm 0.2$} & \bm{$47.4 \pm 0.5$} & \bm{$52.2 \pm 0.7$} & $144 \pm 1$ & Reuters & \bm{$85.3 \pm 7.3$} & \bm{$75.3 \pm 7.0$} & \bm{$73.9 \pm 8.2$} & $427 \pm 275$\\
SportA &  \bm{$85.3 \pm 0.1$} & \bm{$37.2 \pm 0.2$} & \bm{$49.4 \pm 0.2$} & $341 \pm 4$ & 20NewsG & \bm{$39.8 \pm 2.6$} & \bm{$42.7 \pm 1.1$} & \bm{$25.4 \pm 1.3$} & $978 \pm 49$\\
Optdigits & \bm{$82.1 \pm 2.0$} & \bm{$79.2 \pm 1.0$} & \bm{$72.6 \pm 2.4$} & $443 \pm 3$ & MouseAtlas & \bm{$72.4 \pm 7.7$} & \bm{$75.8 \pm 3.5$} & \bm{$64.5 \pm 8.2$} & $1326 \pm 84$\\
BBCSports & \bm{$98.4 \pm 1.1$} & \bm{$95.5 \pm 2.0$} & \bm{$95.9 \pm 2.7$} & $56 \pm 4$ & GeneExp & \bm{$100.0 \pm 0.0$} & \bm{$100.0 \pm 0.0$} & \bm{$100.0 \pm 0.0$} & $14 \pm 2$\\
BBCNews & \bm{$97.9 \pm 0.4$} & \bm{$93.3 \pm 0.8$} & \bm{$94.9 \pm 0.9$} & $137 \pm 11$ & HDendritic & \bm{$87.7 \pm 7.3$} & \bm{$81.5 \pm 3.7$} & \bm{$82.6 \pm 5.5$} & $22 \pm 2$\\
\bottomrule
\end{tabular}}
\end{table*}

\subsection{Analysis of the Column Partitions}
\begin{table*}[t]
\centering
\caption{Total ($m$) and cluster-relevant columns ($|C_1|$) as identified by \Method. Entries for $|C_1|$ correspond to the mean of ten executions $\pm$ the standard deviation.}
\label{tab:numberColumns}
\resizebox{1\textwidth}{!}{
\begin{tabular}{l|cccccccccccc}
\toprule
& Synth & Wholesales & SportA & Optdigits & BBCSports & BBCNews & WebKB & Reuters & 20NewsG & MouseAtlas & GeneExp & HDendritic \\
\midrule
$m$ & $6$ & $6$ & $55$ & $64$ & $2000$ & $2000$ & $2000$ & $2000$ & $2000$ & $15006$ & $20531$ & $26593$\\
$|C_1|$ & $2 \pm 0$ & $5 \pm 0$ & $41 \pm 1$ & $56 \pm 1$ & $799 \pm 8$ & $1379 \pm 6$ & $1315 \pm 18$ & $1322 \pm 8$ & $1454 \pm 15$ & $14566 \pm 59$ & $6445 \pm 3$ & $16673 \pm 406$\\
\bottomrule
\end{tabular}}
\end{table*}

In addition to the clustering performance, we analyze the number of columns selected as relevant for the clustering task, i.e., the number of columns in $C_1$. These results are shown in Table~\ref{tab:numberColumns}.
For many high-dimensional data sets such as BBCSports, BBCNews, WebKB, Reuters, 20NewsG, and GeneExp, \Method noticeable reduces the number of selected columns. 
Our main takeaway is that our approach fulfills its goal by outperforming PoissonL in most cases, while considering less columns for clustering, e.g., only $\sim 32\%$ in the case of GeneExp. Noteworthy is also the result for BBCSports, where \Method outperforms all competitors while only using $\sim 40\%$ of the columns.

\begin{figure}[t]
    \centering
    \begin{subfigure}{0.48\textwidth}
        \centering
        \includegraphics[height=0.36\textwidth]{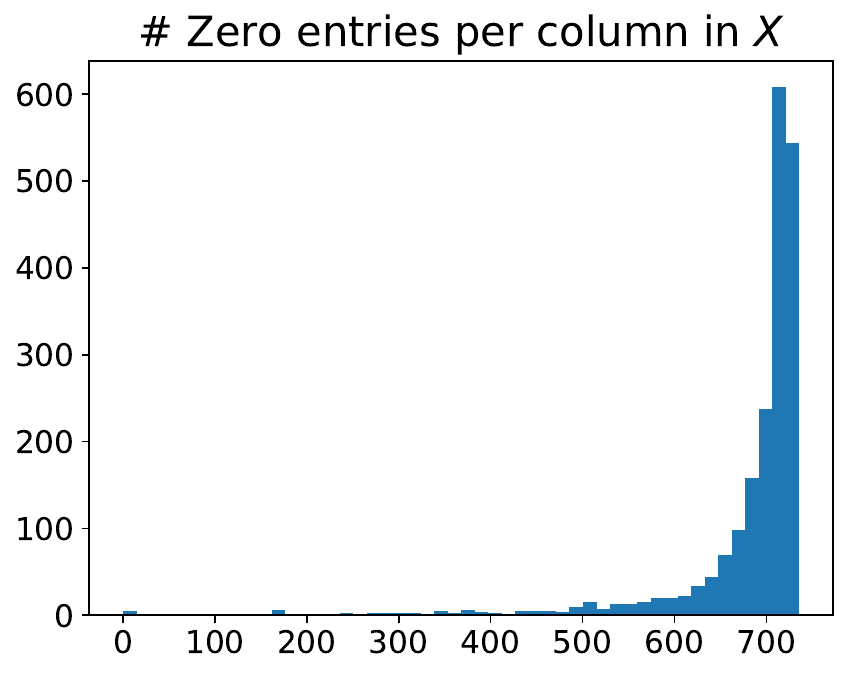}
        \includegraphics[height=0.36\textwidth]{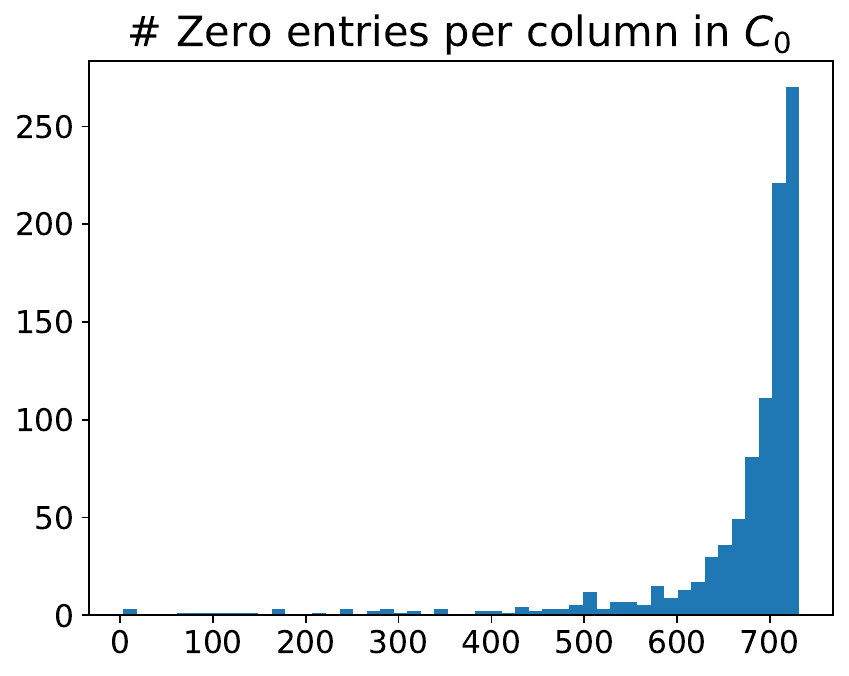}
        \includegraphics[height=0.36\textwidth]{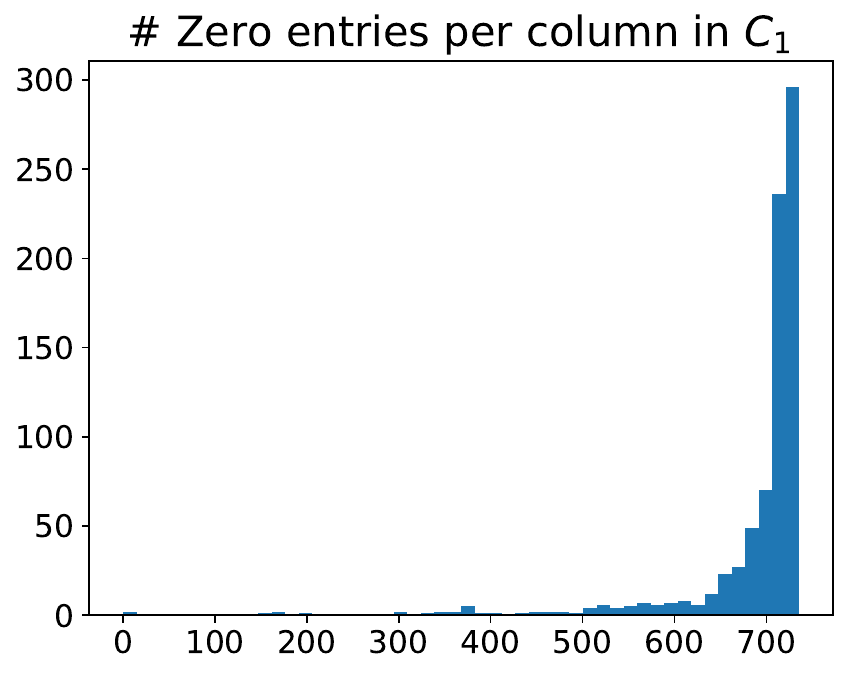}
        \includegraphics[height=0.36\textwidth]{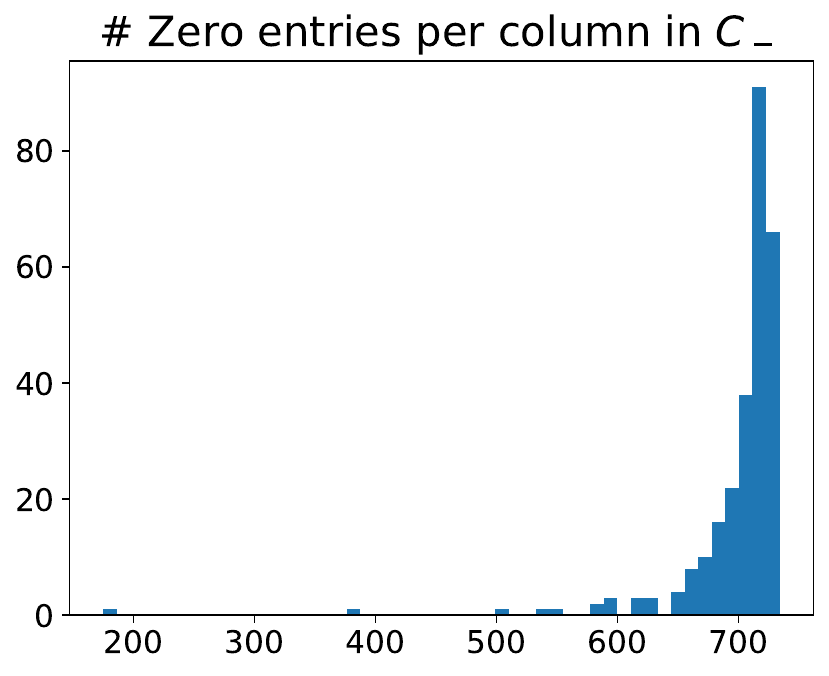}
        \caption{BBCSports data set.}
    \end{subfigure}
    \begin{subfigure}{0.48\textwidth}
        \centering
        \includegraphics[height=0.36\textwidth]{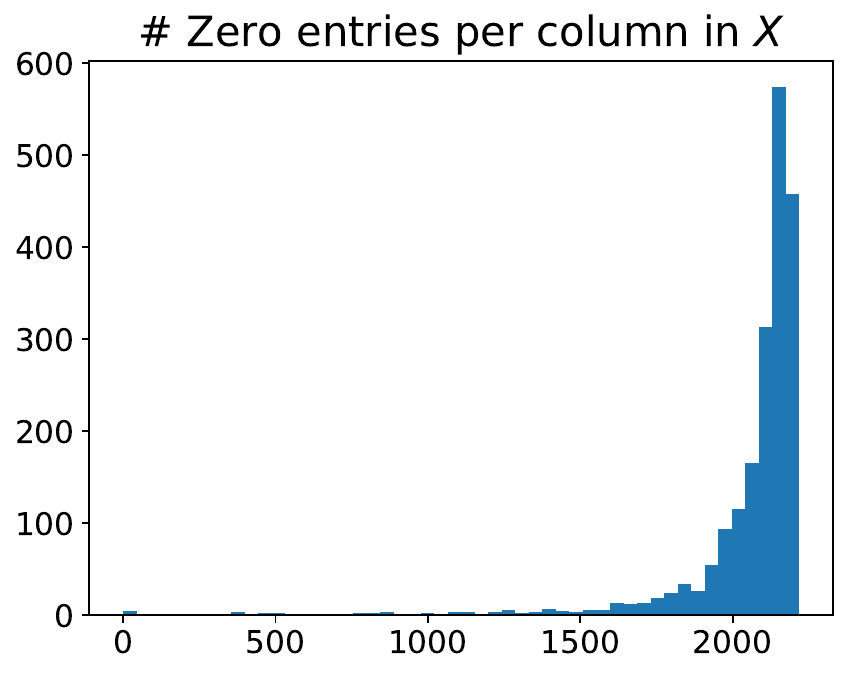}
        \includegraphics[height=0.36\textwidth]{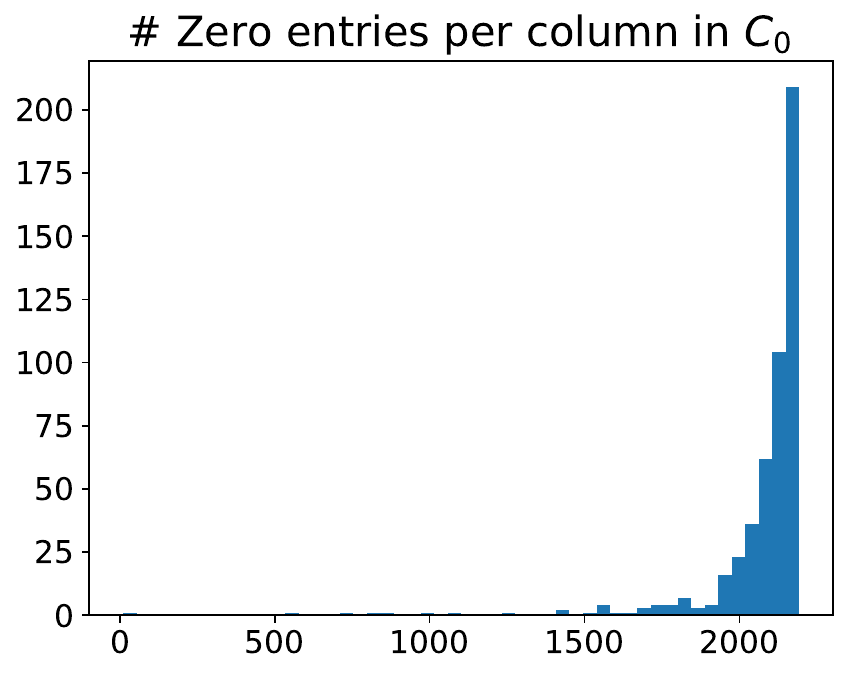}
        \includegraphics[height=0.36\textwidth]{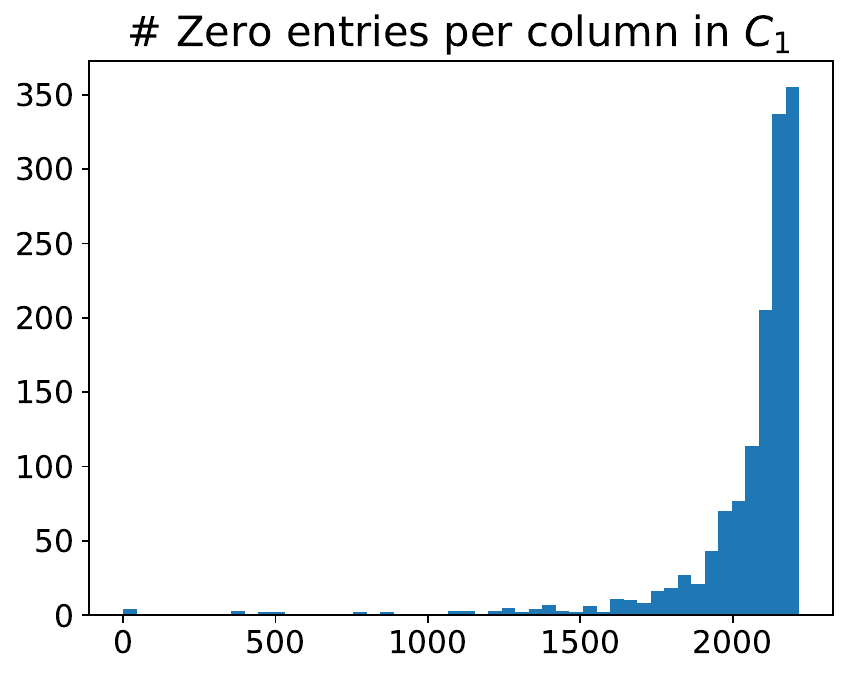}
        \includegraphics[height=0.36\textwidth]{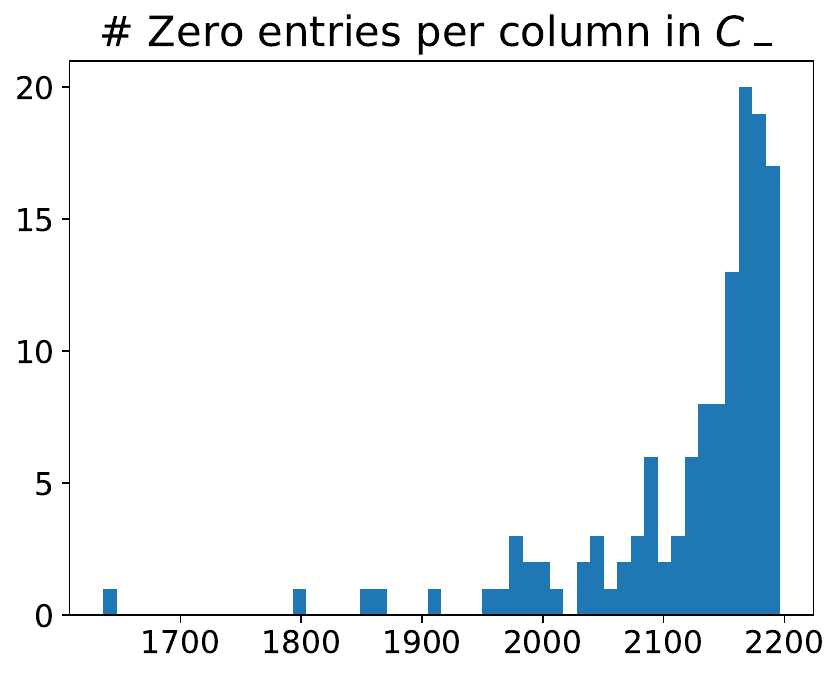}
        \caption{BBCNews data set.}
    \end{subfigure}
    \caption{Histograms showing the amount of zeros per column in each column partition.}
    \label{fig:zero_entries}
\end{figure}

As count matrices often contain many zero entries (see Table~\ref{tab:dataset_characteristics}), we further investigate whether \Method simply filters out columns that contain mostly zeros. To disprove this hypothesis, we analyze the amount of zero entries per column in the whole data ${\bf X}$ as well as in $C_1, C_0$ and $C_-$. The histograms in Fig.~\ref{fig:zero_entries} illustrate that the general distribution of the zero entries per column is similar across all column partitions. This applies to data sets where only $\sim 40\%$ of the columns are put into $C_1$ (see BBCSports) and to data sets where more than $65\%$ of the columns are assigned to $C_1$ (see BBCNews). This confirms that \Method's performance is driven by structural patterns in the count data rather than a trivial filtering of zero-heavy columns.

\begin{table*}[t]
\centering
\caption{The ten most important words (based on the largest difference $\lambda_{j|l}^C-\lambda_{j|0}^C$ for $j \in C_1$ and based on the largest $\lambda_{j|l}^C$ for $j\in C_0 \cup C_-$) within the BBCSports and BBCNews data sets per column partition and cluster as identified by \Method.}
\label{tab:detailed_columns}
\resizebox{0.96\textwidth}{!}{
\begin{tabular}{l|l|l|c}
\toprule
Data set & Partition + Cluster & Most important words w.r.t. $\lambda_{j|l}^C-\lambda_{j|0}^C$ or $\lambda_{j|l}^C$ (stemmed) & Ground Truth Fit \\
\midrule
BBCSports & $C_1 \qquad 1~ (\Rightarrow l=1)$ & `half', `six', `their', `franc', `nation', `the', `rugbi', `ireland', `wale', `england' & Rugby\\ 
BBCSports & $C_1 \qquad 2~ (\Rightarrow l=2)$ & `year', `indoor', `world', `race', `athlet', `her', `olymp', `she', `in', `the' & Athletics\\ 
BBCSports & $C_1 \qquad 3~ (\Rightarrow l=3)$ & `four', `day', `by', `run', `his', `test', `wicket', `over', `off', `ball' & Cricket\\ 
BBCSports & $C_1 \qquad 4~ (\Rightarrow l=4)$ & `the', `australian', `she', `her', `final', `in', `seed', `roddick', `set', `open' & Tennis\\ 
BBCSports & $C_1 \qquad 5~ (\Rightarrow l=5)$ & `unit', `leagu', `have', `chelsea', `club', `that', `he', `is', `we', `it' & Football\\ 
BBCSports & $C_0 \qquad ~~ (\Rightarrow l=0)$ & `as', `at', `with', `but', `was', `on', `for', `of', `and', `to' & --\\ 
BBCSports & $C_- \qquad ~ (\Rightarrow l=-)$ & `former', `move', `face', `month', `week', `old', `intern', `injuri', `last', `said' & --\\
\midrule
BBCNews & $C_1 \qquad 1~ (\Rightarrow l=1)$ & `tori', `blair', `would', `govern', `said', `parti', `elect', `labour', `he', `mr' & Politics\\ 
BBCNews & $C_1 \qquad 2~ (\Rightarrow l=2)$ & `music', `was', `for', `award', `best', `of', `in', `and', `film', `the' & Entertainment\\ 
BBCNews & $C_1 \qquad 3~ (\Rightarrow l=3)$ & `player', `after', `england', `game', `play', `win', `we', `his', `he', `but' & Sport\\ 
BBCNews & $C_1 \qquad 4~ (\Rightarrow l=4)$ & `said', `firm', `bank', `market', `compani', `us', `of', `it', `in', `the' & Business\\ 
BBCNews & $C_1 \qquad 5~ (\Rightarrow l=5)$ & `can', `user', `phone', `game', `are', `peopl', `that', `technolog', `mobil', `use' & Tech\\ 
BBCNews & $C_0 \qquad ~~ (\Rightarrow l=0)$ & `about', `new', `up', `their', `an', `this', `will', `as', `is', `to' & --\\ 
BBCNews & $C_- \qquad ~ (\Rightarrow l=-)$ & '12', `earlier', `20', `januari', `charg', `2003', `third', `dure', `follow', `week' & --\\
\bottomrule
\end{tabular}}
\end{table*}

Lastly, we provide a qualitative analysis of the column partitioning using the text data sets BBCSports and BBCNews. We identify the most discriminative words in $C_1$ for each cluster by calculating the difference $\lambda^C_{j|k}-\lambda^C_{j|0}$. Additionally, we report the most frequent words in $C_0$ and $C_-$, characterized by $\lambda^C_{j|0}$ and $\lambda^C_{j|-}$, respectively. The results are shown in Tab.~\ref{tab:detailed_columns}. \Method is able to capture essential cluster characteristics (e.g., `party' and `govern' for `Politics' or `athlete' and `olymp' for `Athletics') and also recognizes irrelevant words (e.g., stop words) by moving them into $C_0$ and $C_-$.

\subsection{Robustness to Noise Columns}

\begin{figure*}[t]
    \centering
    \begin{subfigure}{0.385\textwidth}
        \centering
        \includegraphics[width=\textwidth]{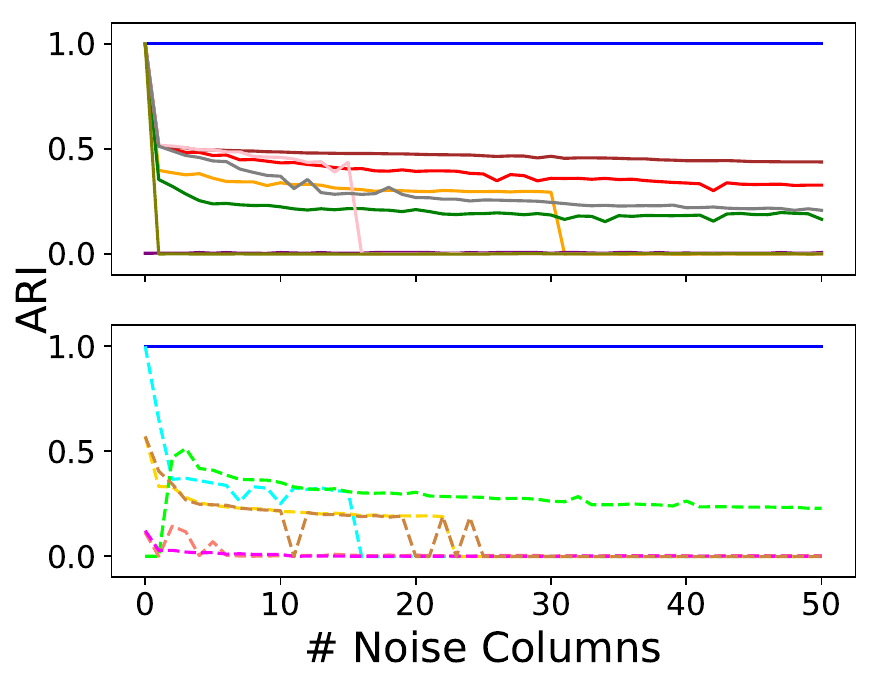}
        \caption{Increasing the amount of noise columns.}
        \label{fig:robustness_increasing_number_columns}
    \end{subfigure}
    \begin{subfigure}{0.385\textwidth}
        \centering
        \includegraphics[width=\textwidth]{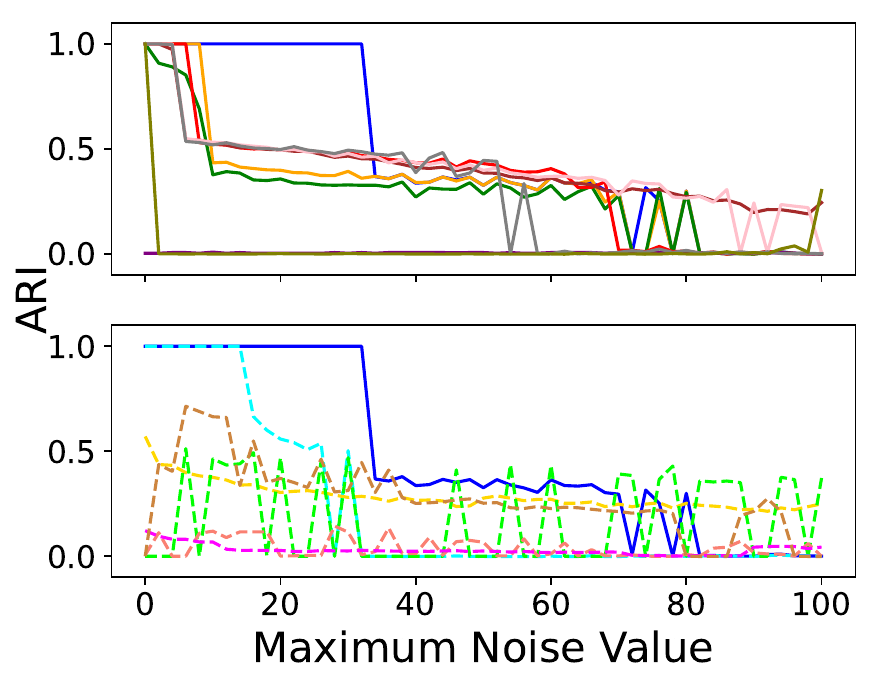}
        \caption{Increasing the maximum value within the noise column.}
        \label{fig:robustness_maximum_noise_value}
    \end{subfigure}
    \begin{subfigure}{0.21\textwidth}
        \centering
        \includegraphics[width=\textwidth]{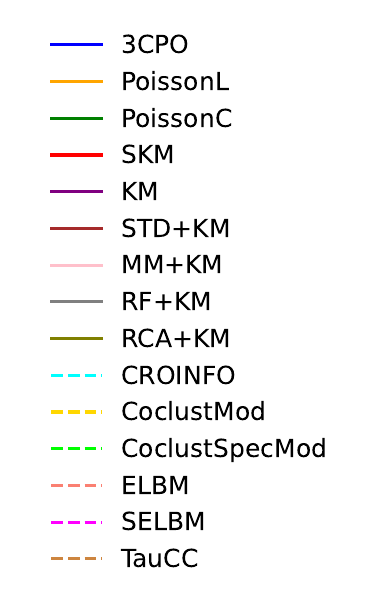}
        \vspace{0.5cm}
    \end{subfigure}
    \caption{Robustness experiments investigating the impact of columns containing uniformly distributed noise on the clustering performance (measured by ARI). The upper plots show traditional clustering algorithms and the lower plots show co-clustering algorithms. In \textbf{(a)} we increase the amount noise columns with values within $[0,15]$ and in \textbf{(b)} we have a single noise column with values within $[0,u]$ ($u$ is given on the x-axis).}
    \label{fig:robustness}
\end{figure*}

We investigate the robustness of \Method with respect to noise columns by conducting two experiments. The basis for these experiments is the Synth data set, where we only consider the first four columns ($C_1$ and $C_0$). In the first experiment (Fig. \ref{fig:robustness_increasing_number_columns}), we add an increasing number of columns containing uniformly distributed values within $[0, 15]$. The results show that \Method is the only algorithm that is not affected by these noise columns, as the ARI is constantly $1.0$, indicating that it successfully assigns all those columns into $C_-$. In the second experiment (Fig. \ref{fig:robustness_maximum_noise_value}), we add a single column containing uniformly distributed values within $[0, u]$, where $u \in [1, 100]$. Here, all traditional clustering algorithms except KM and RCA+KM perform well if $u \le 7$. With respect to the co-clustering algorithms only CROINFO returns good results if $u \le 18$. However, with increasing $u$, \Method is the only method that still returns high-quality results until around $u=32$.

\subsection{Ablation Studies}\label{sec:ablation}

We conduct a series of experiments in which certain components of our method are ignored. Here, we fix $C_{-}=\emptyset$, $C_{0}=\emptyset$, ignore the penalty for column selection by setting $b({\bf x}|K)=0$, or use BIC as the penalty for column selection as described in Sect. \ref{sec:column}. The results are summarized in Table~\ref{tab:ablation_ari}. As our proposed version of \Method never performs significantly worse than a competitor, it is considered the most reliable across all experiments. In addition, it usually identifies a solution with fewer selected columns.
We see that ignoring $b({\bf x}|K)$ and setting $C_0=\emptyset$ leads to similar results, including a significant increase in the number of selected columns without positively influencing the clustering results. This confirms our hypothesis that a penalty term is needed to identify columns that behave similarly for all clusters. In many cases, removing such columns positively influences the clustering, as less noise is included in the final result. This is particularly evident with BBCSports and 20NewsG. Without the integration of $C_-$, \Method also has trouble properly clustering WebKB. When using $b({\bf x}|K)=BIC({\bf X}_{\cdot j}|K)$ instead of the MDL-based formulation, the results are very similar, indicating stability w.r.t. the choice of $b({\bf x}|K)$. We want to emphasize that in situations where the proposed variant of \Method does not perform best (e.g., in the case of SportA), the difference is usually small and within the standard deviation.

\begin{table*}
\centering
\caption{ARI results (in \%) and the final number of columns included in $C_1$ of various ablation studies. Entries correspond to the mean of ten executions $\pm$ the standard deviation. Colors indicate whether the average performance of \Method lies above (green), within (yellow) or below (red) the standard deviation band, where above is better for ARI and below is better for $|C_1|$.}
\label{tab:ablation_ari}
\resizebox{0.65\textwidth}{!}{
\begin{tabular}{l|l|ccccc}
\toprule
\textbf{Data set} & \textbf{Metric} & 3CPO & $C_{-}=\emptyset$ & $C_{0}=\emptyset$ & $b({\bf X}_{\cdot j})=0$ & $b({\bf X}_{\cdot j})=\text{BIC}({\bf X}_{\cdot j})$\\
\midrule
Synth & ARI & $95.5 \pm 0.0$ & \cellcolor{green!30}$33.3 \pm 0.1$ & \cellcolor{yellow!30}$95.5 \pm 0.0$ & \cellcolor{yellow!30}$95.5 \pm 0.0$ & \cellcolor{yellow!30}$95.5 \pm 0.0$\\
& $|C_1|$ & $2.0 \pm 0.0$ & \cellcolor{green!30}$5.8 \pm 0.6$ & \cellcolor{green!30}$4.0 \pm 0.0$ & \cellcolor{green!30}$4.0 \pm 0.0$ & \cellcolor{yellow!30}$2.0 \pm 0.0$\\
\midrule
Wholesales & ARI & $30.1 \pm 0.0$ & \cellcolor{red!30}$31.3 \pm 0.0$ & \cellcolor{yellow!30}$30.1 \pm 0.0$ & \cellcolor{yellow!30}$30.1 \pm 0.0$ & \cellcolor{yellow!30}$30.1 \pm 0.0$\\
& $|C_1|$ & $5.0 \pm 0.0$ & \cellcolor{green!30}$6.0 \pm 0.0$ & \cellcolor{yellow!30}$5.0 \pm 0.0$ & \cellcolor{yellow!30}$5.0 \pm 0.0$ & \cellcolor{yellow!30}$5.0 \pm 0.0$\\
\midrule
SportA & ARI & $32.3 \pm 0.5$ & \cellcolor{yellow!30}$32.3 \pm 0.1$ & \cellcolor{red!30}$32.6 \pm 0.2$ & \cellcolor{red!30}$32.6 \pm 0.2$ & \cellcolor{yellow!30}$32.5 \pm 0.3$\\
& $|C_1|$ & $41.1 \pm 1.4$ & \cellcolor{green!30}$42.5 \pm 1.1$ & \cellcolor{green!30}$52.0 \pm 0.0$ & \cellcolor{green!30}$52.0 \pm 0.0$ & \cellcolor{yellow!30}$41.3 \pm 0.9$\\
\midrule
Optdigits & ARI & $66.7 \pm 2.2$ & \cellcolor{yellow!30}$66.9 \pm 1.8$ & \cellcolor{yellow!30}$66.8 \pm 2.2$ & \cellcolor{yellow!30}$66.8 \pm 2.2$ & \cellcolor{yellow!30}$66.8 \pm 2.2$\\
& $|C_1|$ & $55.8 \pm 0.6$ & \cellcolor{green!30}$56.0 \pm 0.0$ & \cellcolor{green!30}$56.0 \pm 0.0$ & \cellcolor{green!30}$62.0 \pm 0.0$ & \cellcolor{green!30}$56.0 \pm 0.0$\\
\midrule
BBCSports & ARI & $90.1 \pm 2.9$ & \cellcolor{green!30}$78.0 \pm 8.2$ & \cellcolor{green!30}$75.2 \pm 8.5$ & \cellcolor{green!30}$63.3 \pm 7.6$ & \cellcolor{yellow!30}$87.0 \pm 6.3$\\
& $|C_1|$ & $799.2 \pm 7.6$ & \cellcolor{green!30}$860.6 \pm 28.1$ & \cellcolor{green!30}$1218.8 \pm 14.5$ & \cellcolor{green!30}$1953.7 \pm 5.9$ & \cellcolor{red!30}$720.8 \pm 14.7$\\
\midrule
BBCNews & ARI & $89.9 \pm 0.9$ & \cellcolor{yellow!30}$90.1 \pm 0.7$ & \cellcolor{yellow!30}$89.6 \pm 0.7$ & \cellcolor{yellow!30}$89.5 \pm 0.4$ & \cellcolor{yellow!30}$90.1 \pm 0.7$\\
& $|C_1|$ & $1379.0 \pm 6.2$ & \cellcolor{green!30}$1437.4 \pm 4.5$ & \cellcolor{green!30}$1758.4 \pm 2.3$ & \cellcolor{green!30}$1984.8 \pm 2.3$ & \cellcolor{red!30}$1367.4 \pm 6.1$\\
\midrule
WebKB & ARI & $31.5 \pm 3.2$ & \cellcolor{green!30}$24.9 \pm 1.1$ & \cellcolor{yellow!30}$33.8 \pm 3.5$ & \cellcolor{yellow!30}$32.6 \pm 2.7$ & \cellcolor{yellow!30}$30.9 \pm 2.5$\\
& $|C_1|$ & $1314.9 \pm 18.4$ & \cellcolor{green!30}$1407.9 \pm 42.5$ & \cellcolor{green!30}$1854.2 \pm 6.8$ & \cellcolor{green!30}$1936.2 \pm 4.6$ & \cellcolor{red!30}$1282.7 \pm 19.2$\\
\midrule
Reuters & ARI & $67.8 \pm 1.6$ & \cellcolor{yellow!30}$65.7 \pm 5.1$ & \cellcolor{yellow!30}$67.2 \pm 3.3$ & \cellcolor{yellow!30}$65.4 \pm 4.2$ & \cellcolor{yellow!30}$68.3 \pm 1.3$\\
& $|C_1|$ & $1321.9 \pm 7.9$ & \cellcolor{green!30}$1390.7 \pm 29.1$ & \cellcolor{green!30}$1914.6 \pm 2.1$ & \cellcolor{green!30}$1976.0 \pm 1.7$ & \cellcolor{red!30}$1213.3 \pm 11.7$\\
\midrule
20NewsG & ARI & $23.0 \pm 1.3$ & \cellcolor{yellow!30}$23.2 \pm 1.7$ & \cellcolor{yellow!30}$22.4 \pm 1.0$ & \cellcolor{yellow!30}$22.5 \pm 0.9$ & \cellcolor{yellow!30}$22.8 \pm 2.0$\\
& $|C_1|$ & $1453.7 \pm 15.3$ & \cellcolor{yellow!30}$1456.7 \pm 13.1$ & \cellcolor{green!30}$1990.5 \pm 0.8$ & \cellcolor{green!30}$1999.0 \pm 0.0$ & \cellcolor{green!30}$1564.4 \pm 10.4$\\
\midrule
MouseAtlas & ARI & $56.3 \pm 6.2$ & \cellcolor{yellow!30}$55.4 \pm 5.0$ & \cellcolor{yellow!30}$56.3 \pm 6.2$ & \cellcolor{yellow!30}$56.3 \pm 6.2$ & \cellcolor{yellow!30}$56.2 \pm 6.2$\\
& $|C_1|$ & $14566.0 \pm 58.8$ & \cellcolor{green!30}$14753.0 \pm 14.8$ & \cellcolor{green!30}$14650.3 \pm 26.2$ & \cellcolor{green!30}$14736.9 \pm 21.3$ & \cellcolor{yellow!30}$14550.3 \pm 60.4$\\
\midrule
GeneExp & ARI & $99.1 \pm 0.1$ & \cellcolor{green!30}$98.7 \pm 0.1$ & \cellcolor{yellow!30}$99.1 \pm 0.1$ & \cellcolor{green!30}$98.8 \pm 0.1$ & \cellcolor{yellow!30}$99.0 \pm 0.2$\\
& $|C_1|$ & $6445.1 \pm 3.1$ & \cellcolor{green!30}$7652.8 \pm 3.0$ & \cellcolor{green!30}$7718.8 \pm 3.0$ & \cellcolor{green!30}$14135.0 \pm 0.0$ & \cellcolor{green!30}$8065.6 \pm 5.3$\\
\midrule
HDendritic & ARI & $79.6 \pm 5.8$ & \cellcolor{yellow!30}$83.4 \pm 5.2$ & \cellcolor{yellow!30}$79.6 \pm 5.8$ & \cellcolor{yellow!30}$79.6 \pm 5.7$ & \cellcolor{yellow!30}$79.6 \pm 5.7$\\
& $|C_1|$ & $16672.9 \pm 406.2$ & \cellcolor{green!30}$23694.0 \pm 388.1$ & \cellcolor{green!30}$17599.1 \pm 393.0$ & \cellcolor{green!30}$20094.5 \pm 274.9$ & \cellcolor{yellow!30}$16574.9 \pm 343.4$\\
\bottomrule
\end{tabular}}
\end{table*}

\subsection{Estimating the Number of Clusters}\label{sec:estimate_k}

In many unsupervised scenarios, the number of clusters $K$ is unknown to the user. Therefore, it is a significant advantage if an algorithm gives guidance on how to define this value. Similarly to our penalty term $b({\bf x}|K)$, we define an MDL-based penalty term to describe different parameterizations for $K$: $v(K)=L^0(K)+n\log K$, where $L^0(\cdot)$ is the universal prior for integers \citep{rissanen1983universal}. This term encodes the number of clusters by $L^0(K)$ and the cluster assignments by $n\log K$. Note that the cluster-specific column values are already encoded by $b({\bf x}|K)$. The penalty $v(K)$ is added to Eq. \eqref{eq:likelihood} to prevent the model from overfitting when increasing $K$.

\begin{figure}[t]
    \centering
    \begin{subfigure}{0.3\textwidth}
        \centering
        \includegraphics[width=\textwidth]{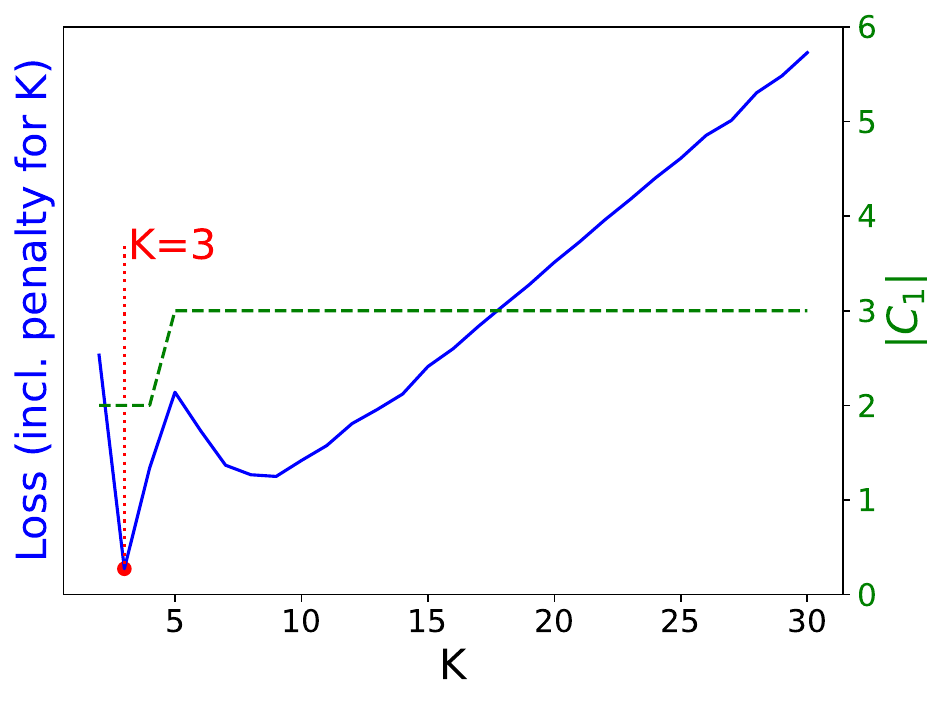}
        \caption{Synth ($K_{gt}=3$).}
    \end{subfigure}
    \begin{subfigure}{0.295\textwidth}
        \centering
        \includegraphics[width=\textwidth]{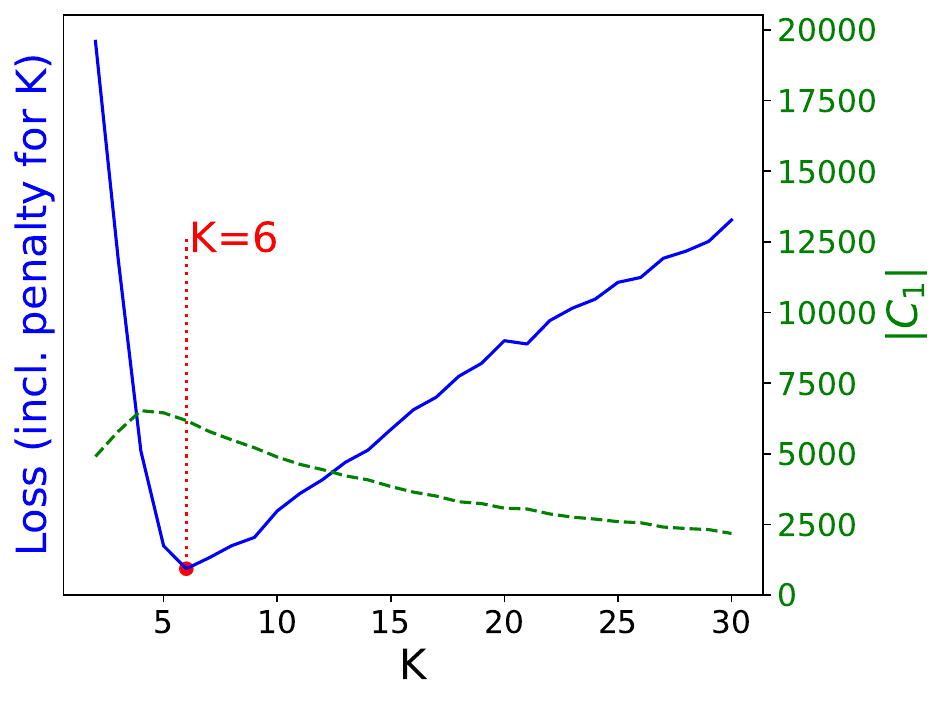}
        \caption{GeneExp ($K_{gt}=5$).}
    \end{subfigure}
    \begin{subfigure}{0.3\textwidth}
        \centering
        \includegraphics[width=\textwidth]{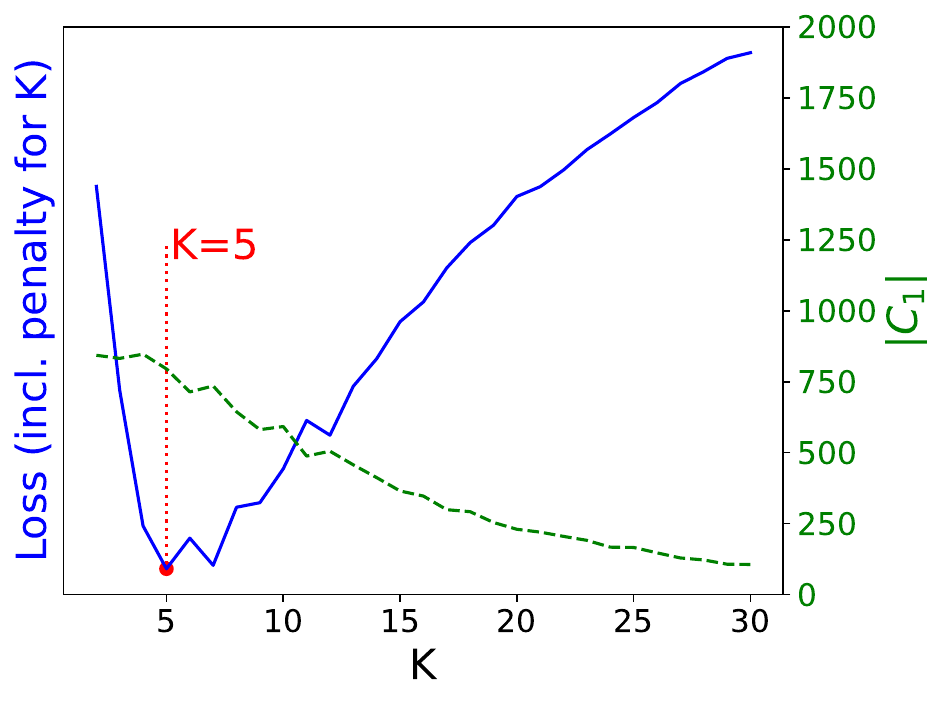}
        \caption{BBCSports ($K_{gt}=5$).}
    \end{subfigure}
    \caption{Loss of \Method~ (incl. the penalty term $v(K)$ -- blue line) and number of columns in $C_1$ (green line) for varying $K$. Each entry corresponds to the best loss after $20$ runs. The red line marks the $K$ for which the minimum loss occurs.}
    \label{fig:kEstimation}
\end{figure}

Fig. \ref{fig:kEstimation} shows the extended loss, i.e., ${\cal L}_K({\bf X})={\cal L}({\bf X})+v(K)$, and the number of columns in $C_1$ of \Method for Synth, GeneExp, and BBCSports using $K\in\{2,3,\dots,30\}$. To receive more reliable results, each entry corresponds to the best result after $20$ runs. 

We notice that the number of columns in $C_1$ decreases with increasing $K$ as $b({\bf x}|K)$ depends on $K$. Yet, \Method identifies the ground truth $K$ for Synth and BBCSports and is only off by one for GeneExp. When analyzing the identified clusters regarding GeneExp, we see that \Method splits the largest ground truth cluster with $300$ samples into two clusters of size $65$ and $235$, leading to an ARI of $87.9$. This behavior potentially offers additional insights for domain experts. In the case of Synth, we can see a conspicuous bump in the loss at $K=9$ due to the fifth column containing uniformly distributed noise. Here, \Method subdivides each ground truth cluster into three additional clusters, e.g., by dividing the noise into ranges $[0,5]$, $[6,10]$, and $[11,15]$. 

Appendix~\ref{sec:estimate_k_kmeans} provides results (estimated $K$, ARI, Purity) for all data sets, discusses $k$-Means-based $k$-estimation techniques and states results regarding the number of clusters as identified by TauCC. In general, our experiments show that \Method can obtain high quality results without knowing the ground truth number of clusters when dealing with high-dimensional text data sets but struggles with certain tabular data.

\section{Limitations}

Despite its robust clustering performance and options for interpretability, \Method possesses certain limitations that provide avenues for future research. First, the model assumes a Poisson-based structure. While effective for many count data scenarios, this may be sub-optimal for data sets exhibiting significant overdispersion or zero-inflation beyond the Poisson mean-variance identity. Extending the methodologies to other distributions such as the negative binomial would enhance its versatility. Second, the current definition of the penalty term for column selection, $b({\bf x}|K)$, and the encoding of outliers are subject to heuristic choices. As the number of clusters $K$ currently influences feature selection, the model may converge to undesired local optima in specific high-dimensional settings. Finally, another issue is that \Method requires the number of clusters to be known apriori. While we propose a heuristic to estimate the number of clusters for \Method, a more cohesive integration -- akin to the iterative splitting approach of $X$-Means \citep{pelleg2000x} -- could further improve the quality and stability of the process, increasing the applicability in real-world scenarios where the number of clusters is often unknown.

\section{Discussion and Conclusions}\label{sec:discussion}

We introduced \Method, a principled approach for clustering count data with integrated column selection that combines theoretical rigor with practical efficiency.
The method outputs a hard clustering of rows and simultaneously finds a subset of columns that are relevant for clustering. We found that \Method outperforms the compared methods in many scenarios. Notably, our approach often beats tailored bag-of-words text representations for text data, indicating its general applicability. 
Beyond competitive performance on standard benchmarks, \Method shows a superior performance in noisy settings. The integration of a Poisson-based noise model enables \Method to outperform clustering methods using normalized versions of the count matrix, e.g., by applying RCA. 
Our outlier detection feature further improves clustering quality by highlighting samples that do not fit the underlying clustering assumptions. This opens up interesting possibilities for future work, for example, by refining the outlier detection mechanism and integrating an automatic estimation of the number of clusters. 

\backmatter

\bmhead{Acknowledgements}
 We thank the Research Council of Finland (decisions 364226, 368654) and the Technology Industries of Finland Centennial Foundation for funding. We thank the Finnish Computing Competence Infrastructure (FCCI) for computational and data storage resources.

\section*{Declarations}
The authors have no competing interests to declare.

\bibliography{dami2025}

\newpage

\begin{appendices}

\section{A Poisson-based Clustering Model}\label{app:a}

The Poisson distribution is a natural foundation when working with count data. The probability mass function regarding a discrete random variable ${\bf X}_{ij}$ is defined as
\begin{align}
\label{eq:poisson}
    p({\bf X}_{ij} \mid \mu_{ij}) = \frac{\mu_{ij}^{{\bf X}_{ij}}}{({\bf X}_{ij})!}e^{-\mu_{ij}},
\end{align}
where $\mu_{ij}$ are the expected number of occurrences.
Since we assume that all entries have been drawn independently, the likelihood of a given $n \times m$ data matrix ${\bf X}$ is given by
\begin{align}
\label{eq:poisson_likelihood}
p({\bf X}\mid \mu)=\prod_{i=1}^n{\prod_{j=1}^m{p({\bf X}_{ij}\mid \mu_{ij})}}.
\end{align}
If we do not restrict the parameters $\mu_{ij}$, the solution is trivially to set the parameters equal to the observations or $\mu_{ij}={\bf X}_{ij}$. However, this will lead to overfitting because the number of parameters equals the number of observations.  Therefore, we need some simplifying modeling assumptions. We follow the definition given, e.g., in \citep{cai2004clustering,kim2007measuring}, where the mean is a product of row- and column-specific factors 
\begin{align}\label{eq:RC}
\mu_{ij}=\lambda_i^R\lambda_j^C,
\end{align}
where $\lambda_i^R\in{\mathbb{R}}_{>0}$ and $\lambda_j^C\in{\mathbb{R}}_{>0}$. This concept is closely related to \textit{Poisson Factorization}~\citep{gopalan2015scalabla}. We can compute optimal values for the parameters $\lambda_i^R$ and $\lambda_j^C$ so that the likelihood of Eq. \eqref{eq:poisson_likelihood} is maximized by applying Maximum Likelihood Estimation (MLE). Since $\text{argmax}_{\mu \in {\mathbb{R}}_{>0}} p({\bf X}\mid \mu)= \text{argmax}_{\mu \in{\mathbb{R}}_{>0}} \log (p({\bf X}\mid \mu))$ our objective can be formulated as to maximizing the following function
\begin{align}
\begin{split}
    \log \big(p({\bf X} \mid \mu)\big) & = \log \left(\prod_{i=1}^{n}\prod_{j=1}^{m} \frac{(\lambda_{i}^R\lambda_{j}^C)^{{\bf X}_{ij}}}{({\bf X}_{ij})!}e^{-(\lambda_{i}^R\lambda_{j}^C)}\right)\\
    & = \sum_{i=1}^{n}\sum_{j=1}^{m} {\bf X}_{ij} \log (\lambda_{i}^R) + {\bf X}_{ij} \log (\lambda_{j}^C)  - (\lambda_{i}^R\lambda_{j}^C)-\log(({\bf X}_{ij})!).
\end{split}
\end{align}
Setting the first derivative regarding $\lambda_{i}^R$ to zero yields
\begin{align}
\begin{split}
    0 &\stackrel{!}{=} \frac{\partial}{\partial\lambda_{i}^R}\log \big( p({\bf X}\mid \mu) \big) = \sum_{j=1}^{m} \frac{{\bf X}_{ij}}{\lambda_{i}^R} - \lambda_{j}^C\\
    \Leftrightarrow \lambda_{i}^R & = \frac{1}{\sum_{j=1}^{m} \lambda_{j}^C} \sum_{j=1}^{m} {\bf X}_{ij}.\label{eq:poisson_lambda_i}
\end{split}
\end{align}
Analogously, when considering the derivative regarding $\lambda_{j}^C$, we obtain
\begin{align}
    \lambda_{j}^C & = \frac{1}{\sum_{i=1}^{n} \lambda_{i}^R} \sum_{i=1}^{n} {\bf X}_{ij}.
\end{align}
Note that for any constant $\kappa\in{\mathbb{R}}_{>0}$, we can divide $\lambda_i^R\leftarrow\lambda_i^R/\kappa$ and multiply $\lambda_j^C\leftarrow\kappa\lambda_j^C$ without affecting the parameters $\mu_{ij}$. Therefore, w.l.o.g. we can set $\sum_{j=1}^{m} \lambda_{j}^C = 1$. This yields
\begin{align}
    \label{eq:lambdas}
    \lambda_{i}^R  = \sum_{j=1}^{m} {\bf X}_{ij} = r_i \qquad \text{and} \qquad
    \lambda_{j}^C  = \frac{1}{\sum_{i=1}^{n} r_{i}} \sum_{i=1}^{n} {\bf X}_{ij} = \frac{c_j}{\sum_{i=1}^{n} r_{i}} = \frac{c_j}{N}
\end{align}
and finally
\begin{align}
    \mu_{ij} & = \frac{r_{i} c_{j}}{\sum_{i'=1}^{n} r_{i'}} = \frac{r_{i} c_{j}}{\sum_{j'=1}^{m} c_{j'}}=\frac{r_{i} c_{j}}{N}.
\end{align}

\noindent\textbf{Standardization using the Poisson Model.} The described Poisson model can be used to apply a Poisson-based standardization (or z-normalization) to a data matrix ${\bf X}$. We show how this relates to Chi-squared and the Revealed Comparative Advantage (RCA) \citep{balassa1965Trade}. 

The standard score of an entry ${\bf X}_{ij}$ is calculated as
\begin{align}
    \text{STD}({\bf X}_{ij}) = \frac{{\bf X}_{ij} - \mu_{ij}}{\sigma_{ij}}.
\end{align}
Considering that in the case of the Poisson distribution $\mu=\sigma^2$, the Poisson-based standard score can be formulated as
\begin{align}
\label{eq:pstd}
    \text{P-STD}({\bf X}_{ij}) = \frac{{\bf X}_{ij} - \mu_{ij}}{\sqrt{\mu_{ij}}}.
\end{align}
Note that this formulation is equal to the square root of Pearson's Chi-squared test, i.e., P-STD$({\bf X}_{ij}) = \sqrt{\mathcal{X}^2({\bf X}_{ij})}$.

Furthermore, we can consider a generalized version of P-STD with a weighting factor $\alpha$ for the denominator (which is set to $0.5$ in Eq. \eqref{eq:pstd}).
\begin{align}
     \text{P-STD}({\bf X}_{ij}, \alpha) = \frac{{\bf X}_{ij} - \mu_{ij}}{(\mu_{ij})^\alpha}
\end{align}
If $\alpha=1$, this formulation is equal to the RCA up to a constant of $-1$, i.e.,
\begin{align}
     \text{P-STD}({\bf X}_{ij}, 1) = \frac{{\bf X}_{ij} - \mu_{ij}}{(\mu_{ij})^1}=\frac{{\bf X}_{ij}N}{r_ic_j} - 1= \text{RCA}({\bf X}_{ij})-1.
\end{align}
As P-STD with $\alpha=0.5$ builds on a solid statistical model that takes into account increasing variances for larger expected values, we argue that it is more robust with regard to noisy data than RCA. 

\subsection{PoissonL/PoissonC: Poisson-based Clustering}
\label{appendix:poissonl}
The described Poisson model was used to propose the PoissonL and PoissonC algorithms \citep{cai2004clustering} for clustering gene expression data. However, as shown in our experiments, they are not only applicable to gene expression data but also more generally to count data. The objective is to divide the data into $K\in \mathbb{N}_{>0}$ sets of rows $R_k$ that show a similar relationship between the columns and, therefore, can be modeled using the described Poisson model by representing each set of rows $R_k$ using a specific $\lambda_{j|k}^C \in{\mathbb{R}}_{>0}$. This can be formulated as minimizing the loss function 
\begin{align}
\begin{split}
    \label{eq:rewardPoissonL}
    \mathcal{L}({\bf X}) &=-\sum_{k=1}^K\sum_{i \in R_k}\sum_{j=1}^{m} \log(p({\bf X}_{ij} \mid \lambda_i^R \lambda_{j|k}^C))\\
    & = -\sum_{k=1}^K\sum_{i \in R_k}\sum_{j=1}^{m} {\bf X}_{ij} \log (\lambda_{i}^R) + {\bf X}_{ij} \log (\lambda_{j|k}^C)  - \lambda_{i}^R\lambda_{j|k}^C-\log(({\bf X}_{ij})!).
    \end{split}
\end{align}
As we are only optimizing with respect to $\lambda_{j|k}^C$ and $R_k$, this can be simplified to
\begin{align}
\label{eq:poissonL_Loss}
     \mathcal{L}({\bf X}) &=-\sum_{k=1}^K\sum_{i \in R_k}\sum_{j=1}^{m} {\bf X}_{ij} \log (\lambda_{j|k}^C)  - \lambda_{i}^R\lambda_{j|k}^C+\text{const}.
\end{align}
Using this definition, we perform an iterative approach to update the cluster assignments $R_k$ and cluster-specific $\lambda_{j|k}^C$ values. 

We start with initial cluster assignments, e.g., by random initialization or by utilizing more sophisticated initialization strategies like the one proposed in Sect. \ref{sec:initial}, which builds on $k$-Means{+}{+} \citep{Arthur07kmeans}. Afterward, we repeat two steps until convergence. First, we compute the cluster-specific column values
\begin{align}
    \label{eq:lambdaCUpdate}
    \lambda_{j|k}^C = \frac{1}{\sum_{i \in R_k} r_{i}} \sum_{i \in R_k} {\bf X}_{ij},
\end{align}
where the column parameters are scaled such that $\forall_{k \in [K]}~\sum_{j=1}^m \lambda_{j|k}^C=1$ and $\lambda_{i}^R$ is defined as in Eq. \eqref{eq:lambdas}. 
Second, we update the cluster assignments by choosing the cluster that gives the highest log-probability, considering the new column values
\begin{align}
    \label{eq:assignmentUpdate}
    R_k=\{i \in [n] \mid k = \text{argmax}_{k'\in [K]} \sum_{j=1}^{m} {\bf X}_{ij} \log (\lambda_{j|k'}^C)  - \lambda_{i}^R\lambda_{j|k'}^C\}.
\end{align}
PoissonC \citep{cai2004clustering} uses an alternative strategy to update the cluster assignments by choosing the cluster that gives the minimum Chi-squared value.
The complete process is given in Algorithm \ref{algo:poissonmeans}.

\begin{algorithm2e}[t]
	\SetAlgoVlined
	\DontPrintSemicolon
	\KwIn{data set ${\bf X}$, number of clusters $K$}
	\KwOut{the cluster assignments $R_k$, the cluster-specific column values $\lambda_{j|k}^C$}
    // Initialization\;
    $\lambda^R_i \gets \sum_{j=1}^m {\bf X}_{ij}$\;
    $R_1, \dots, R_K \gets $ initialize cluster assignments (e.g., by $k$-Means{+}{+})\;
    \While{any $R_k$ changed in last iteration}{
        // Update lambdas\;
        \For{$k \in [K]$}{
            \For{$j \in [m]$}{
                $\lambda_{j|k}^C \gets \frac{1}{\sum_{i \in R_k} r_{i}} \sum_{i \in R_k} {\bf X}_{ij}$\;
            }
        }
        // Update cluster assignments\;
        $R_1, \dots, R_K \gets \emptyset$\;
        \For{$i \in [n]$}{
            \uIf{Algorithm is PoissonL}{
                $k \gets \text{argmax}_{k'\in [K]} \sum_{j=1}^{m} {\bf X}_{ij} \log (\lambda_{j|k'}^C)  - \lambda_{i}^R\lambda_{j|k'}^C$\;
            }
            \uElseIf{Algorithm is PoissonC}{
                $k \gets \text{argmin}_{k'\in [K]} \sum_{j=1}^{m} \frac{({\bf X}_{ij} - \lambda_{i}^R\lambda_{j|k'}^C)^2}{\lambda_{i}^R\lambda_{j|k'}^C}$\;
            }
            $R_k \gets R_k \cup \{i\}$\;
        }
    }
	\Return{$R_k, \lambda_{j|k}^C$}
	\caption{The PoissonL/PoissonC algorithms}
 \label{algo:poissonmeans}
\end{algorithm2e}

\subsection{\Method: Poisson-based Subspace Clustering}
\label{sec:3cpo_appendix}

The PoissonL and PoissonC algorithms have two main \textbf{assumptions} regarding the data matrix $\bf X$:
\begin{enumerate}
    \item All entries within row $i$ are scaled by the row-specific value $\lambda^R_i$. \label{item:scaling}
    \item The values in cluster $k$ all follow the proportions as defined in the column-specific values $\lambda^C_{j|k}$ and are different to other clusters. \label{item:clusterProportions}
\end{enumerate}
If one of these properties is violated, the quality of the clustering result can suffer greatly. Therefore, our proposed algorithm \Method automatically discovers noisy columns and defines a subspace that only contains columns that fulfill the mentioned assumptions. This also has the advantage that the resulting clustering solution is easier to interpret, as only a subset of the columns has to be analyzed. \Method identifies relevant columns by optimizing the following loss function
\begin{align}
\begin{split}
    \mathcal{L}({\bf X})= -\sum\nolimits_{i = 1}^n & \sum\nolimits_{j=1}^m {\log{p({\bf X}_{ij},\mu_{ij})}}  +  \sum_{j\in C_1}{b({\bf X}_{\cdot j}|K)}\\
    =- \sum_{k=1}^K \sum_{i\in R_k} \Bigg(
    &\overbrace{\sum_{j \in C_-} \log \big(p({\bf X}_{ij} \mid \lambda_{j|-}^C)\big)}^\text{columns violating assumption (\ref{item:scaling})}
    + \overbrace{\sum_{j \in C_0} \log \big(p({\bf X}_{ij} \mid \lambda_i^R \lambda_{j|0}^C) \big)}^\text{columns violating assumption (\ref{item:clusterProportions})}\\
    + &\underbrace{\sum_{j \in C_1} \log \big(p({\bf X}_{ij} \mid \lambda_i^R \lambda_{j|k}^C) \big)}_\text{relevant columns for clustering}
    \Bigg) 
    + \underbrace{\sum_{j \in C_1} b({\bf X}_{\cdot j}|K),}_\text{column selection penalty term}
\end{split}
\end{align}
where $C_-$, $C_0$ and $C_1$ are non-overlapping partitions of the $m$ columns so that $C_- \cup C_0 \cup C_1 = [m]$ and $b({\bf X}_{\cdot j}|K)$ is a penalty term explained in Sect.~\ref{sec:column}. Note that columns that do not follow assumption (\ref{item:scaling}) are contained in $C_-$ and do not use the row-specific values $\lambda_i^R$ to calculate the expected value.

Given an input matrix ${\bf Y}$, we set ${\bf X}_{ij} = {\bf Y}_{ij} + \alpha - 1$, as defined in the main paper. Considering constant terms that are not relevant for the optimization yields
\begin{align}
\begin{split}
    \mathcal{L}({\bf X})= -\sum_{k=1}^K \sum_{i\in R_k} \Bigg( 
    &\sum_{j \in C_-} {\bf X}_{ij}\log(\lambda_{j|-}^C)-\lambda_{j|-}^C
    + \sum_{j \in C_0} {\bf X}_{ij} \log(\lambda_{j|0}^C)-\lambda_i^R \lambda_{j|0}^C\\
    +& \sum_{j \in C_1} {\bf X}_{ij} \log(\lambda_{j|k}^C)-\lambda_i^R \lambda_{j|k}^C
     + \sum_{j \in C_0 \cup C_1} {\bf X}_{ij} \log(\lambda_{i}^R)
    \Bigg) \\
    +& \sum_{j \in C_1} b({\bf X}_{\cdot j}|K) + \text{const.}.
\label{eq:lossSubspaceClustering}
\end{split}
\end{align}

The lambda parameters can be optimized again by performing MLE. This gives us
\begin{align}
    \label{eq:lambda_i_update_appendix}
    \lambda_{i}^R & = \frac{1}{\sum_{j \in C_1} \lambda_{j|k}^C + \sum_{j \in C_0}\lambda_{j|0}^C} \sum_{j \in C_0 \cup C_1} {\bf X}_{ij},\quad~~~{\rm if}~~~ i \in R_k,\\
    \lambda_{j|-}^C & = \frac{1}{n} \sum_{i = 1}^n {\bf X}_{ij},\\
    \lambda_{j|0}^C & = \frac{1}{\sum_{i =1}^n \lambda_{i}^R} \sum_{i = 1}^n {\bf X}_{ij},&\\
    \label{eq:lambda_jk_update_appendix}
    \lambda_{j|k}^C & = \frac{1}{\sum_{i \in R_k} \lambda_{i}^R} \sum_{i \in R_k} {\bf X}_{ij}.
\end{align}

Since $\lambda^R_i$ depends on $\lambda_{j|0}^C$ and $\lambda_{j|k}^C$, we cannot follow the strategy proposed before and simply set the sum of the cluster-specific column values equal to one. Since $\lambda_{j|0}$ is shared across all clusters, the factor $\kappa$ used for scaling $\lambda_i^R \gets \kappa \lambda_{i}^R$ and $\lambda_{j|l}^C \gets \frac{1}{\kappa}\lambda_{j|l}^C$ has to be equal for all $l\in \{0\}\cup [K]$ so that $\lambda_i^R\lambda_{j|l}^C=\mu_{ij}$ is fulfilled. However, w.l.o.g., we can scale $\sum_{i=1}^n\lambda^R_i=1$. It follows that
\begin{align}
    \lambda_{j|0}^C = \sum_{i=1}^n {\bf X}_{ij} \quad \Rightarrow \quad \sum_{j \in C_0} \lambda_{j|0}^C = \sum_{i=1}^n \sum_{j \in C_0} {\bf X}_{ij}.
\end{align}
Furthermore, we consider the following formulations (based on Eq. \eqref{eq:lambda_i_update_appendix} and \eqref{eq:lambda_jk_update_appendix}) regarding cluster $k$
\begin{align}
    \label{eq:sum_c1_lambda_jk_appendix}
    \sum_{j \in C_1} \lambda_{j|k}^C = \frac{\sum_{i \in R_k} \sum_{j \in C_1} {\bf X}_{ij}}{\sum_{i \in R_k} \lambda_i^R}
\end{align}
and
\begin{align}
\begin{split}
    \sum_{i \in R_k} \lambda_i^R =& \frac{\sum_{i \in R_k} \sum_{j\in C_0 \cup C_1} {\bf X}_{ij}}{\sum_{j \in C_1} \lambda_{j|k}^C + \sum_{j \in C_0} \lambda_{j|0}^C}\\
    \Rightarrow \sum_{j \in C_1} \lambda_{j|k}^C =& \frac{\sum_{i \in R_k} \sum_{j\in C_0 \cup C_1} {\bf X}_{ij}}{\sum_{i \in R_k} \lambda_i^R} - \sum_{j \in C_0} \lambda_{j|0}^C.
\end{split}
\end{align}
Inserting Eq. \eqref{eq:sum_c1_lambda_jk_appendix} leads to
\begin{align}
\begin{split}
    \frac{\sum_{i \in R_k} \sum_{j \in C_1} {\bf X}_{ij}}{\sum_{i \in R_k} \lambda_i^R} = &\frac{\sum_{i \in R_k} \sum_{j\in C_0 \cup C_1} {\bf X}_{ij}}{\sum_{i \in R_k} \lambda_i^R} - \sum_{j \in C_0} \lambda_{j|0}^C\\
    \Rightarrow \sum_{i \in R_k} \lambda_i^R = &\frac{\sum_{i \in R_k} \sum_{j\in C_0 \cup C_1} {\bf X}_{ij}- \sum_{i \in R_k} \sum_{j \in C_1} {\bf X}_{ij}}{\sum_{j \in C_0}  \lambda_{j|0}^C}\\
    \Rightarrow \sum_{i \in R_k} \lambda_i^R = &\frac{\sum_{i \in R_k} \sum_{j\in C_0} {\bf X}_{ij}}{\sum_{i=1}^n \sum_{j \in C_0} {\bf X}_{ij}}.
\end{split}
\end{align}
This formulation can be used to compute $\lambda_{j|k}^C$ as
\begin{align}
    \lambda_{j|k}^C = \frac{(\sum_{i \in R_k} {\bf X}_{ij}) (\sum_{i=1}^n \sum_{j \in C_0} {\bf X}_{ij})}{\sum_{i \in R_k} \sum_{j\in C_0} {\bf X}_{ij}}.
\end{align}
Lastly, for $i \in R_k$ this yields
\begin{align}
\begin{split}
    \lambda_i^R =&\frac{\sum_{j\in C_0 \cup C_1} {\bf X}_{ij}}{\sum_{j \in C_1} \frac{(\sum_{i \in R_k} {\bf X}_{ij}) (\sum_{i=1}^n \sum_{j \in C_0} {\bf X}_{ij})}{\sum_{i \in R_k} \sum_{j\in C_0} {\bf X}_{ij}} + \sum_{i=1}^n \sum_{j \in C_0} {\bf X}_{ij}} \\
    =&\frac{(\sum_{j\in C_0 \cup C_1} {\bf X}_{ij})(\sum_{i \in R_k} \sum_{j\in C_0} {\bf X}_{ij})}{(\sum_{i=1}^n \sum_{j \in C_0} {\bf X}_{ij}) (\sum_{i \in R_k} \sum_{j \in C_1} {\bf X}_{ij} + \sum_{i \in R_k} \sum_{j\in C_0} {\bf X}_{ij})} \\
    =&\frac{(\sum_{j\in C_0 \cup C_1} {\bf X}_{ij}) (\sum_{i \in R_k} \sum_{j\in C_0} {\bf X}_{ij})}{(\sum_{i=1}^n \sum_{j \in C_0} {\bf X}_{ij})(\sum_{i \in R_k} \sum_{j \in C_1 \cup C_0} {\bf X}_{ij})}.
\end{split}
\end{align}

When using these definitions, we have to make sure that $C_0 \neq \emptyset$, because the denominator in the formulation of $\lambda_i^R$ and $\lambda_{j|k}^C$ would be zero if $C_0 = \emptyset$. In this case, we can set $\sum_{i \in R_k} \lambda_i^R = \frac{\sum_{i \in R_k}\sum_{j \in C_1}{\bf X}_{ij}}{\sum_{i=1}^n\sum_{j \in C_1}{\bf X}_{ij}}$ for all $k \in [K]$ as the $\lambda_{j|k}^C$ parameters are independent. This fulfills the assumption $\sum_{i=1}^n \lambda_i^R=1$ and leads to
\begin{align}
    \lambda_{j|k}^C=\frac{\sum_{i=1}^n\sum_{j \in C_1}{\bf X}_{ij}}{\sum_{i \in R_k}\sum_{j \in C_1}{\bf X}_{ij}} \sum_{i \in R_k} {\bf X}_{ij},\quad~~~{\rm if}~~~ C_0 = \emptyset
\end{align}
and 
\begin{align}
    \lambda_{i}^R=\frac{1}{\sum_{i=1}^n\sum_{j \in C_1}{\bf X}_{ij}}\sum_{j \in C_1} {\bf X}_{ij},\quad~~~{\rm if}~~~ C_0 = \emptyset.
\end{align}

These results can be summarized as
\begin{align}\label{eq:summary_lambda_appendix}
\lambda^C_{j\mid -}&=\sum\nolimits_{i=1}^n{{\bf X}_{ij}/n},\\
\lambda^C_{j\mid 0}&=\sum\nolimits_{i=1}^n{{\bf X}_{ij}},\\
\lambda^C_{j\mid k}&=\begin{cases}
    \frac{(\sum_{i \in R_k} {\bf X}_{ij}) (\sum_{i=1}^n \sum_{j \in C_0} {\bf X}_{ij})}{\sum_{i \in R_k} \sum_{j\in C_0} {\bf X}_{ij}},&~~~{\rm if}~~~ C_0 \neq \emptyset\\
    \frac{(\sum_{i \in R_k} {\bf X}_{ij}) (\sum_{i=1}^n \sum_{j \in C_1} {\bf X}_{ij})}{\sum_{i \in R_k} \sum_{j\in C_1} {\bf X}_{ij}},&~~~{\rm if}~~~ C_0 = \emptyset
\end{cases}\\
\lambda^R_i&=\begin{cases}
\frac{(\sum_{j\in C_0 \cup C_1} {\bf X}_{ij}) (\sum_{i \in R_k} \sum_{j\in C_0} {\bf X}_{ij})}{(\sum_{i=1}^n \sum_{j \in C_0} {\bf X}_{ij})(\sum_{i \in R_k} \sum_{j \in C_1 \cup C_0} {\bf X}_{ij})},&~~~{\rm if}~~~ i \in R_k~{\rm and}~C_0 \neq \emptyset\\
\frac{\sum_{j \in C_1} {\bf X}_{ij}}{\sum_{i=1}^n \sum_{j \in C_1} {\bf X}_{ij}},&~~~{\rm if}~~~ i \in R_k~{\rm and}~C_0 = \emptyset.
\end{cases}
\end{align}

\subsection{Scale Invariance}\label{sec:scale_invariance_appendix}

We observe that the scaling of the matrix ${\bf X}$ does not matter if $\forall_{j\in C_1}{b({\bf X}_{\cdot j}|K)}=0$.
\begin{theorem}
If we fix $b({\bf X}_{\cdot j}|K)=0$ the clustering $R_1,\ldots,R_K$ and the column partition is unchanged if we multiply ${\bf X}\rightarrow\alpha{\bf X}$ by a constant $\alpha\in{\mathbb{R}}_{>0}$.
\end{theorem}

\begin{proof}
Assume the iterations described in Sects. \ref{sec:em1}, \ref{sec:em2}, and \ref{sec:column} have converged. If we transform ${\bf X}\leftarrow\alpha{\bf X}$ and continue the iterations, $\lambda^R_i$ defined by Eq. \eqref{eq:lambdaR} remains unchanged. The parameters $\lambda^C_{j\mid -}\leftarrow\alpha\lambda^C_{j\mid -}$ (Eq. \eqref{eq:meanminus}), $\lambda^C_{j\mid 0}\leftarrow\alpha\lambda^C_{j\mid 0}$ (Eq. \eqref{eq:mean0}), and $\lambda^C_{j\mid k}\leftarrow\alpha\lambda^C_{j\mid k}$ (Eq. \eqref{eq:mean1}) are each multiplied by $\alpha$. Then $S_{i}^k$ of Eq. \eqref{eq:Sik} becomes $S_{i}^k\leftarrow\alpha S_{i}^k+\sum\nolimits_{j=1}^m{\alpha{\bf X}_{ij}\log{\alpha}}$. The clustering remains unchanged because the transformation does not change the ordering of $S_{i}^k$ for a given $i$ (the multiplication by $\alpha$ does not change the ordering and the term $\sum\nolimits_{j=1}^m{\alpha{\bf X}_{ij}\log{\alpha}}$ is constant with respect to $k$). Similarly, the terms for the column partition change as $L_j(-)\leftarrow\alpha L_j(-)+\sum\nolimits_{i=1}^n{\alpha{\bf X}_{ij}\log{\alpha}}$ (Eq. \eqref{eq:Lminus}), 
$L_j(0)\leftarrow\alpha L_j(0)+\sum\nolimits_{i=1}^n{\alpha{\bf X}_{ij}\log{\alpha}}$ (Eq. \eqref{eq:L0}), and 
$L_j(1)\leftarrow\alpha L_j(1)+\sum\nolimits_{i=1}^n{\alpha{\bf X}_{ij}\log{\alpha}}$ (Eq. \eqref{eq:L1}).
Again, the column partition remains unchanged, since multiplication by $\alpha$ does not change ordering and because the term $\sum\nolimits_{i=1}^n{\alpha{\bf X}_{ij}\log{\alpha}}$  is the same for all $L_j(-)$, $L_j(0)$, and $L_j(1)$. Therefore, if we are at the local optimum and multiply ${\bf X}$ by a constant, the iterative algorithm does not change clustering and the column partitions.
\end{proof}
If the bias term is non-zero, as in Eq. \eqref{eq:Bx}, we have the scaling $L_j(1)\leftarrow\alpha L_j(1)+\sum\nolimits_{i=1}^n{\alpha{\bf X}_{ij}\log{\alpha}}-(K-1)\log{\alpha}$, breaking the invariance, slightly favouring solutions with less columns in $C_1$ for larger $\alpha>1$. 

\subsection{Poisson-based Distance Function}\label{sec:distance_function_appendix}

In many situations, it is desirable to have a general-purpose distance function for count data that can, for example, be used for initialization purposes (compare Sect. \ref{sec:initial}) or when conducting agglomerative clustering \citep{mullner2011modern}. To define such a distance function, we ignore the partitions of the columns, i.e., $C_1=[m]$ and $C_- = C_0 = \emptyset$. We start with the log-loss of Eq. \eqref{eq:poissonL_Loss}.
\begin{align}
     L_K({\bf X}) &=-\sum_{k=1}^K\sum_{i \in R_k}\sum_{j=1}^{m} {\bf X}_{ij} \log (\lambda_{i}^R \lambda_{j|k}^C)  - \lambda_{i}^R\lambda_{j|k}^C+\text{const}.\\
     &=-\sum_{k=1}^K\sum_{i \in R_k}\sum_{j=1}^{m} {\bf X}_{ij} \log (\frac{r_ic_{kj}}{N_k})  - \frac{r_ic_{kj}}{N_k}+\text{const}.,
\end{align}
where $r_i=\sum_{j=1}^m{\bf X}_{ij}$, $c_{kj}=\sum_{i\in R_k}{\bf X}_{ij}$ and $N_k = \sum_{i \in R_k}r_i$. If each row defines its own cluster, this results in 
\begin{align}
    L_n({\bf X})  &=-\sum_{i=1}^n\sum_{j=1}^{m} {\bf X}_{ij} \log ({\bf X}_{ij})  - {\bf X}_{ij}+\text{const}..
\end{align}
The distance function is then based on the difference between those two losses.
\begin{align}
    \Delta L({\bf X})=&L_K({\bf X})-L_n({\bf X})\\
    =&\sum_{k=1}^K\sum_{i \in R_k}\sum_{j=1}^{m} -{\bf X}_{ij} \log (\frac{r_ic_{kj}}{N_k})  + \frac{r_ic_{kj}}{N_k}+ {\bf X}_{ij} \log ({\bf X}_{ij})  - {\bf X}_{ij}\\
    =&\sum_{k=1}^K\sum_{i \in R_k}\sum_{j=1}^{m} {\bf X}_{ij} \log (\frac{{\bf X}_{ij}N_k}{r_ic_{kj}})  + \frac{r_ic_{kj}}{N_k} - {\bf X}_{ij}\\
    =&\sum_{k=1}^K\sum_{i \in R_k} \frac{r_i}{N_k} 
\left( \sum_{j=1}^m c_{kj} \right) - r_i +  \sum_{j=1}^{m} {\bf X}_{ij} \log (\frac{{\bf X}_{ij}N_k}{r_ic_{kj}}) \\
    =&\sum_{k=1}^K\sum_{i \in R_k} r_i - r_i + \sum_{j=1}^{m} {\bf X}_{ij} \log (\frac{{\bf X}_{ij}N_k}{r_ic_{kj}})
\end{align}
Finally, the distance between rows $a$ and $b$ is the difference of the log-losses, considering that the rows originate from two clusters (containing only themselves) or a shared cluster.
\begin{align}
    d(a,b)=&\Delta L_{ab}({\bf X})=\sum_{i \in \{a,b\}}\sum_{j=1}^{m} {\bf X}_{ij} \log (\frac{{\bf X}_{ij}N_{ab}}{r_i c_{(ab)j}}),
\end{align}
where, $c_{(ab)j}={\bf X}_{aj}+{\bf X}_{bj}$ and $N_{ab}=r_a+r_b$.

\subsection{Outlier Detection}\label{app:outlier}

When considering outliers during the optimization (as explained in Sect.~\ref{sec:outlier_detection}), the clustering objective from Eq.~\eqref{eq:lossSubspaceClustering} becomes

\begin{align}
\begin{split}
    \mathcal{L}({\bf X})=- \Bigg(
    &\sum_{i=1}^n \sum_{j \in C_-} {\bf X}_{ij}\log(\lambda_{j|-}^C)-\lambda_{j|-}^C
    + \sum_{i=1}^n\sum_{j \in C_0} {\bf X}_{ij} \log(\lambda_{j|0}^C)-\lambda_i^R \lambda_{j|0}^C\\
    +& \sum_{k=1}^K\sum_{i \in R_k}\sum_{j \in C_1} {\bf X}_{ij} \log(\lambda_{j|k}^C)-\lambda_i^R \lambda_{j|k}^C
     + \sum_{i=1}^n\sum_{j \in C_0 \cup C_1} {\bf X}_{ij} \log(\lambda_{i}^R)\\
    + &\sum_{i \in O} \sum_{j \in C_1} {\bf X}_{ij} \log(\lambda_{j|0}^C)-\lambda_i^R \lambda_{j|0}^C
    \Bigg) + \sum_{j \in C_1} b({\bf X}_{\cdot j}|K) + \text{const.},
\end{split}
\end{align}
where $O$ contains the indices of the outlier rows.
Performing MLE gives us the same update rules for $\lambda^C_{j|-}, \lambda^C_{j|k}$ as shown in Eq.~\eqref{eq:lambda_i_update_appendix}-\eqref{eq:lambda_jk_update_appendix}. Only the updating mechanisms for $\lambda_i^R$ and $\lambda_{j|0}^C$ have to be adapted
\begin{align}
    \lambda_{i}^R= & \begin{cases}
        \frac{1}{\sum_{j \in C_1} \lambda_{j|k}^C + \sum_{j \in C_0}\lambda_{j|0}^C} \sum_{j \in C_0 \cup C_1} {\bf X}_{ij}& \quad \text{, if $i \in R_k$}\\
        \frac{1}{\sum_{j \in C_0 \cup C_1} \lambda_{j|0}^C} \sum_{j \in C_0 \cup C_1} {\bf X}_{ij}&\quad \text{, if $i \in O$.}
    \end{cases}\\
    \lambda_{j|0}^C= & \begin{cases}
        \frac{1}{\sum_{i = 1}^n \lambda_i^R} \sum_{i=1}^n {\bf X}_{ij}& \quad \text{, if $j \in C_0$}\\
        \frac{1}{\sum_{i \in O} \lambda_i^R} \sum_{i \in O} {\bf X}_{ij}& \quad \text{, if $j\in C_1$.}
    \end{cases}
\end{align}
Since we do not want to optimize $\lambda_{j|0}$ for $j \in C_1$ for outliers (else, we would just add another cluster with $\lambda_{j|0}^C \widehat{=} \lambda_{j|K+1}^C$ for $j \in C_1$), we force $\lambda_{j|0}=\sum_{i=1}^n {\bf X}_{ij}$ for all $j \in C_0 \cup C_1$. It follows that
\begin{align}
    \lambda_i^R=\begin{cases}
        \frac{(\sum_{j\in C_0 \cup C_1} X_{ij}) (\sum_{i \in R_k} \sum_{j\in C_0} X_{ij})}{(\sum_{i=1}^n \sum_{j \in C_0} X_{ij})(\sum_{i \in R_k} \sum_{j \in C_1 \cup C_0} X_{ij})}& \quad \text{, if $i \in R_k$}\\
         \frac{1}{\sum_{i=1}^n \sum_{j \in C_0 \cup C_1} {\bf X}_{ij}} \sum_{j \in C_0 \cup C_1} {\bf X}_{ij}&\quad \text{, if $i \in O$.}
    \end{cases}
\end{align}

Extending the summary from Eq. \eqref{eq:summary_lambda_appendix} gives us
\begin{align}
\lambda^C_{j\mid -}&=\sum\nolimits_{i=1}^n{{\bf X}_{ij}/n}\\
\lambda^C_{j\mid 0}&=\sum\nolimits_{i=1}^n{{\bf X}_{ij}}\\
\lambda^C_{j\mid k}&=\begin{cases}
    \frac{(\sum_{i \in R_k} {\bf X}_{ij}) (\sum_{i=1}^n \sum_{j \in C_0} {\bf X}_{ij})}{\sum_{i \in R_k} \sum_{j\in C_0} {\bf X}_{ij}},&~~~{\rm if}~~~ C_0 \neq \emptyset\\
    \frac{(\sum_{i \in R_k} {\bf X}_{ij}) (\sum_{i=1}^n \sum_{j \in C_1} {\bf X}_{ij})}{\sum_{i \in R_k} \sum_{j\in C_1} {\bf X}_{ij}},&~~~{\rm if}~~~ C_0 = \emptyset
\end{cases}\\
\lambda^R_i&=\begin{cases}
\frac{(\sum_{j\in C_0 \cup C_1} {\bf X}_{ij}) (\sum_{i \in R_k} \sum_{j\in C_0} {\bf X}_{ij})}{(\sum_{i=1}^n \sum_{j \in C_0} {\bf X}_{ij})(\sum_{i \in R_k} \sum_{j \in C_1 \cup C_0} {\bf X}_{ij})},&~~~{\rm if}~~~ i \in R_k~{\rm and}~C_0 \neq \emptyset\\
\frac{\sum_{j \in C_1} {\bf X}_{ij}}{\sum_{i=1}^n \sum_{j \in C_1} {\bf X}_{ij}},&~~~{\rm if}~~~ i \in R_k~{\rm and}~C_0 = \emptyset\\
\frac{\sum_{j \in C_0 \cup C_1} {\bf X}_{ij}}{\sum_{i=1}^n \sum_{j \in C_0 \cup C_1} {\bf X}_{ij}}, &~~~{\rm if}~~~ i \in O.
\end{cases}
\end{align}

\section{Data sets}\label{sec:dataset_appendix}

In our study, we consider one synthetic and $11$ real-world data set from various domains such as economics, biology, images, and texts.

\textit{Wholesales}\footnote{\label{footnote:uci}\url{https://archive.ics.uci.edu/} (accessed 07.01.2026)}~\citep{uciRepository} counts the monetary units that were spent annually on six product groups in a horeca (\textbf{ho}tel/\textbf{re}staurant/\textbf{ca}fe) or a retail store. The image data set \textit{Optdigits}\textsuperscript{\ref{footnote:uci}}~\citep{uciRepository} summarizes $32\times32$ bitmaps showing handwritten digits ($0$--$9$) by counting the number of activated pixels within blocks of size $4 \times 4$.
Furthermore, we consider the text data sets \textit{BBCSports}\footnote{\label{footnote:bbc}\url{http://mlg.ucd.ie/datasets/bbc.html} (accessed 07.01.2026)}~\citep{greene2006practical}, \textit{BBCNews}\textsuperscript{\ref{footnote:bbc}}~\citep{greene2006practical}, \textit{WebKB}\footnote{\url{https://www.cs.cmu.edu/~webkb/} (accessed 07.01.2026)}, \textit{Reuters21578}\textsuperscript{\ref{footnote:uci}} (Reuters)~\citep{uciRepository} and \textit{20Newsgroups}\textsuperscript{\ref{footnote:uci}} (20NewsG)~\citep{uciRepository}. BBCSports consists of sports news articles regarding athletics, cricket, football, rugby, and tennis and BBCNews consists of news stories regarding business, entertainment, politics, sport, and tech. WebKB is a collection of university websites within six different categories. In the case of Reuters, we use articles from the categories `grain', `money-fx', `earn', `acq', and `crude', and 20NewsG contains messages from $20$ newsgroups. All text data sets are pre-processed by applying stemming\footnote{\url{https://www.nltk.org/api/nltk.stem.SnowballStemmer.html} (accessed 07.01.2026)} and using bag-of-words\footnote{\url{https://scikit-learn.org/stable/modules/generated/sklearn.feature_extraction.text.CountVectorizer.html} (accessed 07.01.2026)}. Afterward, we consider the most frequent $2000$ terms.
Another text data set used is \textit{Sport Articles}\textsuperscript{\ref{footnote:uci}} (SportA)~\citep{uciRepository}, which does not count occurrences of words but the frequencies of $55$ different word groups like `personal pronouns.'
The \textit{Gene Expression}\textsuperscript{\ref{footnote:uci}} (GeneExp) data set~\citep{uciRepository} contains five clusters, each describing a different kind of tumor. The \textit{Human Dendritic Cells}\footnote{\label{footnote:cell_benchmark}\url{https://github.com/JinmiaoChenLab/Batch-effect-removal-benchmarking/} (accessed 07.01.2026)} (HDendritic) data set~\citep{tran2020benchmark} consists of human blood dendritic cell single-cell RNA sequencing data regarding four cell types and the \textit{Mouse Cell Atlas}\textsuperscript{\ref{footnote:cell_benchmark}} (MouseAtlas) data~\citep{tran2020benchmark} contains $11$ cell types from different organ systems. As gene expression data often contains estimated counts, we round all values to the closest integer.

\begin{table*}[t]
    \caption{Data set characteristics: the number of rows $n$, columns $m$ and clusters $K$ as well as the data range, the sparsity (ratio of zeros) and the imbalance (ratio between the smallest and largest ground truth cluster, i.e., $\text{min}_{k\in[K]}|C_k| / \text{max}_{k\in[K]}|C_k|$). Furthermore, the table includes information about the data domain. A dagger $\dagger$ indicates that the original number of features has been modified.}
    \label{tab:dataset_characteristics}
    \centering
    \resizebox{1\textwidth}{!}{
    \begin{tabular}{l|r|r|r|r|r|r|r}
        \toprule
        \textbf{Data set} & $n$ & $m$ & $K$ & Range & Sparsity & Imbalance & Domain\\
         \midrule
        Synth & $1000$ & $6$ & $3$ & $[0, 125]$ & $1\%$ & $1.00$ & Tabular (synthetic)\\
        Wholesales & $440$ & $6$ & $2$ & $[3, 112151]$ & $0\%$ & $0.48$ & Tabular (economics)\\
        SportA & $1000$ & $55$ & $2$ & $[0, 897]$ & $31\%$ & $0.57$ & Tabular (text characteristics)\\
        Optdigits & $5620$ & $64$ & $10$ & $[0, 16]$ & $49\%$ & $0.97$ & Image (digits)\\
        BBCSports & $737$ & $2000^\dagger$ & $5$ & $[0, 145]$ & $92\%$ & $0.38$ & Text (sport articles)\\
        BBCNews & $2225$ & $2000^\dagger$ & $5$ & $[0, 245]$ & $92\%$ & $0.76$ & Text (news articles)\\
        WebKB & $4518$ & $2000^\dagger$ & $6$ & $[0, 3992]$ & $95\%$ & $0.08$ & Text (university websites)\\
        Reuters & $8367$ & $2000^\dagger$ & $5$ & $[0, 245]$ & $97\%$ & $0.15$ & Text (news articles)\\
        20NewsG & $18846$ & $2000^\dagger$ & $20$ & $[0, 11151]$ & $97\%$ & $0.63$ & Text (news articles)\\
        MouseAtlas & $6954$ & $15006$ & $11$ & $[0, 274689]$ & $91\%$ & $0.06$ & Tabular (biology)\\
        GeneExp & $801$ & $20531$ & $5$ & $[0, 21]$ & $16\%$ & $0.26$ & Tabular (biology)\\
        HDendritic & $576$ & $26593$ & $4$ & $[0, 598415]$ & $82\%$ & $0.50$ & Tabular (biology)\\
        \bottomrule
    \end{tabular}
    }
\end{table*}

Lastly, we consider a synthetic data set (Synth) with three clusters. The distribution of the first four columns follows the logic shown in Fig.~\ref{fig:exampleMatrix}, where the cluster membership determines the ratio between column entries. For rows in the first cluster, the ratio between the first four columns is $(2, 3, 4, 1)$, meaning $v_1=2v_4, v_2=3v_4$, and $v_3=4v_4$. The second and third clusters follow ratios of $(4, 1, 4, 1)$ and $(3, 2, 4, 1)$, respectively. The base value $v_4$ ranges within $[1, 30]$, resulting in $v_i\in [1, 120]$ for $i\in [4]$. The fifth column contains uniformly distributed values within $[0, 15]$, and the last column is stable across all clusters with a constant value of $20$. To simulate real-world variance, we add random noise within $[1, 5]$ to $20\%$ of the entries. The data set is designed such that only the first two columns are relevant for clustering and, thus, should be assigned to $C_1$. Since the third and fourth columns follow the same distribution for all rows but depend on the row scaling, they should be assigned to $C_0$. Finally, the fifth and sixth columns are independent of both row scaling and cluster distributions and should therefore be moved to $C_-$. A good clustering result for Synth suggests that the algorithm is able to ignore noise features and thus concentrate on cluster-relevant structures.

The characteristics of all benchmark data sets are given in Table~\ref{tab:dataset_characteristics}.

\section{Implementations}
\label{sec:implementations}

For the conducted experiments, we use our own implementations for \Method, PoissonL, PoissonC and Spherical $k$-Means (SKM). These can be found in our repository\footnote{\url{https://github.com/collinleiber/3CPO} (accessed 27.01.2026)}. The $k$-Means-based approaches use the $k$-Means implementation from \textit{scikit-learn}\footnote{\url{https://scikit-learn.org/stable/modules/generated/sklearn.cluster.KMeans.html} (accessed 07.01.2026)}. For CROINFO, CoclustMod and CoclustSpecMod we use the implementations provided by the \textit{coclust}\footnote{\url{https://github.com/franrole/cclust_package} (accessed 07.01.2026)} package \citep{role2019coclust}. The implementations of EBML\footnote{\label{footnote:elbm}\url{https://github.com/Saeidhoseinipour/ELBMcoclust} (accessed 07.01.2026)}, SEBML\textsuperscript{\ref{footnote:elbm}} and TauCC\footnote{\url{https://github.com/rupensa/tauCC} (accessed 07.01.2026)} are taken from the repositories as referenced in the respective publications.

\section{Additional Results}
In the following sections, we provide complementary experimental results that did not fit into the main paper.

\subsection{ACC and NMI Results}\label{sec:additional_results}

In addition to the ARI results presented in Table~\ref{tab:experiments_ari}, Table~\ref{tab:tfidf_ari} and Table~\ref{tab:ablation_ari}, we also evaluate the clustering results using ACC and NMI. 
These extended evaluations are provided in Table~\ref{tab:experiments} (traditional algorithms), Table~\ref{tab:coclustering} (co-clustering), and Table~\ref{tab:tfidf} (tf-idf/bm25-based text data analysis). Furthermore, Table~\ref{tab:ablation} contains the additional results for the ablation study discussed in Sect.~\ref{sec:ablation}.

\begin{table*}
\centering
\caption{Clustering results (in \%) of traditional clustering algorithms. Entries correspond to the mean of ten executions $\pm$ the standard deviation. The best result for each data set and evaluation metric is marked in \textbf{bold}, the runner-up is \underline{underlined}, and the third place is \dashuline{dashed-underlined}.}
\label{tab:experiments}
\resizebox{1\textwidth}{!}{
\begin{tabular}{l|l|ccccccccc}
\toprule
\textbf{Data set} & \textbf{Metric} & 3CPO & PoissonL & PoissonC & SKM & KM & STD+KM & MM+KM & RF+KM & RCA+KM\\
\midrule
Synth & ACC & \bm{$98.5 \pm 0.0$} & $58.0 \pm 0.0$ & $57.2 \pm 0.0$ & \underline{$61.3 \pm 0.1$} & $36.6 \pm 0.2$ & $61.0 \pm 0.0$ & \dashuline{$61.1 \pm 0.0$} & $61.0 \pm 0.0$ & $35.8 \pm 0.0$\\
& NMI & \bm{$92.5 \pm 0.0$} & $42.4 \pm 0.1$ & $39.9 \pm 0.0$ & $49.0 \pm 0.1$ & $0.3 \pm 0.0$ & \dashuline{$49.1 \pm 0.0$} & \underline{$49.5 \pm 0.1$} & $48.3 \pm 0.0$ & $0.2 \pm 0.0$\\
& ARI & \bm{$95.5 \pm 0.0$} & $33.3 \pm 0.1$ & $29.9 \pm 0.0$ & \dashuline{$42.7 \pm 0.0$} & $0.2 \pm 0.0$ & $42.1 \pm 0.0$ & \underline{$42.9 \pm 0.1$} & $42.5 \pm 0.0$ & $0.0 \pm 0.0$\\
\midrule
Wholesales & ACC & $77.5 \pm 0.0$ & \dashuline{$78.2 \pm 0.0$} & \underline{$78.9 \pm 0.0$} & $75.9 \pm 0.0$ & $59.3 \pm 0.0$ & $76.8 \pm 0.0$ & $76.6 \pm 0.1$ & $73.2 \pm 0.0$ & \bm{$84.4 \pm 0.1$}\\
& NMI & \underline{$31.9 \pm 0.0$} & $23.1 \pm 0.0$ & \dashuline{$24.0 \pm 0.0$} & $19.5 \pm 0.0$ & $0.9 \pm 0.0$ & $20.0 \pm 0.0$ & $20.3 \pm 0.2$ & $17.2 \pm 0.0$ & \bm{$36.1 \pm 0.2$}\\
& ARI & $30.1 \pm 0.0$ & \dashuline{$31.3 \pm 0.0$} & \underline{$32.8 \pm 0.0$} & $26.5 \pm 0.0$ & $-3.1 \pm 0.0$ & $28.2 \pm 0.0$ & $27.9 \pm 0.2$ & $21.3 \pm 0.0$ & \bm{$47.0 \pm 0.3$}\\
\midrule
SportA & ACC & \underline{$78.5 \pm 0.2$} & \bm{$78.6 \pm 0.1$} & $73.4 \pm 0.8$ & $77.2 \pm 0.0$ & $73.3 \pm 0.0$ & \dashuline{$78.3 \pm 0.0$} & $75.6 \pm 0.0$ & $64.3 \pm 0.0$ & $63.6 \pm 0.0$\\
& NMI & \dashuline{$24.1 \pm 0.5$} & \underline{$24.2 \pm 0.2$} & $18.2 \pm 0.2$ & $23.0 \pm 0.0$ & $14.1 \pm 0.0$ & \bm{$24.4 \pm 0.0$} & $17.8 \pm 0.1$ & $13.3 \pm 0.0$ & $0.3 \pm 0.0$\\
& ARI & \underline{$32.3 \pm 0.5$} & \bm{$32.5 \pm 0.2$} & $21.9 \pm 1.5$ & $29.5 \pm 0.0$ & $21.3 \pm 0.0$ & \dashuline{$31.9 \pm 0.0$} & $25.9 \pm 0.1$ & $7.5 \pm 0.0$ & $0.1 \pm 0.0$\\
\midrule
Optdigits & ACC & \underline{$80.0 \pm 2.0$} & \bm{$80.1 \pm 2.1$} & $63.7 \pm 3.5$ & \underline{$80.0 \pm 0.1$} & $79.3 \pm 0.2$ & \bm{$80.1 \pm 0.1$} & \dashuline{$79.4 \pm 0.4$} & $79.1 \pm 0.2$ & $10.6 \pm 0.1$\\
& NMI & $75.1 \pm 0.9$ & \dashuline{$75.2 \pm 1.0$} & $63.4 \pm 2.1$ & $74.8 \pm 0.1$ & \bm{$75.6 \pm 0.3$} & $74.8 \pm 0.1$ & \underline{$75.5 \pm 0.4$} & $75.1 \pm 0.3$ & $1.4 \pm 0.2$\\
& ARI & $66.7 \pm 2.2$ & $67.0 \pm 1.8$ & $45.0 \pm 3.1$ & \underline{$67.5 \pm 0.1$} & $67.1 \pm 0.2$ & \bm{$67.6 \pm 0.1$} & \dashuline{$67.3 \pm 0.3$} & $66.5 \pm 0.2$ & $0.0 \pm 0.0$\\
\midrule
BBCSports & ACC & \bm{$96.1 \pm 1.2$} & \underline{$70.7 \pm 6.6$} & \dashuline{$54.1 \pm 4.3$} & $37.1 \pm 1.7$ & $28.7 \pm 0.3$ & $36.3 \pm 1.5$ & $35.9 \pm 2.5$ & $37.4 \pm 1.5$ & $40.6 \pm 2.7$\\
& NMI & \bm{$89.4 \pm 2.0$} & \underline{$58.6 \pm 6.0$} & \dashuline{$32.5 \pm 4.5$} & $17.7 \pm 2.4$ & $2.7 \pm 0.2$ & $16.3 \pm 1.7$ & $15.2 \pm 3.0$ & $16.4 \pm 2.3$ & $11.8 \pm 5.9$\\
& ARI & \bm{$90.1 \pm 2.9$} & \underline{$51.6 \pm 8.1$} & \dashuline{$27.7 \pm 4.0$} & $10.4 \pm 1.7$ & $0.3 \pm 0.2$ & $9.5 \pm 0.9$ & $6.9 \pm 3.6$ & $9.8 \pm 1.7$ & $3.7 \pm 2.6$\\
\midrule
BBCNews & ACC & \bm{$95.7 \pm 0.4$} & \underline{$95.4 \pm 0.4$} & \dashuline{$81.1 \pm 5.0$} & $63.2 \pm 3.0$ & $33.2 \pm 0.7$ & $56.8 \pm 2.0$ & $50.9 \pm 1.7$ & $55.5 \pm 2.2$ & $37.8 \pm 9.5$\\
& NMI & \bm{$87.2 \pm 0.8$} & \underline{$86.2 \pm 0.8$} & \dashuline{$66.5 \pm 4.3$} & $38.3 \pm 1.8$ & $8.0 \pm 0.3$ & $34.1 \pm 1.0$ & $25.6 \pm 0.6$ & $32.2 \pm 1.8$ & $28.3 \pm 15.2$\\
& ARI & \bm{$89.9 \pm 0.9$} & \underline{$89.2 \pm 0.8$} & \dashuline{$63.4 \pm 6.7$} & $31.5 \pm 4.0$ & $6.0 \pm 0.3$ & $24.6 \pm 1.6$ & $16.4 \pm 1.5$ & $25.6 \pm 1.8$ & $15.4 \pm 12.0$\\
\midrule
WebKB & ACC & \bm{$57.4 \pm 2.7$} & \underline{$51.8 \pm 2.0$} & \dashuline{$44.5 \pm 4.0$} & $38.7 \pm 0.3$ & $39.2 \pm 0.1$ & $39.2 \pm 0.4$ & $38.8 \pm 0.5$ & $40.0 \pm 3.1$ & $36.3 \pm 0.3$\\
& NMI & \bm{$37.2 \pm 1.2$} & \underline{$34.1 \pm 1.4$} & \dashuline{$26.6 \pm 2.3$} & $18.2 \pm 0.2$ & $2.8 \pm 0.0$ & $18.0 \pm 0.2$ & $16.7 \pm 0.4$ & $15.6 \pm 3.3$ & $1.5 \pm 0.4$\\
& ARI & \bm{$31.5 \pm 3.2$} & \underline{$25.4 \pm 0.9$} & \dashuline{$20.2 \pm 3.8$} & $11.4 \pm 0.2$ & $3.6 \pm 0.0$ & $12.0 \pm 0.2$ & $10.9 \pm 0.4$ & $11.0 \pm 2.7$ & $-0.0 \pm 0.1$\\
\midrule
Reuters & ACC & \bm{$80.9 \pm 2.9$} & \underline{$80.1 \pm 3.4$} & \dashuline{$65.7 \pm 4.2$} & $50.5 \pm 0.2$ & $44.5 \pm 6.0$ & $57.9 \pm 0.0$ & $57.2 \pm 0.0$ & $41.1 \pm 1.2$ & $47.3 \pm 0.3$\\
& NMI & \bm{$70.8 \pm 2.9$} & \underline{$68.1 \pm 2.5$} & \dashuline{$53.2 \pm 3.2$} & $35.9 \pm 0.3$ & $19.6 \pm 1.6$ & $37.2 \pm 0.0$ & $37.2 \pm 0.0$ & $29.7 \pm 0.2$ & $1.4 \pm 0.6$\\
& ARI & \bm{$67.8 \pm 1.6$} & \underline{$65.2 \pm 5.0$} & \dashuline{$44.1 \pm 4.0$} & $21.9 \pm 0.2$ & $15.2 \pm 0.6$ & $32.5 \pm 0.1$ & $31.9 \pm 0.0$ & $3.7 \pm 4.2$ & $0.3 \pm 0.3$\\
\midrule
20NewsG & ACC & \bm{$38.4 \pm 2.5$} & \underline{$37.2 \pm 2.4$} & \dashuline{$29.6 \pm 1.9$} & $12.2 \pm 0.3$ & $7.5 \pm 0.1$ & $11.8 \pm 0.3$ & $11.1 \pm 0.2$ & $9.8 \pm 0.2$ & $5.7 \pm 0.4$\\
& NMI & \bm{$39.9 \pm 1.0$} & \underline{$39.4 \pm 0.9$} & \dashuline{$32.2 \pm 1.3$} & $5.8 \pm 0.4$ & $2.7 \pm 0.0$ & $5.8 \pm 0.2$ & $5.0 \pm 0.2$ & $3.7 \pm 0.4$ & $0.9 \pm 0.5$\\
& ARI & \bm{$23.0 \pm 1.3$} & \underline{$22.2 \pm 1.9$} & \dashuline{$14.6 \pm 1.5$} & $2.0 \pm 0.2$ & $0.4 \pm 0.0$ & $2.1 \pm 0.1$ & $1.5 \pm 0.1$ & $1.1 \pm 0.2$ & $0.0 \pm 0.0$\\
\midrule
MouseAtlas & ACC & \bm{$67.5 \pm 6.0$} & \underline{$67.1 \pm 3.9$} & $46.6 \pm 7.3$ & \dashuline{$56.4 \pm 2.9$} & $32.6 \pm 0.4$ & $54.6 \pm 1.7$ & $52.4 \pm 1.7$ & $50.9 \pm 2.1$ & $20.7 \pm 0.0$\\
& NMI & \bm{$67.6 \pm 2.3$} & \underline{$65.7 \pm 3.0$} & $42.9 \pm 7.1$ & \dashuline{$57.2 \pm 1.6$} & $21.8 \pm 0.6$ & $56.7 \pm 0.9$ & $51.3 \pm 1.3$ & $49.4 \pm 1.9$ & $0.8 \pm 0.1$\\
& ARI & \bm{$56.3 \pm 6.2$} & \underline{$55.4 \pm 5.0$} & $18.5 \pm 9.5$ & \dashuline{$39.3 \pm 2.6$} & $1.6 \pm 0.1$ & $38.4 \pm 1.4$ & $36.3 \pm 2.0$ & $28.6 \pm 2.6$ & $0.1 \pm 0.0$\\
\midrule
GeneExp & ACC & \bm{$99.6 \pm 0.1$} & \underline{$99.4 \pm 0.1$} & $93.0 \pm 6.8$ & \underline{$99.4 \pm 0.0$} & \dashuline{$99.3 \pm 0.1$} & \bm{$99.6 \pm 0.0$} & \dashuline{$99.3 \pm 0.0$} & \dashuline{$99.3 \pm 0.1$} & $39.3 \pm 4.2$\\
& NMI & \underline{$98.4 \pm 0.2$} & \dashuline{$97.8 \pm 0.2$} & $90.3 \pm 6.5$ & $97.6 \pm 0.0$ & $97.3 \pm 0.2$ & \bm{$98.5 \pm 0.1$} & $97.1 \pm 0.0$ & $97.4 \pm 0.2$ & $6.0 \pm 13.5$\\
& ARI & \bm{$99.1 \pm 0.1$} & \underline{$98.7 \pm 0.1$} & $86.7 \pm 11.0$ & \dashuline{$98.5 \pm 0.0$} & $98.2 \pm 0.1$ & \bm{$99.1 \pm 0.1$} & $98.3 \pm 0.0$ & \dashuline{$98.5 \pm 0.1$} & $3.7 \pm 9.8$\\
\midrule
HDendritic & ACC & $87.5 \pm 7.2$ & \dashuline{$91.2 \pm 6.1$} & $84.8 \pm 8.0$ & \underline{$91.9 \pm 0.1$} & $57.0 \pm 11.1$ & \bm{$92.1 \pm 0.2$} & $89.7 \pm 0.1$ & $57.0 \pm 11.1$ & $33.6 \pm 0.6$\\
& NMI & \underline{$77.8 \pm 3.8$} & \bm{$80.9 \pm 4.0$} & $73.7 \pm 3.1$ & $77.2 \pm 0.2$ & $48.1 \pm 10.6$ & \dashuline{$77.4 \pm 0.2$} & $73.9 \pm 0.1$ & $48.1 \pm 10.6$ & $1.0 \pm 0.3$\\
& ARI & $79.6 \pm 5.8$ & \bm{$83.4 \pm 5.2$} & $74.8 \pm 6.1$ & \dashuline{$79.9 \pm 0.2$} & $37.5 \pm 13.3$ & \underline{$80.2 \pm 0.4$} & $75.5 \pm 0.2$ & $37.5 \pm 13.3$ & $-0.1 \pm 0.1$\\
\bottomrule
\end{tabular}}
\end{table*}

\begin{table*}
\centering
\caption{Clustering results (in \%) of co-clustering  algorithms. Entries correspond to the mean of ten executions $\pm$ the standard deviation. The best result for each data set and evaluation metric is marked in \textbf{bold}, the runner-up is \underline{underlined}, and the third place is \dashuline{dashed-underlined}.}
\label{tab:coclustering}
\resizebox{0.9\textwidth}{!}{
\begin{tabular}{l|l|ccccccc}
\toprule
\textbf{Data set} & \textbf{Metric} & 3CPO & CROINFO & CCMod & CCSMod & ELBM & SELBM & TauCC\\
\midrule
Synth & ACC & \bm{$98.5 \pm 0.0$} & $57.9 \pm 0.1$ & \underline{$60.0 \pm 0.2$} & $45.8 \pm 14.5$ & $45.4 \pm 7.7$ & $42.5 \pm 0.2$ & \dashuline{$59.8 \pm 0.0$}\\
& NMI & \bm{$92.5 \pm 0.0$} & \underline{$42.0 \pm 0.2$} & \dashuline{$34.0 \pm 0.4$} & $20.0 \pm 23.6$ & $14.7 \pm 14.1$ & $7.3 \pm 0.1$ & $33.7 \pm 0.5$\\
& ARI & \bm{$95.5 \pm 0.0$} & \underline{$32.8 \pm 0.3$} & \dashuline{$25.2 \pm 0.1$} & $17.8 \pm 20.8$ & $9.2 \pm 8.6$ & $3.5 \pm 1.6$ & $25.1 \pm 0.6$\\
\midrule
Wholesales & ACC & \dashuline{$77.5 \pm 0.0$} & \bm{$79.2 \pm 0.2$} & \underline{$78.4 \pm 0.0$} & $68.8 \pm 3.1$ & $64.7 \pm 4.3$ & $65.5 \pm 0.1$ & $75.7 \pm 0.0$\\
& NMI & \bm{$31.9 \pm 0.0$} & \underline{$25.1 \pm 0.5$} & \underline{$25.1 \pm 0.0$} & $2.3 \pm 6.7$ & $1.6 \pm 3.9$ & $0.3 \pm 0.1$ & \dashuline{$21.8 \pm 0.0$}\\
& ARI & \dashuline{$30.1 \pm 0.0$} & \bm{$33.6 \pm 0.5$} & \underline{$32.0 \pm 0.0$} & $3.2 \pm 9.1$ & $2.0 \pm 5.0$ & $1.9 \pm 0.6$ & $27.6 \pm 0.0$\\
\midrule
SportA & ACC & \bm{$78.5 \pm 0.2$} & $75.2 \pm 0.4$ & \dashuline{$77.7 \pm 0.0$} & \underline{$78.1 \pm 0.0$} & $66.9 \pm 0.0$ & $66.6 \pm 0.1$ & $70.1 \pm 6.6$\\
& NMI & \bm{$24.1 \pm 0.5$} & $19.4 \pm 0.7$ & \dashuline{$21.2 \pm 0.0$} & \underline{$23.1 \pm 0.0$} & $4.3 \pm 0.0$ & $4.0 \pm 0.1$ & $10.9 \pm 10.5$\\
& ARI & \bm{$32.3 \pm 0.5$} & $25.4 \pm 0.8$ & \dashuline{$30.3 \pm 0.0$} & \underline{$31.4 \pm 0.1$} & $7.1 \pm 0.0$ & $6.5 \pm 0.1$ & $15.8 \pm 15.1$\\
\midrule
Optdigits & ACC & \bm{$80.0 \pm 2.0$} & \underline{$65.7 \pm 1.7$} & $33.6 \pm 1.9$ & $45.8 \pm 3.8$ & \dashuline{$61.5 \pm 5.0$} & $40.3 \pm 2.8$ & $20.3 \pm 7.2$\\
& NMI & \bm{$75.1 \pm 0.9$} & \underline{$60.7 \pm 2.2$} & $34.4 \pm 1.5$ & $50.9 \pm 2.5$ & \dashuline{$57.0 \pm 2.0$} & $39.4 \pm 3.7$ & $19.6 \pm 13.1$\\
& ARI & \bm{$66.7 \pm 2.2$} & \underline{$49.7 \pm 2.1$} & $21.4 \pm 1.5$ & $33.2 \pm 2.9$ & \dashuline{$45.3 \pm 3.0$} & $23.1 \pm 4.3$ & $10.6 \pm 7.5$\\
\midrule
BBCSports & ACC & \bm{$96.1 \pm 1.2$} & \dashuline{$62.3 \pm 3.7$} & \underline{$64.2 \pm 2.3$} & $60.8 \pm 1.1$ & $28.4 \pm 0.2$ & $31.6 \pm 0.0$ & $51.6 \pm 5.7$\\
& NMI & \bm{$89.4 \pm 2.0$} & \underline{$58.3 \pm 3.9$} & $47.6 \pm 5.7$ & \dashuline{$51.7 \pm 1.0$} & $1.7 \pm 0.1$ & $0.6 \pm 0.0$ & $28.0 \pm 7.5$\\
& ARI & \bm{$90.1 \pm 2.9$} & \underline{$48.3 \pm 5.2$} & \dashuline{$41.0 \pm 4.3$} & $38.1 \pm 0.1$ & $0.3 \pm 0.1$ & $0.3 \pm 0.0$ & $26.2 \pm 6.3$\\
\midrule
BBCNews & ACC & \bm{$95.7 \pm 0.4$} & \dashuline{$79.9 \pm 1.4$} & \underline{$87.3 \pm 0.9$} & $76.6 \pm 0.3$ & $31.0 \pm 0.6$ & $27.6 \pm 0.4$ & $60.5 \pm 9.6$\\
& NMI & \bm{$87.2 \pm 0.8$} & $63.2 \pm 1.7$ & \underline{$67.8 \pm 1.7$} & \dashuline{$66.1 \pm 0.3$} & $5.7 \pm 0.2$ & $5.2 \pm 0.2$ & $45.6 \pm 9.2$\\
& ARI & \bm{$89.9 \pm 0.9$} & $62.1 \pm 2.3$ & \underline{$71.9 \pm 1.8$} & \dashuline{$62.3 \pm 0.3$} & $4.5 \pm 0.2$ & $2.9 \pm 0.6$ & $41.8 \pm 10.5$\\
\midrule
WebKB & ACC & \bm{$57.4 \pm 2.7$} & $46.6 \pm 1.7$ & \dashuline{$53.7 \pm 0.7$} & \underline{$56.8 \pm 0.1$} & $35.6 \pm 0.1$ & $36.3 \pm 0.0$ & $48.3 \pm 2.7$\\
& NMI & \bm{$37.2 \pm 1.2$} & \dashuline{$26.1 \pm 1.0$} & $25.8 \pm 0.7$ & \underline{$32.0 \pm 0.0$} & $2.4 \pm 0.0$ & $0.0 \pm 0.0$ & $13.1 \pm 2.3$\\
& ARI & \bm{$31.5 \pm 3.2$} & $20.1 \pm 1.4$ & \dashuline{$20.4 \pm 0.7$} & \underline{$25.8 \pm 0.1$} & $3.7 \pm 0.1$ & $0.0 \pm 0.0$ & $11.3 \pm 2.0$\\
\midrule
Reuters & ACC & \bm{$80.9 \pm 2.9$} & \dashuline{$70.8 \pm 0.7$} & \underline{$77.9 \pm 2.8$} & $67.4 \pm 0.0$ & $59.2 \pm 0.1$ & $48.0 \pm 1.1$ & $68.8 \pm 0.4$\\
& NMI & \bm{$70.8 \pm 2.9$} & $55.6 \pm 1.5$ & \underline{$59.1 \pm 3.1$} & \dashuline{$56.1 \pm 0.0$} & $19.8 \pm 0.1$ & $6.9 \pm 0.6$ & $35.2 \pm 0.9$\\
& ARI & \bm{$67.8 \pm 1.6$} & \dashuline{$59.9 \pm 1.2$} & \underline{$64.6 \pm 4.1$} & $46.1 \pm 0.0$ & $34.6 \pm 0.3$ & $10.1 \pm 1.0$ & $43.2 \pm 0.5$\\
\midrule
20NewsG & ACC & \bm{$38.4 \pm 2.5$} & \underline{$28.4 \pm 1.0$} & $15.5 \pm 0.5$ & \dashuline{$19.1 \pm 0.9$} & $7.9 \pm 0.3$ & $5.4 \pm 0.1$ & $8.7 \pm 1.9$\\
& NMI & \bm{$39.9 \pm 1.0$} & \underline{$30.4 \pm 0.7$} & $17.3 \pm 0.5$ & \dashuline{$24.8 \pm 2.5$} & $2.8 \pm 0.1$ & $0.3 \pm 0.1$ & $6.9 \pm 4.4$\\
& ARI & \bm{$23.0 \pm 1.3$} & \underline{$16.2 \pm 0.7$} & \dashuline{$6.9 \pm 0.6$} & $6.2 \pm 2.3$ & $0.5 \pm 0.1$ & $0.0 \pm 0.0$ & $2.1 \pm 1.5$\\
\midrule
MouseAtlas & ACC & \bm{$67.5 \pm 6.0$} & \underline{$61.1 \pm 2.1$} & \dashuline{$58.2 \pm 2.6$} & $54.8 \pm 0.0$ & $31.3 \pm 0.9$ & $25.8 \pm 1.1$ & $50.4 \pm 6.1$\\
& NMI & \bm{$67.6 \pm 2.3$} & \underline{$62.6 \pm 1.3$} & \dashuline{$57.2 \pm 2.5$} & $52.2 \pm 0.0$ & $22.3 \pm 1.0$ & $8.5 \pm 1.9$ & $51.3 \pm 4.6$\\
& ARI & \bm{$56.3 \pm 6.2$} & \underline{$49.7 \pm 1.8$} & \dashuline{$38.4 \pm 1.8$} & $27.4 \pm 0.1$ & $1.4 \pm 0.3$ & $0.5 \pm 0.1$ & $34.4 \pm 4.6$\\
\midrule
GeneExp & ACC & \bm{$99.6 \pm 0.1$} & $77.3 \pm 2.0$ & \dashuline{$84.4 \pm 2.3$} & \underline{$97.6 \pm 0.0$} & $32.3 \pm 1.7$ & $37.5 \pm 0.0$ & $67.2 \pm 7.5$\\
& NMI & \bm{$98.4 \pm 0.2$} & $73.9 \pm 3.1$ & \dashuline{$80.2 \pm 1.6$} & \underline{$93.8 \pm 0.0$} & $9.1 \pm 1.9$ & $0.0 \pm 0.0$ & $61.3 \pm 8.5$\\
& ARI & \bm{$99.1 \pm 0.1$} & $59.0 \pm 3.1$ & \dashuline{$72.4 \pm 2.9$} & \underline{$93.2 \pm 0.0$} & $4.3 \pm 1.4$ & $0.0 \pm 0.0$ & $53.1 \pm 9.4$\\
\midrule
HDendritic & ACC & \underline{$87.5 \pm 7.2$} & \bm{$92.3 \pm 0.5$} & \dashuline{$80.7 \pm 4.5$} & $33.7 \pm 0.0$ & $72.8 \pm 3.8$ & $36.2 \pm 4.8$ & $64.3 \pm 16.2$\\
& NMI & \underline{$77.8 \pm 3.8$} & \bm{$77.9 \pm 1.5$} & \dashuline{$72.6 \pm 3.2$} & $1.0 \pm 0.0$ & $45.9 \pm 2.6$ & $6.5 \pm 10.9$ & $46.2 \pm 20.3$\\
& ARI & \underline{$79.6 \pm 5.8$} & \bm{$81.9 \pm 1.3$} & \dashuline{$72.2 \pm 5.0$} & $0.0 \pm 0.0$ & $45.1 \pm 4.3$ & $4.0 \pm 7.8$ & $42.3 \pm 23.0$\\
\bottomrule
\end{tabular}
}
\end{table*}

\begin{table*}
\centering
\caption{Clustering results (in \%) regarding text data sets pre-processed by TF-IDF and BM25. Entries correspond to the mean of ten executions $\pm$ the standard deviation. The best performance is highlighted in \textbf{bold}.}
\label{tab:tfidf}
\resizebox{0.65\textwidth}{!}{
\begin{tabular}{l|l|ccccc}
\toprule
\textbf{Data set} & \textbf{Metric} & \Method & TF-IDF+KM & BM25+KM & TF-IDF+SKM & BM25+SKM\\
\midrule
BBCSports & ACC & \bm{$96.1 \pm 1.2$} & $86.9 \pm 6.3$ & $64.1 \pm 8.0$ & $89.2 \pm 5.0$ & $92.4 \pm 2.9$\\
& NMI & \bm{$89.4 \pm 2.0$} & $74.7 \pm 6.3$ & $58.9 \pm 10.5$ & $77.8 \pm 4.7$ & $82.0 \pm 4.6$\\
& ARI & \bm{$90.1 \pm 2.9$} & $71.0 \pm 11.4$ & $43.7 \pm 17.0$ & $76.5 \pm 7.1$ & $81.1 \pm 6.8$\\
\midrule
BBCNews & ACC & $95.7 \pm 0.4$ & $93.2 \pm 1.2$ & $94.0 \pm 5.3$ & $94.1 \pm 0.4$ & \bm{$96.0 \pm 0.2$}\\
& NMI & $87.2 \pm 0.8$ & $81.6 \pm 1.8$ & $85.9 \pm 3.9$ & $83.3 \pm 0.8$ & \bm{$87.9 \pm 0.4$}\\
& ARI & $89.9 \pm 0.9$ & $84.2 \pm 2.8$ & $87.7 \pm 6.8$ & $86.4 \pm 0.8$ & \bm{$90.7 \pm 0.3$}\\
\midrule
WebKB & ACC & \bm{$57.4 \pm 2.7$}  & $45.3 \pm 0.7$ & $44.8 \pm 2.7$ & $49.2 \pm 1.2$ & $51.9 \pm 2.6$\\
& NMI & \bm{$37.2 \pm 1.2$} & $26.7 \pm 0.9$ & $26.0 \pm 1.1$ & $30.4 \pm 1.3$ & $36.0 \pm 2.1$\\
& ARI & \bm{$31.5 \pm 3.2$} & $18.8 \pm 1.0$ & $14.6 \pm 1.3$ & $22.4 \pm 1.1$ & $29.0 \pm 1.9$\\
\midrule
Reuters & ACC & \bm{$80.9 \pm 2.9$} & $50.0 \pm 0.2$ & $60.4 \pm 5.9$ & $56.1 \pm 0.1$ & $73.1 \pm 0.6$\\
& NMI & \bm{$70.8 \pm 2.9$}  & $33.6 \pm 0.2$ & $42.9 \pm 4.5$ & $49.0 \pm 0.3$ & $64.0 \pm 2.3$\\
& ARI & \bm{$67.8 \pm 1.6$}  & $19.3 \pm 0.2$ & $34.8 \pm 5.0$ & $34.0 \pm 0.2$ & $55.8 \pm 1.2$\\
\midrule
20NewsG & ACC & $38.4 \pm 2.5$ & $21.3 \pm 1.2$ & $21.9 \pm 1.6$ & $24.3 \pm 0.9$ & \bm{$39.6 \pm 1.5$}\\
& NMI & \bm{$39.9 \pm 1.0$} & $20.1 \pm 1.2$ & $26.9 \pm 1.4$ & $23.0 \pm 1.3$ & $38.0 \pm 1.2$\\
& ARI & \bm{$23.0 \pm 1.3$} & $5.6 \pm 0.6$ & $5.5 \pm 0.5$ & $8.2 \pm 0.4$ & $21.1 \pm 1.2$\\
\bottomrule
\end{tabular}
}
\end{table*}

\begin{table*}
\centering
\caption{Clustering results (in \%) and the final number of columns included in $C_1$ of various ablation studies. Entries correspond to the mean of ten executions $\pm$ the standard deviation. Colors indicate whether the average performance of \Method lies above (green), within (yellow) or below (red) the standard deviation band, where above is better for ACC, NMI and ARI and below is better for $|C_1|$.}
\label{tab:ablation}
\resizebox{0.7\textwidth}{!}{
\begin{tabular}{l|l|ccccc}
\toprule
\textbf{Data set} & \textbf{Metric} & 3CPO & $C_{-}=\emptyset$ & $C_{0}=\emptyset$ & $b({\bf X}_{\cdot j})=0$ & $b({\bf X}_{\cdot j})=\text{BIC}({\bf X}_{\cdot j})$\\
\midrule
Synth & ACC & $98.5 \pm 0.0$ & \cellcolor{green!30}$58.0 \pm 0.0$ & \cellcolor{yellow!30}$98.5 \pm 0.0$ & \cellcolor{yellow!30}$98.5 \pm 0.0$ & \cellcolor{yellow!30}$98.5 \pm 0.0$\\
& NMI & $92.5 \pm 0.0$ & \cellcolor{green!30}$42.4 \pm 0.1$ & \cellcolor{yellow!30}$92.5 \pm 0.0$ & \cellcolor{yellow!30}$92.5 \pm 0.0$ & \cellcolor{yellow!30}$92.5 \pm 0.0$\\
& ARI & $95.5 \pm 0.0$ & \cellcolor{green!30}$33.3 \pm 0.1$ & \cellcolor{yellow!30}$95.5 \pm 0.0$ & \cellcolor{yellow!30}$95.5 \pm 0.0$ & \cellcolor{yellow!30}$95.5 \pm 0.0$\\
& $|C_1|$ & $2.0 \pm 0.0$ & \cellcolor{green!30}$5.8 \pm 0.6$ & \cellcolor{green!30}$4.0 \pm 0.0$ & \cellcolor{green!30}$4.0 \pm 0.0$ & \cellcolor{yellow!30}$2.0 \pm 0.0$\\
\midrule
Wholesales & ACC & $77.5 \pm 0.0$ & \cellcolor{red!30}$78.2 \pm 0.0$ & \cellcolor{yellow!30}$77.5 \pm 0.0$ & \cellcolor{yellow!30}$77.5 \pm 0.0$ & \cellcolor{yellow!30}$77.5 \pm 0.0$\\
& NMI & $31.9 \pm 0.0$ & \cellcolor{green!30}$23.1 \pm 0.0$ & \cellcolor{yellow!30}$31.9 \pm 0.0$ & \cellcolor{yellow!30}$31.9 \pm 0.0$ & \cellcolor{yellow!30}$31.9 \pm 0.0$\\
& ARI & $30.1 \pm 0.0$ & \cellcolor{red!30}$31.3 \pm 0.0$ & \cellcolor{yellow!30}$30.1 \pm 0.0$ & \cellcolor{yellow!30}$30.1 \pm 0.0$ & \cellcolor{yellow!30}$30.1 \pm 0.0$\\
& $|C_1|$ & $5.0 \pm 0.0$ & \cellcolor{green!30}$6.0 \pm 0.0$ & \cellcolor{yellow!30}$5.0 \pm 0.0$ & \cellcolor{yellow!30}$5.0 \pm 0.0$ & \cellcolor{yellow!30}$5.0 \pm 0.0$\\
\midrule
SportA & ACC & $78.5 \pm 0.2$ & \cellcolor{yellow!30}$78.5 \pm 0.0$ & \cellcolor{yellow!30}$78.6 \pm 0.1$ & \cellcolor{yellow!30}$78.6 \pm 0.1$ & \cellcolor{yellow!30}$78.6 \pm 0.1$\\
& NMI & $24.1 \pm 0.5$ & \cellcolor{yellow!30}$24.0 \pm 0.2$ & \cellcolor{yellow!30}$24.2 \pm 0.2$ & \cellcolor{yellow!30}$24.2 \pm 0.2$ & \cellcolor{yellow!30}$24.4 \pm 0.3$\\
& ARI & $32.3 \pm 0.5$ & \cellcolor{yellow!30}$32.3 \pm 0.1$ & \cellcolor{red!30}$32.6 \pm 0.2$ & \cellcolor{red!30}$32.6 \pm 0.2$ & \cellcolor{yellow!30}$32.5 \pm 0.3$\\
& $|C_1|$ & $41.1 \pm 1.4$ & \cellcolor{green!30}$42.5 \pm 1.1$ & \cellcolor{green!30}$52.0 \pm 0.0$ & \cellcolor{green!30}$52.0 \pm 0.0$ & \cellcolor{yellow!30}$41.3 \pm 0.9$\\
\midrule
Optdigits & ACC & $80.0 \pm 2.0$ & \cellcolor{yellow!30}$80.1 \pm 2.1$ & \cellcolor{yellow!30}$80.1 \pm 2.0$ & \cellcolor{yellow!30}$80.1 \pm 2.0$ & \cellcolor{yellow!30}$80.1 \pm 2.0$\\
& NMI & $75.1 \pm 0.9$ & \cellcolor{yellow!30}$75.1 \pm 1.0$ & \cellcolor{yellow!30}$75.2 \pm 1.0$ & \cellcolor{yellow!30}$75.2 \pm 1.0$ & \cellcolor{yellow!30}$75.2 \pm 1.0$\\
& ARI & $66.7 \pm 2.2$ & \cellcolor{yellow!30}$66.9 \pm 1.8$ & \cellcolor{yellow!30}$66.8 \pm 2.2$ & \cellcolor{yellow!30}$66.8 \pm 2.2$ & \cellcolor{yellow!30}$66.8 \pm 2.2$\\
& $|C_1|$ & $55.8 \pm 0.6$ & \cellcolor{green!30}$56.0 \pm 0.0$ & \cellcolor{green!30}$56.0 \pm 0.0$ & \cellcolor{green!30}$62.0 \pm 0.0$ & \cellcolor{green!30}$56.0 \pm 0.0$\\
\midrule
BBCSports & ACC & $96.1 \pm 1.2$ & \cellcolor{green!30}$87.3 \pm 6.4$ & \cellcolor{green!30}$89.1 \pm 4.0$ & \cellcolor{green!30}$80.3 \pm 8.3$ & \cellcolor{yellow!30}$93.7 \pm 5.5$\\
& NMI & $89.4 \pm 2.0$ & \cellcolor{green!30}$79.4 \pm 5.7$ & \cellcolor{green!30}$78.0 \pm 5.3$ & \cellcolor{green!30}$68.8 \pm 5.2$ & \cellcolor{yellow!30}$87.1 \pm 4.2$\\
& ARI & $90.1 \pm 2.9$ & \cellcolor{green!30}$78.0 \pm 8.2$ & \cellcolor{green!30}$75.2 \pm 8.5$ & \cellcolor{green!30}$63.3 \pm 7.6$ & \cellcolor{yellow!30}$87.0 \pm 6.3$\\
& $|C_1|$ & $799.2 \pm 7.6$ & \cellcolor{green!30}$860.6 \pm 28.1$ & \cellcolor{green!30}$1218.8 \pm 14.5$ & \cellcolor{green!30}$1953.7 \pm 5.9$ & \cellcolor{red!30}$720.8 \pm 14.7$\\
\midrule
BBCNews & ACC & $95.7 \pm 0.4$ & \cellcolor{yellow!30}$95.8 \pm 0.3$ & \cellcolor{yellow!30}$95.6 \pm 0.3$ & \cellcolor{yellow!30}$95.5 \pm 0.2$ & \cellcolor{yellow!30}$95.8 \pm 0.3$\\
& NMI & $87.2 \pm 0.8$ & \cellcolor{yellow!30}$87.2 \pm 0.8$ & \cellcolor{yellow!30}$86.6 \pm 0.7$ & \cellcolor{green!30}$86.7 \pm 0.4$ & \cellcolor{yellow!30}$87.3 \pm 0.6$\\
& ARI & $89.9 \pm 0.9$ & \cellcolor{yellow!30}$90.1 \pm 0.7$ & \cellcolor{yellow!30}$89.6 \pm 0.7$ & \cellcolor{yellow!30}$89.5 \pm 0.4$ & \cellcolor{yellow!30}$90.1 \pm 0.7$\\
& $|C_1|$ & $1379.0 \pm 6.2$ & \cellcolor{green!30}$1437.4 \pm 4.5$ & \cellcolor{green!30}$1758.4 \pm 2.3$ & \cellcolor{green!30}$1984.8 \pm 2.3$ & \cellcolor{red!30}$1367.4 \pm 6.1$\\
\midrule
WebKB & ACC & $57.4 \pm 2.7$ & \cellcolor{green!30}$50.7 \pm 2.3$ & \cellcolor{yellow!30}$58.2 \pm 3.0$ & \cellcolor{yellow!30}$57.4 \pm 2.6$ & \cellcolor{yellow!30}$56.6 \pm 1.9$\\
& NMI & $37.2 \pm 1.2$ & \cellcolor{green!30}$34.5 \pm 0.6$ & \cellcolor{yellow!30}$36.7 \pm 1.2$ & \cellcolor{yellow!30}$37.0 \pm 1.5$ & \cellcolor{yellow!30}$37.1 \pm 0.8$\\
& ARI & $31.5 \pm 3.2$ & \cellcolor{green!30}$24.9 \pm 1.1$ & \cellcolor{yellow!30}$33.8 \pm 3.5$ & \cellcolor{yellow!30}$32.6 \pm 2.7$ & \cellcolor{yellow!30}$30.9 \pm 2.5$\\
& $|C_1|$ & $1314.9 \pm 18.4$ & \cellcolor{green!30}$1407.9 \pm 42.5$ & \cellcolor{green!30}$1854.2 \pm 6.8$ & \cellcolor{green!30}$1936.2 \pm 4.6$ & \cellcolor{red!30}$1282.7 \pm 19.2$\\
\midrule
Reuters & ACC & $80.9 \pm 2.9$ & \cellcolor{yellow!30}$80.2 \pm 3.7$ & \cellcolor{yellow!30}$80.1 \pm 3.2$ & \cellcolor{yellow!30}$78.1 \pm 4.2$ & \cellcolor{yellow!30}$80.8 \pm 2.7$\\
& NMI & $70.8 \pm 2.9$ & \cellcolor{yellow!30}$68.9 \pm 2.7$ & \cellcolor{yellow!30}$70.5 \pm 3.0$ & \cellcolor{yellow!30}$68.7 \pm 3.7$ & \cellcolor{yellow!30}$71.7 \pm 2.3$\\
& ARI & $67.8 \pm 1.6$ & \cellcolor{yellow!30}$65.7 \pm 5.1$ & \cellcolor{yellow!30}$67.2 \pm 3.3$ & \cellcolor{yellow!30}$65.4 \pm 4.2$ & \cellcolor{yellow!30}$68.3 \pm 1.3$\\
& $|C_1|$ & $1321.9 \pm 7.9$ & \cellcolor{green!30}$1390.7 \pm 29.1$ & \cellcolor{green!30}$1914.6 \pm 2.1$ & \cellcolor{green!30}$1976.0 \pm 1.7$ & \cellcolor{red!30}$1213.3 \pm 11.7$\\
\midrule
20NewsG & ACC & $38.4 \pm 2.5$ & \cellcolor{yellow!30}$37.9 \pm 1.3$ & \cellcolor{yellow!30}$38.1 \pm 1.9$ & \cellcolor{yellow!30}$37.9 \pm 1.8$ & \cellcolor{yellow!30}$37.3 \pm 2.0$\\
& NMI & $39.9 \pm 1.0$ & \cellcolor{yellow!30}$39.4 \pm 1.2$ & \cellcolor{yellow!30}$39.8 \pm 0.8$ & \cellcolor{yellow!30}$39.8 \pm 0.8$ & \cellcolor{yellow!30}$39.5 \pm 1.2$\\
& ARI & $23.0 \pm 1.3$ & \cellcolor{yellow!30}$23.2 \pm 1.7$ & \cellcolor{yellow!30}$22.4 \pm 1.0$ & \cellcolor{yellow!30}$22.5 \pm 0.9$ & \cellcolor{yellow!30}$22.8 \pm 2.0$\\
& $|C_1|$ & $1453.7 \pm 15.3$ & \cellcolor{yellow!30}$1456.7 \pm 13.1$ & \cellcolor{green!30}$1990.5 \pm 0.8$ & \cellcolor{green!30}$1999.0 \pm 0.0$ & \cellcolor{green!30}$1564.4 \pm 10.4$\\
\midrule
MouseAtlas & ACC & $67.5 \pm 6.0$ & \cellcolor{yellow!30}$67.1 \pm 3.9$ & \cellcolor{yellow!30}$67.5 \pm 6.0$ & \cellcolor{yellow!30}$67.5 \pm 6.0$ & \cellcolor{yellow!30}$67.5 \pm 6.1$\\
& NMI & $67.6 \pm 2.3$ & \cellcolor{yellow!30}$65.7 \pm 3.0$ & \cellcolor{yellow!30}$67.6 \pm 2.3$ & \cellcolor{yellow!30}$67.6 \pm 2.3$ & \cellcolor{yellow!30}$67.5 \pm 2.3$\\
& ARI & $56.3 \pm 6.2$ & \cellcolor{yellow!30}$55.4 \pm 5.0$ & \cellcolor{yellow!30}$56.3 \pm 6.2$ & \cellcolor{yellow!30}$56.3 \pm 6.2$ & \cellcolor{yellow!30}$56.2 \pm 6.2$\\
& $|C_1|$ & $14566.0 \pm 58.8$ & \cellcolor{green!30}$14753.0 \pm 14.8$ & \cellcolor{green!30}$14650.3 \pm 26.2$ & \cellcolor{green!30}$14736.9 \pm 21.3$ & \cellcolor{yellow!30}$14550.3 \pm 60.4$\\
\midrule
GeneExp & ACC & $99.6 \pm 0.1$ & \cellcolor{green!30}$99.4 \pm 0.1$ & \cellcolor{yellow!30}$99.6 \pm 0.1$ & \cellcolor{yellow!30}$99.5 \pm 0.1$ & \cellcolor{yellow!30}$99.5 \pm 0.1$\\
& NMI & $98.4 \pm 0.2$ & \cellcolor{green!30}$97.7 \pm 0.2$ & \cellcolor{yellow!30}$98.3 \pm 0.2$ & \cellcolor{green!30}$97.9 \pm 0.2$ & \cellcolor{yellow!30}$98.2 \pm 0.3$\\
& ARI & $99.1 \pm 0.1$ & \cellcolor{green!30}$98.7 \pm 0.1$ & \cellcolor{yellow!30}$99.1 \pm 0.1$ & \cellcolor{green!30}$98.8 \pm 0.1$ & \cellcolor{yellow!30}$99.0 \pm 0.2$\\
& $|C_1|$ & $6445.1 \pm 3.1$ & \cellcolor{green!30}$7652.8 \pm 3.0$ & \cellcolor{green!30}$7718.8 \pm 3.0$ & \cellcolor{green!30}$14135.0 \pm 0.0$ & \cellcolor{green!30}$8065.6 \pm 5.3$\\
\midrule
HDendritic & ACC & $87.5 \pm 7.2$ & \cellcolor{yellow!30}$91.2 \pm 6.1$ & \cellcolor{yellow!30}$87.5 \pm 7.2$ & \cellcolor{yellow!30}$87.5 \pm 7.2$ & \cellcolor{yellow!30}$87.5 \pm 7.2$\\
& NMI & $77.8 \pm 3.8$ & \cellcolor{yellow!30}$80.9 \pm 4.0$ & \cellcolor{yellow!30}$77.8 \pm 3.8$ & \cellcolor{yellow!30}$77.8 \pm 3.8$ & \cellcolor{yellow!30}$77.8 \pm 3.8$\\
& ARI & $79.6 \pm 5.8$ & \cellcolor{yellow!30}$83.4 \pm 5.2$ & \cellcolor{yellow!30}$79.6 \pm 5.8$ & \cellcolor{yellow!30}$79.6 \pm 5.7$ & \cellcolor{yellow!30}$79.6 \pm 5.7$\\
& $|C_1|$ & $16672.9 \pm 406.2$ & \cellcolor{green!30}$23694.0 \pm 388.1$ & \cellcolor{green!30}$17599.1 \pm 393.0$ & \cellcolor{green!30}$20094.5 \pm 274.9$ & \cellcolor{yellow!30}$16574.9 \pm 343.4$\\
\bottomrule
\end{tabular}}
\end{table*}

Generally, the different metrics show similar trends and rarely diverge. Striking differences arise mainly with ACC when dealing with highly unbalanced cluster sizes, such as in the WebKB data set. Here, the smaller clusters have significantly less influence on the final score. Consequently, the ACC is often disproportionately higher than the corresponding NMI and ARI results, which more accurately reflect the difficulty of capturing smaller groups.

\subsection{Initialization Strategies for $R_k$}

\begin{table*}
\centering
\caption{Clustering results (in \%) and the final number of columns included in $C_1$ for different initialization strategies for $R_k$. Poisson-$k$-Means++ corresponds to the strategy of \Method, proposed in Sect. \ref{sec:initial}. Representation corresponds to Tab. \ref{tab:experiments_ari}.}
\label{tab:experimentInit}
\resizebox{0.95\textwidth}{!}{
\begin{tabular}{l|l|ccccccc}
\toprule
\textbf{Data set} & \textbf{Metric} & Poisson-$k$-Means++ & random labels & random centers & $k$-Means++ & $k$-Means (full) & RF+$k$-Means++ & RF+$k$-Means (full)\\
\midrule
Synth & ACC & \bm{$98.5 \pm 0.0$} & \bm{$98.5 \pm 0.0$} & \bm{$98.5 \pm 0.0$} & \bm{$98.5 \pm 0.0$} & \bm{$98.5 \pm 0.0$} & \bm{$98.5 \pm 0.0$} & \underline{$60.0 \pm 0.0$}\\
& NMI & \bm{$92.5 \pm 0.0$} & \bm{$92.5 \pm 0.0$} & \bm{$92.5 \pm 0.0$} & \bm{$92.5 \pm 0.0$} & \bm{$92.5 \pm 0.0$} & \bm{$92.5 \pm 0.0$} & \underline{$44.6 \pm 0.0$}\\
& ARI & \bm{$95.5 \pm 0.0$} & \bm{$95.5 \pm 0.0$} & \bm{$95.5 \pm 0.0$} & \bm{$95.5 \pm 0.0$} & \bm{$95.5 \pm 0.0$} & \bm{$95.5 \pm 0.0$} & \underline{$37.1 \pm 0.0$}\\
& $|C_1|$ & \bm{$2.0 \pm 0.0$} & \bm{$2.0 \pm 0.0$} & \bm{$2.0 \pm 0.0$} & \bm{$2.0 \pm 0.0$} & \bm{$2.0 \pm 0.0$} & \bm{$2.0 \pm 0.0$} & \underline{$3.0 \pm 0.0$}\\
\midrule
Wholesales & ACC & \underline{$77.5 \pm 0.0$} & \underline{$77.5 \pm 0.0$} & \underline{$77.5 \pm 0.0$} & \bm{$77.7 \pm 0.3$} & \underline{$77.5 \pm 0.0$} & \underline{$77.5 \pm 0.0$} & \underline{$77.5 \pm 0.0$}\\
& NMI & \bm{$31.9 \pm 0.0$} & \bm{$31.9 \pm 0.0$} & \bm{$31.9 \pm 0.0$} & \underline{$29.2 \pm 3.9$} & \bm{$31.9 \pm 0.0$} & \bm{$31.9 \pm 0.0$} & \bm{$31.9 \pm 0.0$}\\
& ARI & \underline{$30.1 \pm 0.0$} & \underline{$30.1 \pm 0.0$} & \underline{$30.1 \pm 0.0$} & \bm{$30.5 \pm 0.5$} & \underline{$30.1 \pm 0.0$} & \underline{$30.1 \pm 0.0$} & \underline{$30.1 \pm 0.0$}\\
& $|C_1|$ & \bm{$5.0 \pm 0.0$} & \bm{$5.0 \pm 0.0$} & \bm{$5.0 \pm 0.0$} & \underline{$5.3 \pm 0.4$} & \bm{$5.0 \pm 0.0$} & \bm{$5.0 \pm 0.0$} & \bm{$5.0 \pm 0.0$}\\
\midrule
SportA & ACC & \underline{$78.5 \pm 0.2$} & \underline{$78.5 \pm 0.1$} & \underline{$78.5 \pm 0.2$} & \underline{$78.5 \pm 0.2$} & \dashuline{$78.2 \pm 0.0$} & \bm{$78.6 \pm 0.2$} & $76.5 \pm 0.0$\\
& NMI & \underline{$24.1 \pm 0.5$} & \dashuline{$24.0 \pm 0.1$} & \dashuline{$24.0 \pm 0.5$} & \dashuline{$24.0 \pm 0.4$} & $23.8 \pm 0.0$ & \bm{$24.2 \pm 0.5$} & $20.5 \pm 0.0$\\
& ARI & \underline{$32.3 \pm 0.5$} & \underline{$32.3 \pm 0.1$} & \underline{$32.3 \pm 0.5$} & \underline{$32.3 \pm 0.3$} & \dashuline{$31.7 \pm 0.0$} & \bm{$32.5 \pm 0.4$} & $27.9 \pm 0.0$\\
& $|C_1|$ & \dashuline{$41.1 \pm 1.4$} & \underline{$40.6 \pm 0.5$} & \dashuline{$41.1 \pm 1.4$} & \dashuline{$41.1 \pm 0.7$} & \bm{$39.0 \pm 0.0$} & $41.9 \pm 1.1$ & $46.0 \pm 0.0$\\
\midrule
Optdigits & ACC & \dashuline{$80.0 \pm 2.0$} & \underline{$80.1 \pm 2.0$} & \bm{$80.9 \pm 0.3$} & $78.6 \pm 3.1$ & $76.0 \pm 3.1$ & $79.3 \pm 2.8$ & $77.1 \pm 0.1$\\
& NMI & \underline{$75.1 \pm 0.9$} & \underline{$75.1 \pm 1.0$} & \bm{$75.5 \pm 0.2$} & $74.4 \pm 1.5$ & $73.2 \pm 1.5$ & \dashuline{$74.8 \pm 1.3$} & $74.0 \pm 0.1$\\
& ARI & \dashuline{$66.7 \pm 2.2$} & \underline{$66.8 \pm 2.3$} & \bm{$67.6 \pm 0.4$} & $65.3 \pm 3.2$ & $62.0 \pm 3.6$ & $65.9 \pm 3.0$ & $64.6 \pm 0.1$\\
& $|C_1|$ & $55.8 \pm 0.6$ & \bm{$54.0 \pm 0.0$} & $55.4 \pm 0.9$ & \dashuline{$55.0 \pm 1.0$} & \underline{$54.2 \pm 0.6$} & \dashuline{$55.0 \pm 1.0$} & $55.6 \pm 0.8$\\
\midrule
BBCSports & ACC & \bm{$96.1 \pm 1.2$} & $89.9 \pm 9.1$ & $91.2 \pm 5.0$ & $91.5 \pm 6.7$ & $88.1 \pm 8.0$ & \dashuline{$92.0 \pm 5.3$} & \underline{$93.2 \pm 5.6$}\\
& NMI & \bm{$89.4 \pm 2.0$} & $83.5 \pm 7.7$ & $84.2 \pm 3.2$ & \dashuline{$85.8 \pm 5.0$} & $84.2 \pm 3.8$ & $83.9 \pm 5.7$ & \underline{$87.0 \pm 5.0$}\\
& ARI & \bm{$90.1 \pm 2.9$} & $82.8 \pm 10.6$ & $84.4 \pm 3.5$ & \dashuline{$84.7 \pm 7.2$} & $82.7 \pm 5.8$ & $83.7 \pm 6.5$ & \underline{$86.9 \pm 5.8$}\\
& $|C_1|$ & $799.2 \pm 7.6$ & $796.8 \pm 7.9$ & $798.0 \pm 19.1$ & $792.8 \pm 21.1$ & \bm{$782.3 \pm 22.7$} & \underline{$785.3 \pm 17.1$} & \dashuline{$791.8 \pm 13.2$}\\
\midrule
BBCNews & ACC & \underline{$95.7 \pm 0.4$} & \bm{$96.0 \pm 0.5$} & \underline{$95.7 \pm 0.6$} & $95.5 \pm 0.8$ & $95.5 \pm 0.7$ & $95.4 \pm 0.5$ & \dashuline{$95.6 \pm 0.5$}\\
& NMI & \underline{$87.2 \pm 0.8$} & \bm{$87.7 \pm 1.1$} & \dashuline{$87.1 \pm 1.5$} & $86.7 \pm 1.7$ & $86.8 \pm 1.6$ & $86.6 \pm 1.0$ & $86.9 \pm 1.0$\\
& ARI & \underline{$89.9 \pm 0.9$} & \bm{$90.5 \pm 1.2$} & \underline{$89.9 \pm 1.5$} & $89.4 \pm 1.9$ & $89.4 \pm 1.8$ & $89.3 \pm 1.2$ & \dashuline{$89.5 \pm 1.0$}\\
& $|C_1|$ & $1379.0 \pm 6.2$ & $1378.4 \pm 5.5$ & $1380.4 \pm 5.5$ & \bm{$1376.2 \pm 6.3$} & \dashuline{$1376.8 \pm 6.7$} & \underline{$1376.3 \pm 4.2$} & $1378.2 \pm 4.7$\\
\midrule
WebKB & ACC & \underline{$57.4 \pm 2.7$} & \bm{$58.8 \pm 2.3$} & \dashuline{$57.1 \pm 1.9$} & $56.0 \pm 3.8$ & $56.3 \pm 3.4$ & $55.4 \pm 2.7$ & $55.0 \pm 3.7$\\
& NMI & \dashuline{$37.2 \pm 1.2$} & \bm{$37.9 \pm 0.4$} & \underline{$37.3 \pm 0.8$} & $36.4 \pm 1.5$ & $36.7 \pm 1.1$ & $36.1 \pm 1.5$ & $36.6 \pm 1.3$\\
& ARI & \underline{$31.5 \pm 3.2$} & \bm{$32.6 \pm 2.7$} & \dashuline{$31.4 \pm 2.3$} & $30.0 \pm 4.0$ & $30.8 \pm 3.3$ & $30.2 \pm 2.5$ & $30.9 \pm 2.8$\\
& $|C_1|$ & $1314.9 \pm 18.4$ & $1326.4 \pm 17.5$ & $1317.2 \pm 8.2$ & \dashuline{$1312.9 \pm 31.8$} & $1313.1 \pm 26.1$ & \underline{$1309.0 \pm 9.9$} & \bm{$1307.6 \pm 25.5$}\\
\midrule
Reuters & ACC & $80.9 \pm 2.9$ & \underline{$81.3 \pm 2.6$} & $79.6 \pm 3.8$ & \dashuline{$81.1 \pm 6.5$} & \bm{$84.9 \pm 4.6$} & $78.9 \pm 2.4$ & $77.6 \pm 4.9$\\
& NMI & \dashuline{$70.8 \pm 2.9$} & $70.5 \pm 1.7$ & $69.7 \pm 2.1$ & \underline{$71.5 \pm 3.1$} & \bm{$72.1 \pm 1.9$} & $69.0 \pm 2.5$ & $67.7 \pm 2.5$\\
& ARI & \bm{$67.8 \pm 1.6$} & \underline{$67.6 \pm 0.7$} & $65.8 \pm 2.9$ & $65.9 \pm 4.3$ & \dashuline{$66.6 \pm 2.8$} & $65.0 \pm 3.7$ & $62.9 \pm 5.1$\\
& $|C_1|$ & $1321.9 \pm 7.9$ & $1321.0 \pm 9.7$ & $1318.9 \pm 11.7$ & $1316.3 \pm 15.2$ & \underline{$1306.5 \pm 17.1$} & \dashuline{$1311.5 \pm 28.7$} & \bm{$1301.8 \pm 32.8$}\\
\midrule
20NewsG & ACC & \underline{$38.4 \pm 2.5$} & $37.0 \pm 1.4$ & \underline{$38.4 \pm 2.4$} & $25.9 \pm 1.3$ & $25.8 \pm 1.5$ & \dashuline{$37.9 \pm 2.0$} & \bm{$39.5 \pm 1.5$}\\
& NMI & \underline{$39.9 \pm 1.0$} & $39.1 \pm 0.8$ & \underline{$39.9 \pm 0.7$} & $32.9 \pm 1.6$ & $32.8 \pm 2.0$ & \dashuline{$39.7 \pm 0.7$} & \bm{$40.6 \pm 0.7$}\\
& ARI & \dashuline{$23.0 \pm 1.3$} & $22.3 \pm 1.3$ & \dashuline{$23.0 \pm 0.8$} & $16.2 \pm 1.2$ & $16.1 \pm 1.5$ & \underline{$23.2 \pm 1.4$} & \bm{$23.7 \pm 1.0$}\\
& $|C_1|$ & $1453.7 \pm 15.3$ & \dashuline{$1405.2 \pm 8.9$} & $1439.8 \pm 11.8$ & \underline{$1144.3 \pm 25.3$} & \bm{$1135.0 \pm 21.9$} & $1444.9 \pm 14.8$ & $1458.0 \pm 11.4$\\
\midrule
MouseAtlas & ACC & $67.5 \pm 6.0$ & $65.9 \pm 2.5$ & \dashuline{$67.7 \pm 2.6$} & $64.2 \pm 3.2$ & $63.2 \pm 2.0$ & \underline{$68.1 \pm 12.5$} & \bm{$74.5 \pm 3.0$}\\
& NMI & $67.6 \pm 2.3$ & $67.6 \pm 2.1$ & \underline{$68.1 \pm 1.4$} & $66.5 \pm 3.3$ & \dashuline{$68.0 \pm 0.5$} & $64.2 \pm 10.9$ & \bm{$70.6 \pm 1.6$}\\
& ARI & \dashuline{$56.3 \pm 6.2$} & $55.4 \pm 3.5$ & \underline{$56.5 \pm 3.4$} & $52.6 \pm 5.3$ & $54.3 \pm 0.9$ & $51.2 \pm 21.8$ & \bm{$63.2 \pm 2.9$}\\
& $|C_1|$ & \dashuline{$14566.0 \pm 58.8$} & $14662.8 \pm 7.2$ & $14576.3 \pm 71.2$ & \underline{$14524.3 \pm 35.0$} & \bm{$14508.8 \pm 17.1$} & $14616.6 \pm 64.7$ & $14605.5 \pm 59.3$\\
\midrule
GeneExp & ACC & \bm{$99.6 \pm 0.1$} & \bm{$99.6 \pm 0.1$} & \underline{$92.8 \pm 7.2$} & \bm{$99.6 \pm 0.0$} & \bm{$99.6 \pm 0.0$} & \bm{$99.6 \pm 0.1$} & \bm{$99.6 \pm 0.0$}\\
& NMI & \underline{$98.4 \pm 0.2$} & \dashuline{$98.3 \pm 0.2$} & $93.6 \pm 4.7$ & \underline{$98.4 \pm 0.1$} & \bm{$98.5 \pm 0.0$} & \dashuline{$98.3 \pm 0.2$} & \bm{$98.5 \pm 0.0$}\\
& ARI & \underline{$99.1 \pm 0.1$} & \underline{$99.1 \pm 0.2$} & $91.0 \pm 8.3$ & \underline{$99.1 \pm 0.1$} & \bm{$99.2 \pm 0.0$} & \dashuline{$99.0 \pm 0.1$} & \bm{$99.2 \pm 0.0$}\\
& $|C_1|$ & $6445.1 \pm 3.1$ & $6446.0 \pm 2.8$ & \bm{$6272.2 \pm 239.6$} & \dashuline{$6444.4 \pm 2.7$} & \underline{$6443.0 \pm 0.0$} & $6447.3 \pm 3.2$ & \underline{$6443.0 \pm 0.0$}\\
\midrule
HDendritic & ACC & \bm{$87.5 \pm 7.2$} & \dashuline{$81.2 \pm 5.9$} & \underline{$84.3 \pm 7.6$} & $49.1 \pm 14.5$ & $46.0 \pm 14.1$ & $72.1 \pm 6.8$ & $66.3 \pm 5.1$\\
& NMI & \bm{$77.8 \pm 3.8$} & \dashuline{$74.6 \pm 3.4$} & \underline{$76.4 \pm 4.5$} & $27.8 \pm 25.1$ & $22.1 \pm 24.4$ & $64.1 \pm 9.6$ & $55.4 \pm 7.1$\\
& ARI & \bm{$79.6 \pm 5.8$} & \dashuline{$74.8 \pm 4.9$} & \underline{$76.9 \pm 6.7$} & $19.7 \pm 18.9$ & $15.3 \pm 18.1$ & $58.2 \pm 15.0$ & $45.0 \pm 11.4$\\
& $|C_1|$ & $16672.9 \pm 406.2$ & $16635.5 \pm 406.8$ & $16601.9 \pm 483.0$ & \underline{$5940.5 \pm 4774.3$} & \bm{$5667.2 \pm 3879.9$} & $15235.7 \pm 1694.7$ & \dashuline{$13682.6 \pm 1782.3$}\\
\bottomrule
\end{tabular}}
\end{table*}

We conduct a series of experiments to evaluate the influence of our proposed initialization strategy for the cluster assignments $R_k$ (Sect. \ref{sec:initial}) -- named Poisson-$k$-Means++. We compare against random cluster assignments (random labels), picking random rows followed by the assignment method described in Sect~\ref{sec:initial} (random centers), the seeding mechanism of $k$-Means++ using the Euclidean distance, and a full execution of $k$-Means. For the $k$-Means++ version, the labels are received by determining for each row ${\bf X}_{i \cdot}$ the closest selected row based on the Euclidean distance. For $k$-Means++ and $k$-Means (full), we also test a version in which we first divide each row by the row-sum (named as RF+). The results are shown in Table \ref{tab:experimentInit}.

We can see that all versions perform similarly for Synth, Wholesales, and GeneExp.
Furthermore, it is evident that there is no clear winner in all other scenarios. Poisson-$k$-Means++, random labels, and random centers all provide good solutions across the benchmark. However, if one considers the better results for BBCSports and HDendritic, the version of \Method using Poisson-$k$-Means++ is a reasonable choice. 

\subsection{Evaluation of the Runtime}

\begin{figure*}[t]
    \centering
    \begin{subfigure}{0.385\textwidth}
        \centering
        \includegraphics[width=\textwidth]{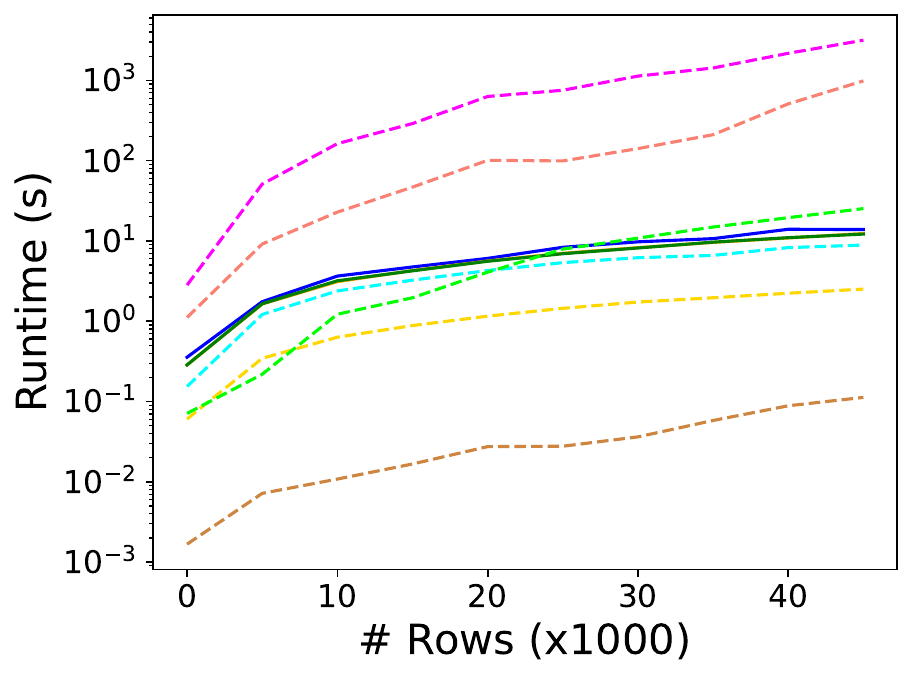}
        \caption{Runtime with an increasing $n$ ($m=6$, logarithmic scale).}
        \label{fig:runtime_n}
    \end{subfigure}
    \begin{subfigure}{0.385\textwidth}
        \centering
        \includegraphics[width=\textwidth]{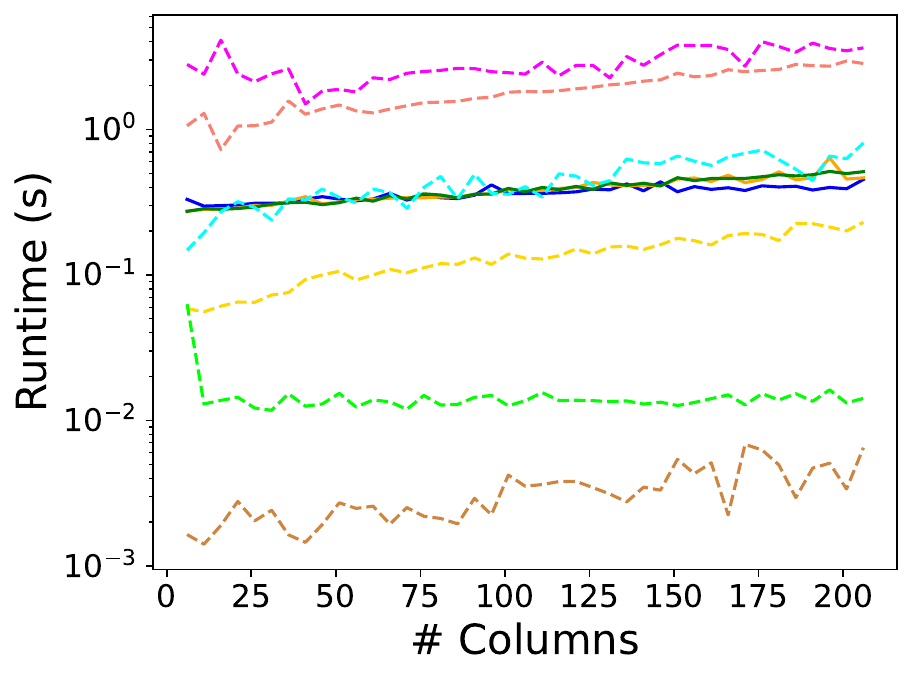}
        \caption{Runtime with an increasing $m$ ($n=1000$, logarithmic scale).}
        \label{fig:runtime_m}
    \end{subfigure}
    \begin{subfigure}{0.21\textwidth}
        \centering
        \includegraphics[width=\textwidth]{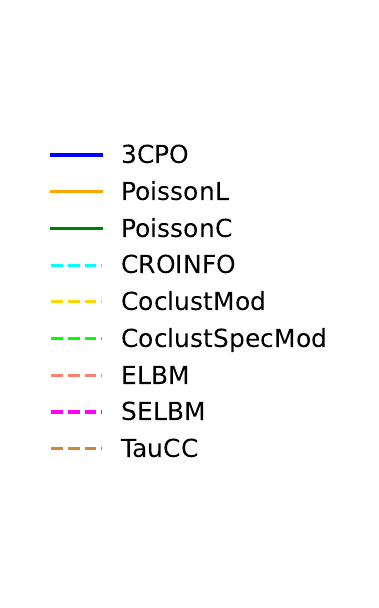}
        \vspace{0.5cm}
    \end{subfigure}
    \caption{Runtime experiments investigating the impact of the number of rows $n$ and columns $m$. Both experiments build on the Synth data set, where the number of clusters $K=3$. The experiments are run on an Apple M3 Pro with $16$GB memory.}
    \label{fig:runtime}
\end{figure*}

We compare the runtime of the Poisson-based approaches \Method, PoissonL, and PoissonC and the co-clustering algorithms. Here, each entry corresponds to the sum of ten runs, except for TauCC, where the entries correspond to a single run. In the first experiment (Fig. \ref{fig:runtime_n}), we use the Synth data set as described above and increase the number of rows $n \in [1\,000, 50\,000]$. The runtime of \Method, PoissonL and PoissonC is very similar and shows a linear increase with $n$. \Method is slightly slower as it updates the column partitions in each iteration and needs to re-calculate $\lambda_i^R$, while it is only computed once for PoissonL and PoissonC. CoclustMod and CROINFO also show a similar behavior while the runtime of CoclustSpecMod, ELBM, SELBM and TauCC increases exponentially.

If we increase the number of columns in $C_0$ for Synth as done in the second experiment (Fig. \ref{fig:runtime_m}), we can see that \Method is slightly slower than PoissonL and PoissonC in the beginning. However, it is faster for large $m$ as \Method assigns the additional columns to $C_0$ and, thus, needs to consider fewer columns when updating the cluster assignments $R_k$. The co-clustering approaches CoclustMod, CROINFO, ELBM, and TauCC show an exponential increase of the runtime with increasing $m$. CoclustSpecMod often terminated immediately after initialization and appears to have problems with the structure of the data.

In summary, the experiments confirm the linear complexity of \Method with respect to $n$ and $m$ as described in Sect. \ref{sec:observations}.

\subsection{Detected Outliers for Optdigits}\label{sec:outliers_optdigits}

Inspecting the outliers identified by \Method within the Optdigits data set, it is noticeable that more than half of those belong to cluster eight according to the ground truth. Considering that our outlier detection uses a distribution shared across all clusters and considering the row-scaling due to $\lambda_i^R$, the outlier definition is prone to rows that use similar values for all pixels. As the digit eight requires the most ``active'' pixels, it is the most likely to fall into the outlier category if it deviates too much from an average writing style. 

To validate this, we visually examine a random set of inliers and outliers and find that the majority of outliers exhibit genuinely unusual characteristics despite the class imbalance in the outlier set. A common trait is, for example, a gap within the structure of the eight.
Fig.~\ref{fig:optdigits_outliers} shows a random selection of inliers and outliers which showcase usual and unusual characteristics.

\begin{figure*}[t]
    \centering
    \begin{subfigure}{0.54\textwidth}
    \centering
        \includegraphics[width=0.22\textwidth]{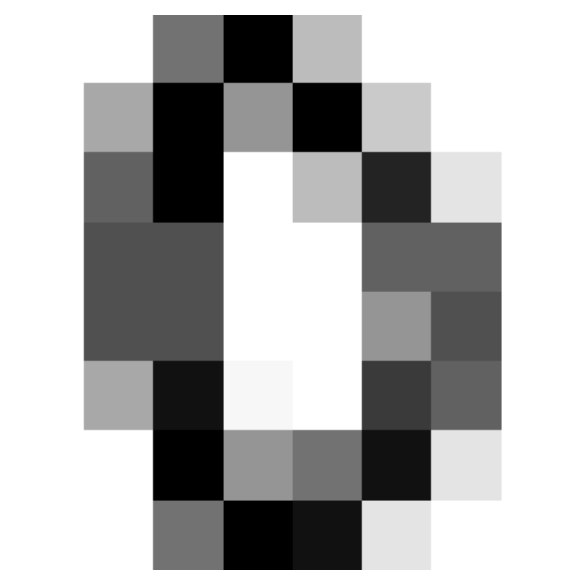}
        \includegraphics[width=0.22\textwidth]{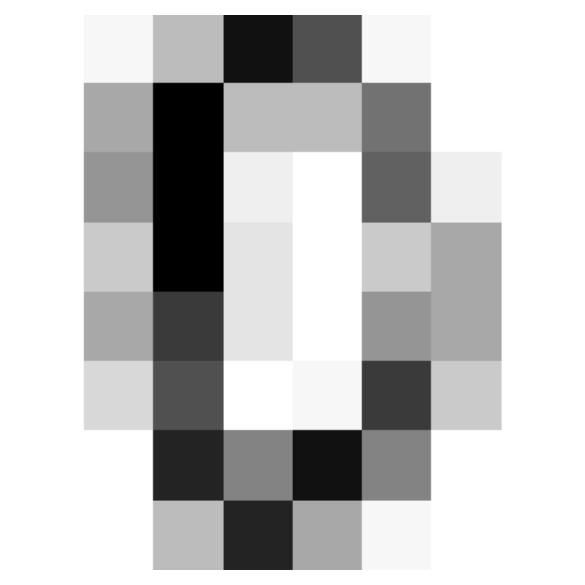}
        \includegraphics[width=0.22\textwidth]{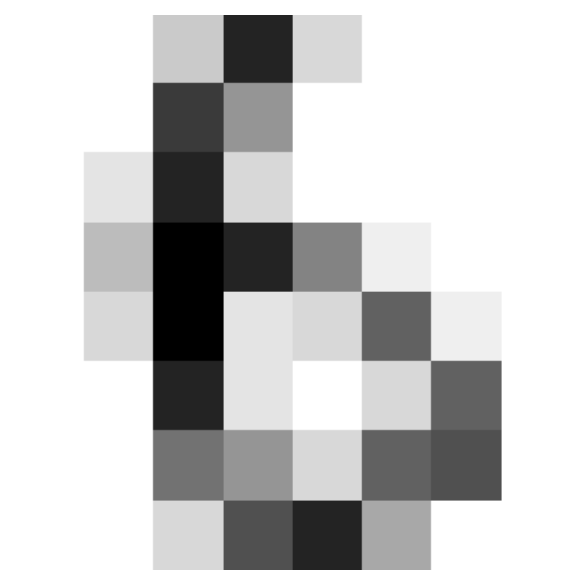}
        \includegraphics[width=0.22\textwidth]{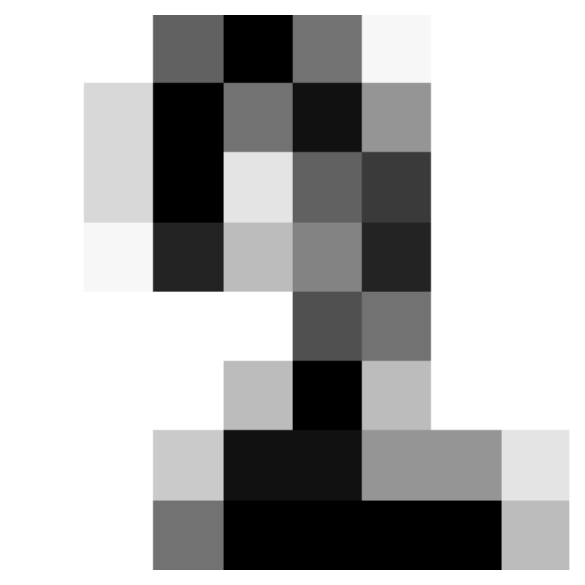}
        \includegraphics[width=0.22\textwidth]{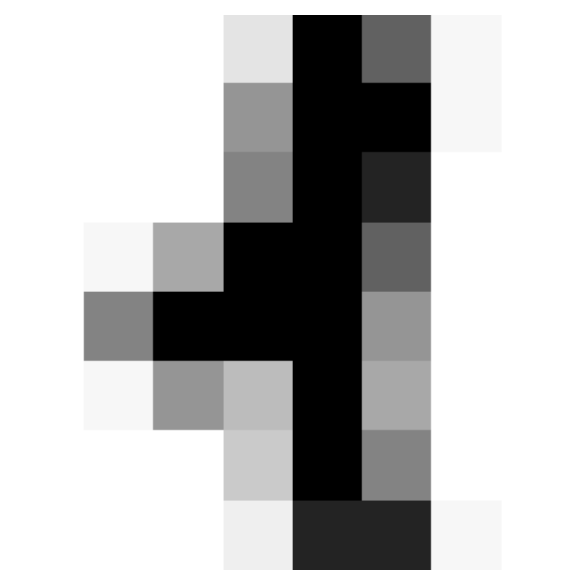}
        \includegraphics[width=0.22\textwidth]{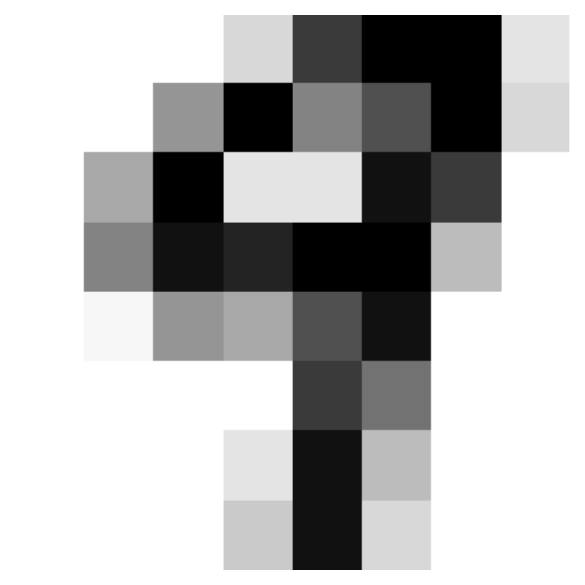}
        \includegraphics[width=0.22\textwidth]{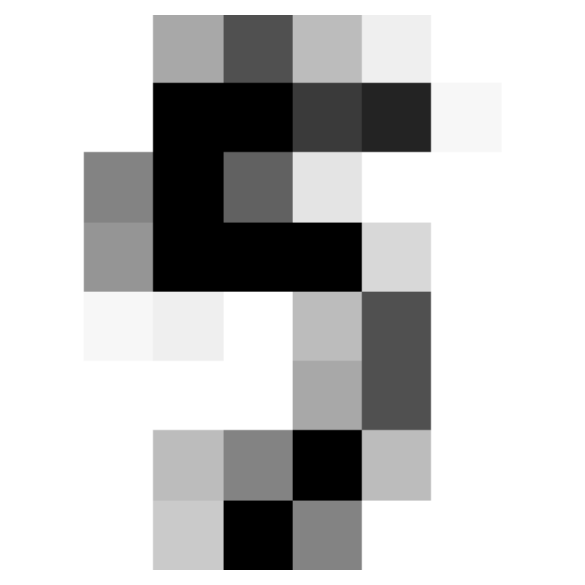}
        \includegraphics[width=0.22\textwidth]{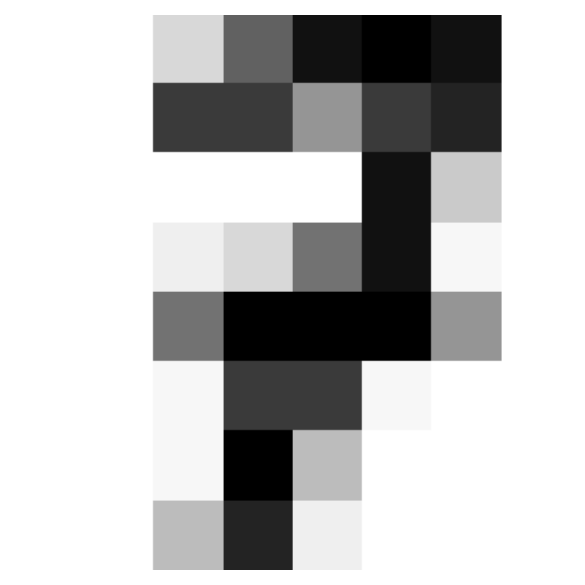}
        \caption{Inliers.}
    \end{subfigure}
    \begin{subfigure}{0.54\textwidth}
        \centering
        \includegraphics[width=0.22\textwidth]{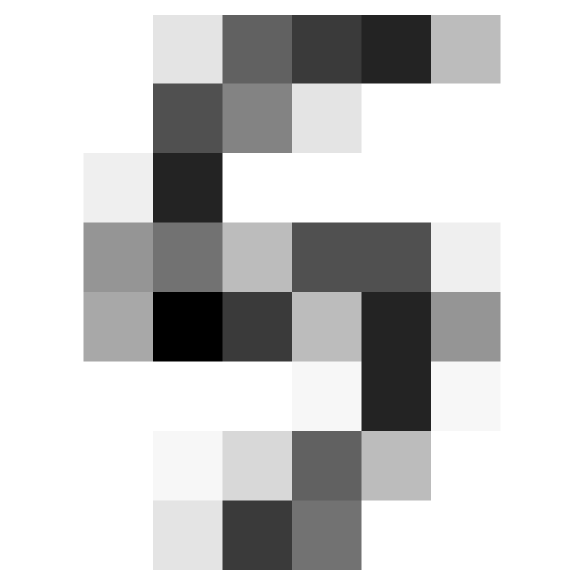}
        \includegraphics[width=0.22\textwidth]{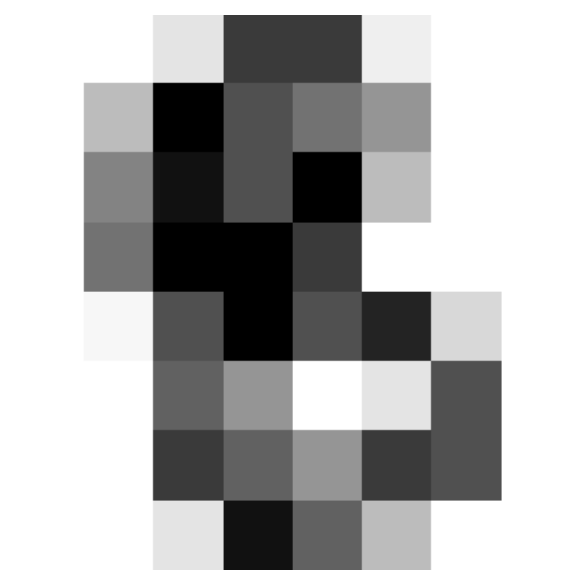}
        \includegraphics[width=0.22\textwidth]{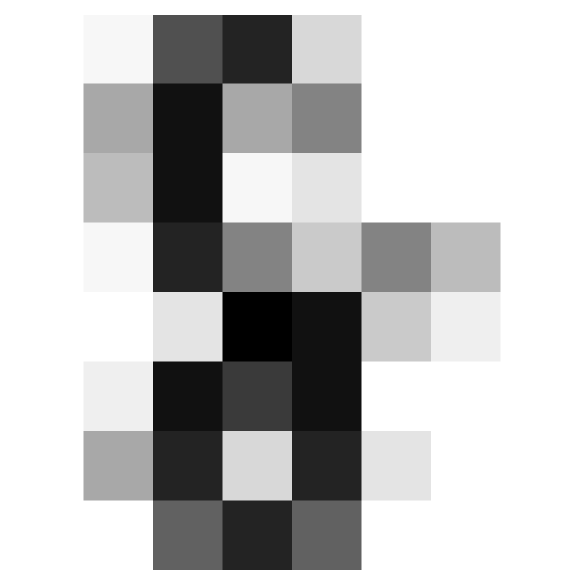}
        \includegraphics[width=0.22\textwidth]{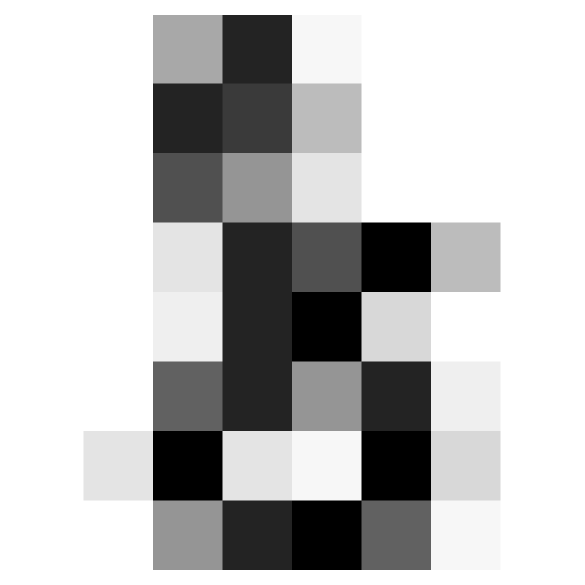}
        \includegraphics[width=0.22\textwidth]{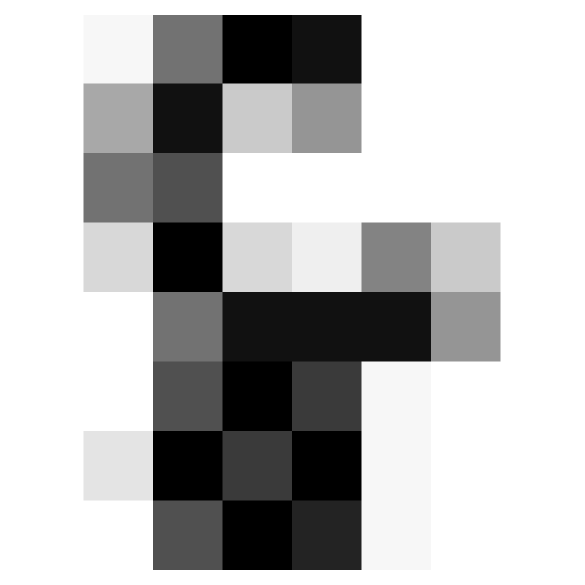}
        \includegraphics[width=0.22\textwidth]{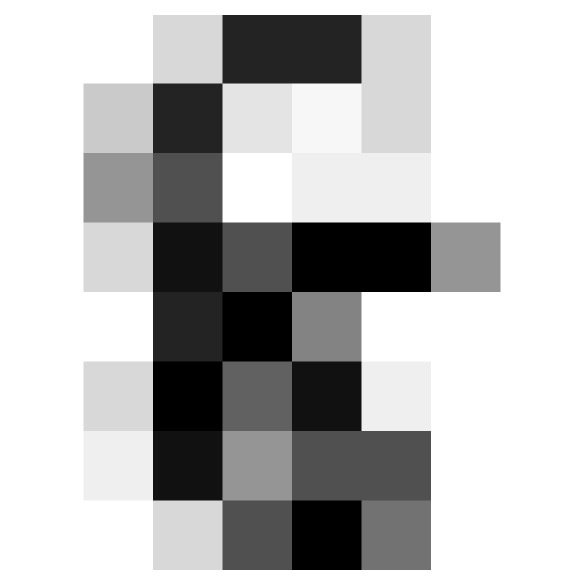}
        \includegraphics[width=0.22\textwidth]{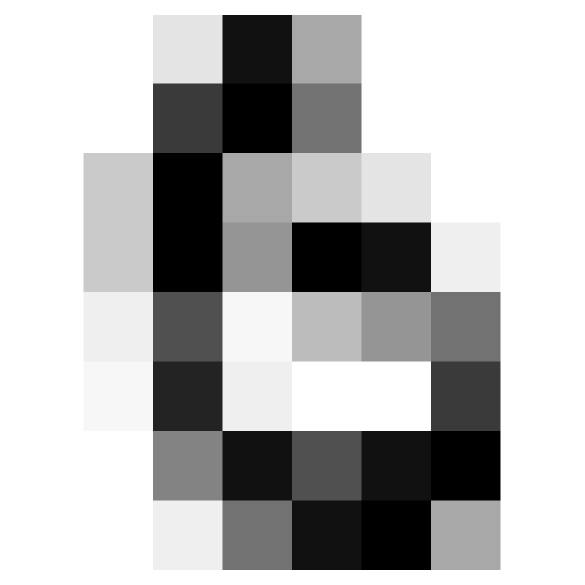}
        \includegraphics[width=0.22\textwidth]{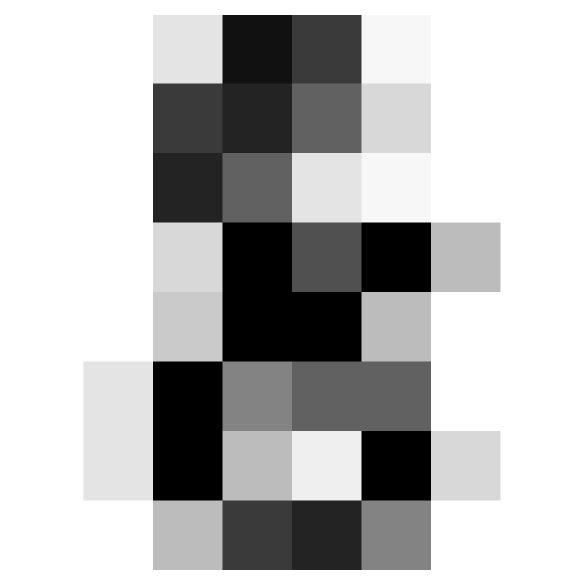}
        \caption{Outliers.}
    \end{subfigure}
    \caption{Visual examples of eight inliers and eight outliers identified by \Method in the Optdigits data set.}
    \label{fig:optdigits_outliers}
\end{figure*}

\section{Estimating the Number of Clusters for $k$-Means-based Approaches}\label{sec:estimate_k_kmeans}

In the following, we evaluate whether the $K$-estimation strategy used by \Method outperforms established strategies for $k$-Means. We utilize three criteria to estimate $K$ for $k$-Means: (1) elbows in the inertia curve \citep{thorndike1953belongs} detected by the Needle algorithm~\citep{satopaa2011finding}, (2) the Silhouette score~\citep{rousseeuw1987silhouettes}, and (3) the BIC score as defined for $X$-Means~\citep{pelleg2000x}. For visualization, inertia and BIC are min–max normalized for each algorithm; the Silhouette score is bound within $[-1,1]$. The remaining setting is equal to the one described in Sect.~\ref{sec:estimate_k}. We vary the number of clusters $K\in\{2,3,\dots,30\}$ and, for each $K$, run each clustering method $20$ times. In the end, we only evaluate the run with the best inertia. The results for Synth, GeneExp and BBCSports are illustrated in Fig.~\ref{fig:k_estimate_km}.

We observe that all standard $K$-estimation techniques struggle to identify the ground-truth number of clusters for the Synth data set. Specifically, the elbow method suggests $K \approx 7$, while the Silhouette score recommends $K=2$ and the BIC score suggests $K=30$. For the GeneExp data set, the $K$-estimation is more consistent; the elbow method and Silhouette score both indicate $K \approx 6$ for most algorithms (excluding RCA+KM), while the BIC score recommends $K=12$ or $K=13$. 

For BBCSports, both the elbow method and Silhouette score provide poor recommendations, whereas BIC yields $K$ within $[3, 8]$ for all algorithms except standard $k$-Means (KM). Overall, we do not have a clear winner that is performing well in all scenarios. Our proposed method \Method, on the other hand, was able to provide a reasonable number of clusters in all three cases.

In addition to the visual examples presented in Fig.~\ref{fig:k_estimate_km}, we report results for all data sets in Tab.~\ref{tab:estimateK}, stating the estimated number of clusters $K$ together with the corresponding ARI and Purity~\citep{manning2008introduction} values. We include Purity to assess how pure clusters are when $K$ is overestimated. Since Purity increases mechanically with larger $K$, it should be interpreted alongside the $K$-estimation error.

Across data sets, the Silhouette score tends to underestimate $K$ as $56$ of $62$ outcomes are $K=2$. This leads to good results for Wholesales and SportA. In contrast, the BIC score often overestimates $K$, returning the upper bound $K \in \{29, 30\}$ in $44$ cases (note that the maximum $K$ is $30$ in our experiments). This often leads to high Purity values but significantly lower ARI results. It should be noted that to some degree these weaknesses may be related to the data transformations that have been conducted (STD, MM, RF, RCA). 

For Optdigits, elbow and Silhouette lead to very good results, providing estimates of $K=9$ or $K=10$ for all $k$-Means-based algorithms except RCA+KM. In this case, the results for \Method and TauCC are far worse. In general, \Method struggles to match the ground truth $K$ on most tabular data sets. However, it performs notably better than the competitors on high-dimensional text data, achieving higher ARI and Purity values in most experiments. Concretely, it estimates the correct $K$ for Synth and BBCSports, is only off by one for BBCNews and GeneExp, and is close for 20NewsG and WebKB. Based on the results for Synth and 20NewsG, \Method appears capable of accurately estimating both low and high cluster counts.

Regarding the co-clustering algorithm TauCC, which determines $K$ intrinsically, we find that it correctly identifies $K=3.0$ (averaged across runs) for Synth. However, it tends to underestimate the number of clusters for GeneExp and BBCSports, returning an average of $K=3.0$ for both. This trend of underestimation is consistent across other data sets, with averages of $K=2.1$, $K=9.5$, and $K=3.6$ for Optdigits, 20NewsG, and MouseAtlas, respectively.

\begin{figure*}[t]
    \centering
    \begin{subfigure}{1\textwidth}
        \centering
        \includegraphics[width=0.3\textwidth]{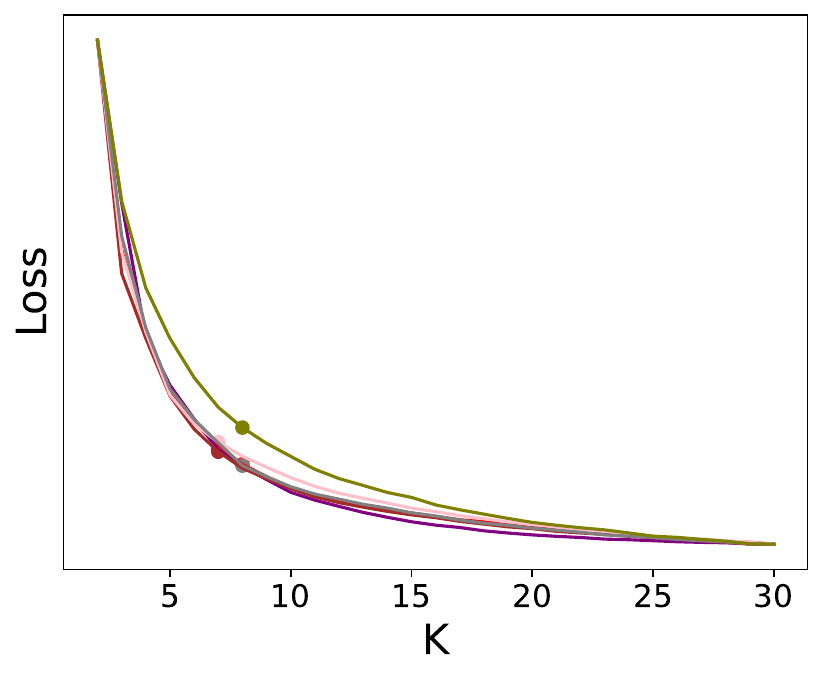}
        \includegraphics[width=0.315\textwidth]{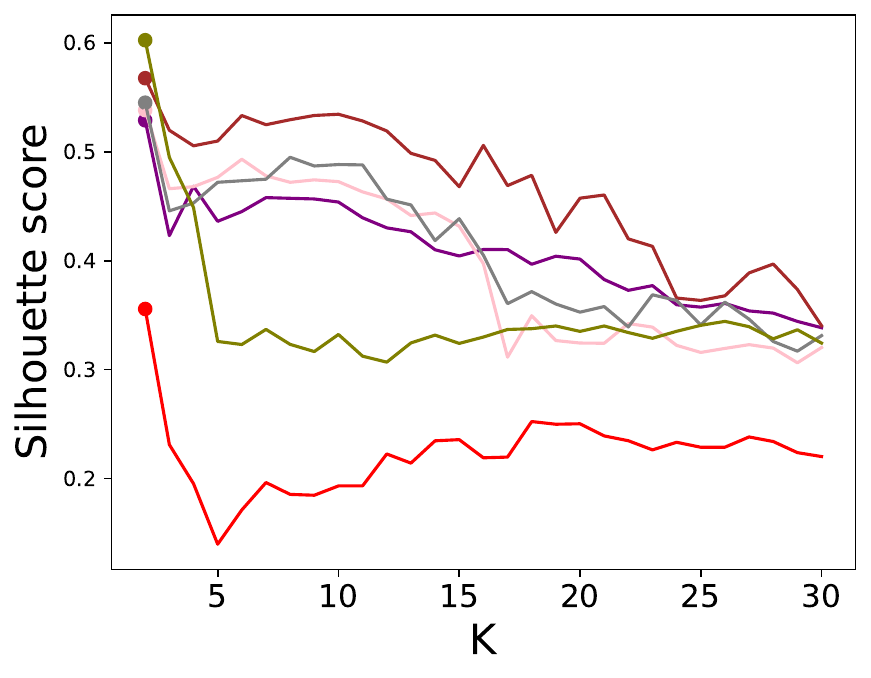}
        \includegraphics[width=0.3\textwidth]{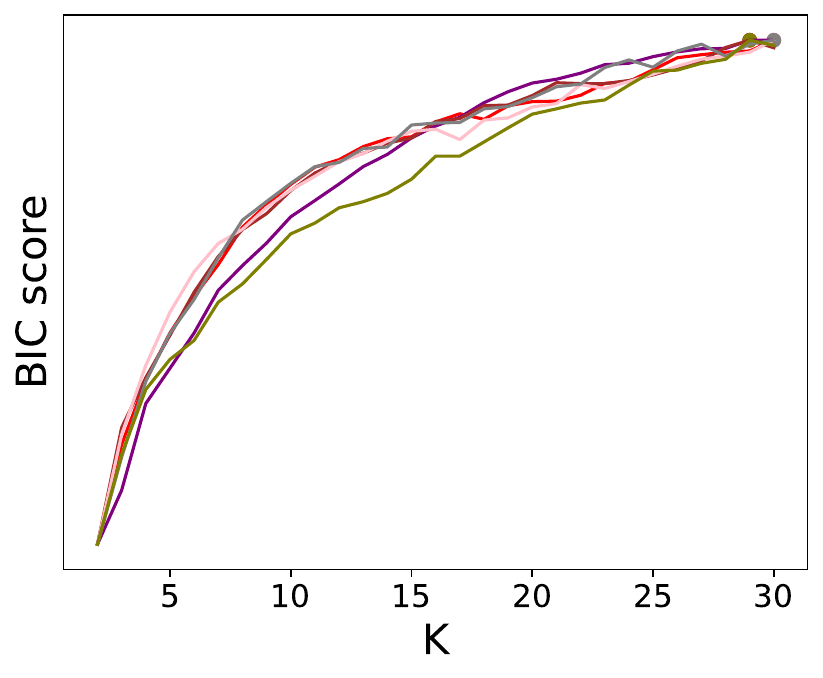}
        \caption{Loss (inertia), Silhouette score, and BIC score for the Synth data set ($K_{gt}=3$).}
    \end{subfigure}
    \begin{subfigure}{1\textwidth}
        \centering
        \includegraphics[width=0.3\textwidth]{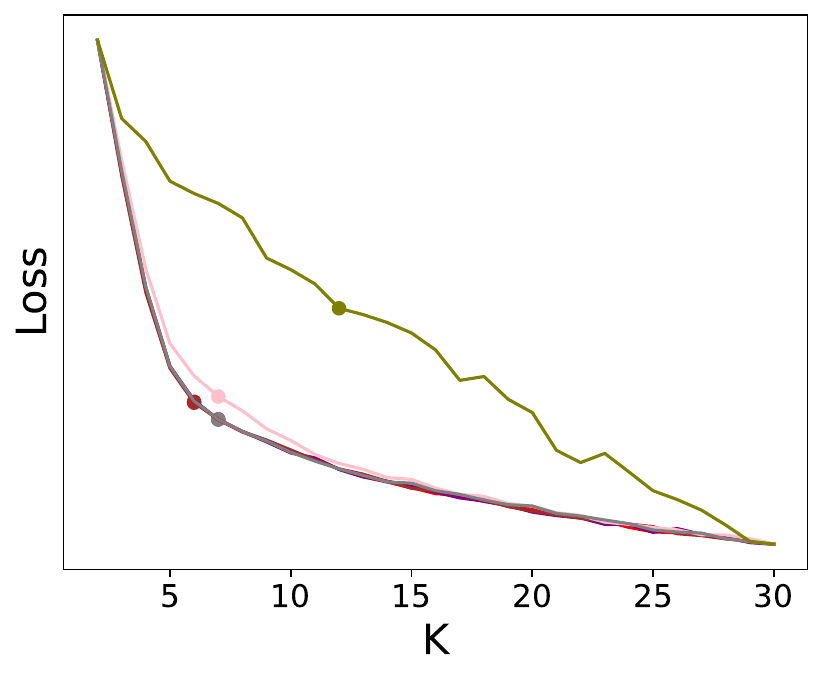}
        \includegraphics[width=0.315\textwidth]{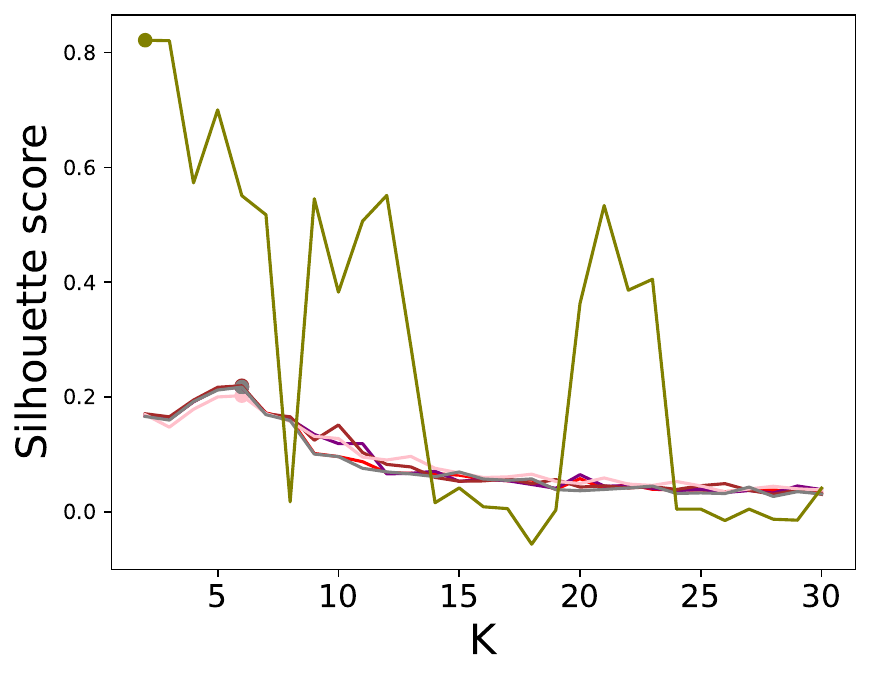}
        \includegraphics[width=0.3\textwidth]{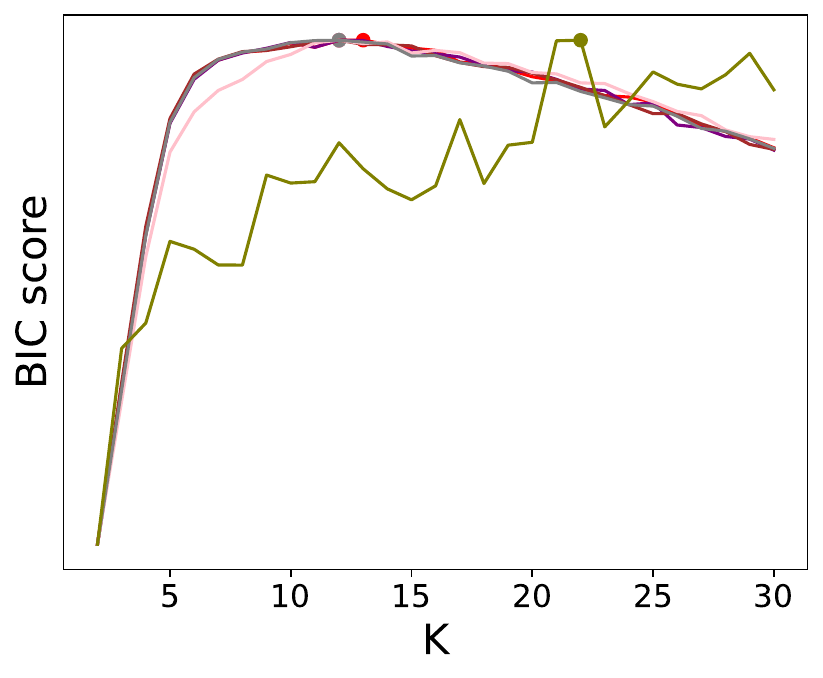}
        \caption{Loss (inertia), Silhouette score, and BIC score for the GeneExp data set ($K_{gt}=5$).}
    \end{subfigure}
    \begin{subfigure}{1\textwidth}
        \centering
        \includegraphics[width=0.3\textwidth]{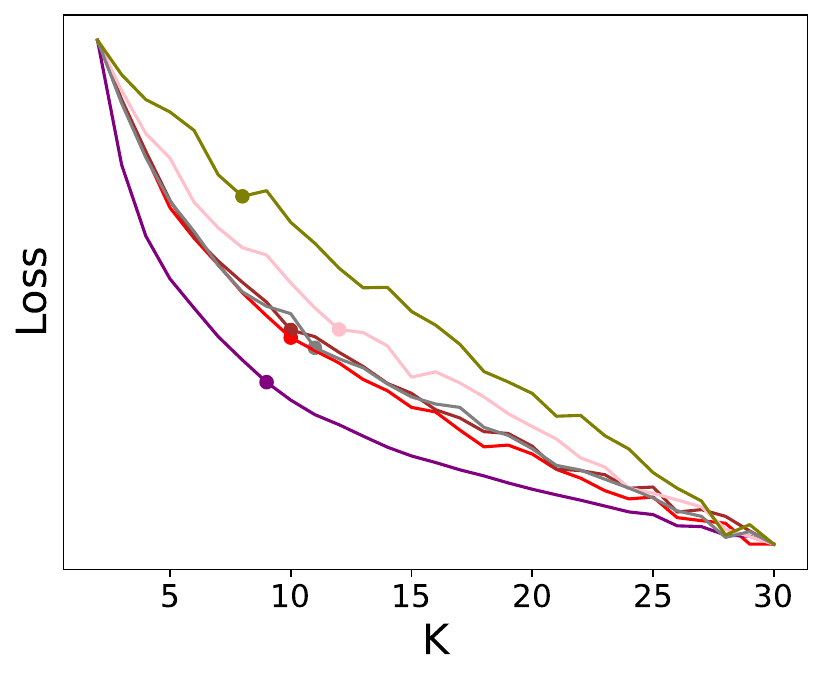}
        \includegraphics[width=0.315\textwidth]{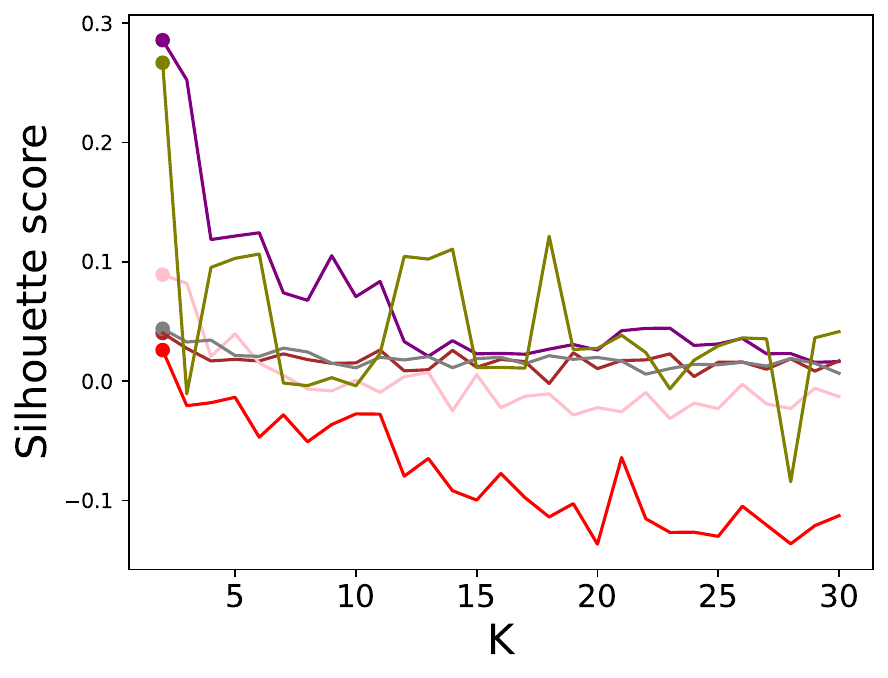}
        \includegraphics[width=0.3\textwidth]{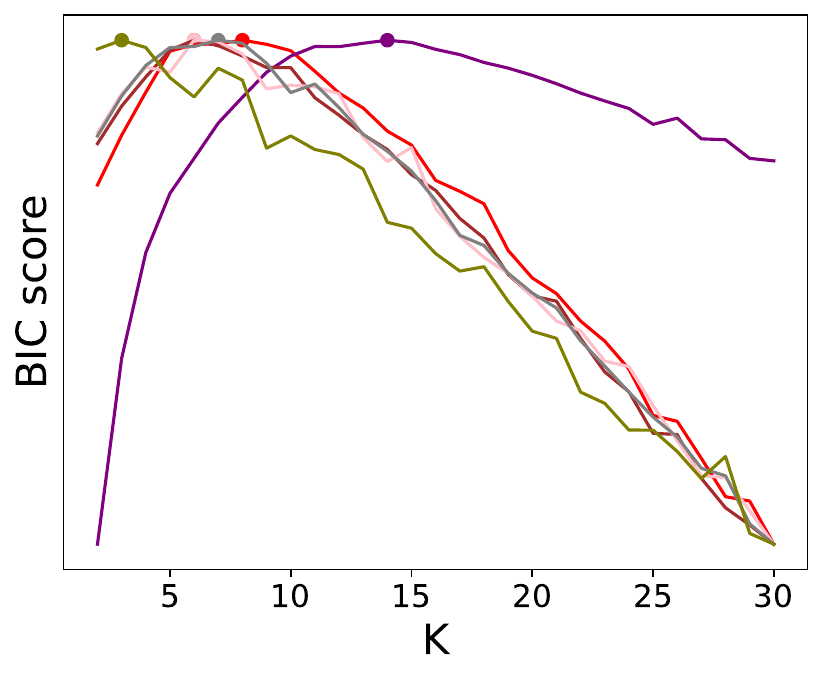}
        \caption{Loss (inertia), Silhouette score, and BIC score for the BBCSports data set ($K_{gt}=5$).}
    \end{subfigure}
    \begin{subfigure}{1\textwidth}
        \centering
        \includegraphics[width=0.8\textwidth]{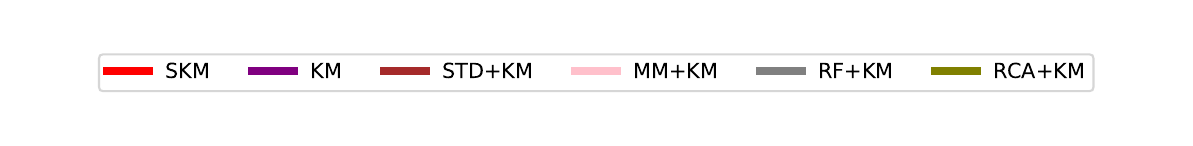}
    \end{subfigure}
    \caption{Loss (inertia), Silhouette score, and BIC score of the $k$-Means-based comparison algorithms for an increasing number of clusters $K$. The colored dots indicate the suggested number of clusters for each algorithm and evaluation method.}
    \label{fig:k_estimate_km}
\end{figure*}

\begin{table*}[t]
\centering
\caption{The estimated number of clusters $K$ and the corresponding ARI and Purity values obtained by different $K$-estimation strategies. TauCC intrinsically estimates $K$, \Method uses the method described in Sect.~\ref{sec:estimate_k}, and the $k$-Means-based approaches use the elbow method, Silhouette score (SIL) and BIC score. The result that best matches the ground truth is highlighted in \textbf{bold}.}
\label{tab:estimateK}
\resizebox{\textwidth}{!}{
\begin{tabular}{l|l|cc|ccc|ccc|ccc|ccc|ccc|ccc}
\toprule
\multirow{2}{*}{\textbf{Data set}} & \multirow{2}{*}{\textbf{Metric}} & \multirow{2}{*}{\Method} & TauCC & \multicolumn{3}{c|}{SKM} & \multicolumn{3}{c|}{KM} & \multicolumn{3}{c|}{STD+KM} & \multicolumn{3}{c|}{MM+KM} & \multicolumn{3}{c|}{RF+KM} & \multicolumn{3}{c}{RCA+KM}\\
& & & (mean) & Elbow & SIL & BIC & Elbow & SIL & BIC & Elbow & SIL & BIC & Elbow & SIL & BIC & Elbow & SIL & BIC & Elbow & SIL & BIC\\
\midrule
Synth & $K$ & \bm{$3$} & \bm{$3.0$} & $8$ & $2$ & $30$ & $7$ & $2$ & $29$ & $7$ & $2$ & $29$ & $7$ & $2$ & $30$ & $8$ & $2$ & $30$ & $8$ & $2$ & $29$ \\
($K_{gt} = 3$)& ARI & \bm{$95.5$} & $25.1$ & $53.1$ & $-0.1$ & $25.2$ & $16.7$ & $0.4$ & $10.2$ & $60.4$ & $0.0$ & $31.7$ & $63.3$ & $-0.1$ & $22.0$ & $50.6$ & $0.0$ & $23.1$ & $21.7$ & $0.0$ & $17.7$ \\
& Purity & \bm{$98.5$} & $59.8$ & $86.0$ & $34.1$ & $95.1$ & $64.8$ & $36.7$ & $88.3$ & $89.0$ & $34.2$ & $94.6$ & $88.1$ & $34.1$ & $94.3$ & $85.3$ & $34.1$ & $95.3$ & $57.4$ & $34.3$ & $83.1$ \\
\midrule
Wholesales & $K$ & $30$ & $3.0$ & $8$ & \bm{$2$} & $29$ & $8$ & $2$ & $30$ & $8$ & \bm{$2$} & $30$ & $9$ & \bm{$2$} & $30$ & $8$ & \bm{$2$} & $30$ & $7$ & $4$ & $30$ \\
($K_{gt} = 2$)& ARI & $3.9$ & $27.6$ & $12.7$ & $26.5$ & $5.1$ & $23.2$ & $-3.1$ & $7.3$ & $13.4$ & \bm{$28.2$} & $5.6$ & $13.3$ & $27.8$ & $7.9$ & $8.8$ & $21.3$ & $2.8$ & $15.3$ & $18.8$ & $3.9$ \\
& Purity & $88.6$ & $75.7$ & $82.3$ & $75.9$ & $86.6$ & $86.4$ & $67.7$ & \bm{$90.7$} & $80.7$ & $76.8$ & $84.5$ & $81.8$ & $76.6$ & $87.3$ & $80.9$ & $73.2$ & $83.6$ & $85.5$ & $84.3$ & $87.0$ \\
\midrule
SportA & $K$ & $19$ & $1.7$ & $8$ & \bm{$2$} & $23$ & $6$ & $3$ & $30$ & $9$ & \bm{$2$} & $22$ & $9$ & \bm{$2$} & $29$ & $9$ & \bm{$2$} & $26$ & $8$ & \bm{$2$} & $29$ \\
($K_{gt} = 2$)& ARI & $3.2$ & $15.8$ & $8.0$ & $29.5$ & $2.7$ & $14.5$ & $19.6$ & $6.3$ & $7.0$ & \bm{$31.9$} & $3.1$ & $8.0$ & $25.9$ & $2.6$ & $6.2$ & $7.5$ & $2.5$ & $4.8$ & $0.1$ & $6.0$ \\
& Purity & $79.3$ & $70.1$ & $80.1$ & $77.2$ & $78.7$ & $74.3$ & $72.1$ & \bm{$80.9$} & $78.3$ & $78.3$ & $78.8$ & $80.1$ & $75.6$ & $79.2$ & $78.2$ & $64.3$ & $78.9$ & $0.66$ & $63.6$ & $75.9$ \\
\midrule
Optdigits & $K$ & $30$ & $2.1$ & $9$ & $9$ & $30$ & \bm{$10$} & \bm{$10$} & $30$ & $9$ & $9$ & $30$ & \bm{$10$} & \bm{$10$} & $30$ & $9$ & $9$ & $30$ & $8$ & $2$ & $30$ \\
($K_{gt} = 10$) & ARI & $48.0$ & $10.6$ & $60.1$ & $60.1$ & $43.6$ & \bm{$67.1$} & \bm{$67.1$} & $43.9$ & $60.3$ & $60.3$ & $43.1$ & $67.0$ & $67.0$ & $50.8$ & $59.6$ & $59.6$ & $44.8$ & $0.0$ & $0.0$ & $11.8$ \\
& Purity & $91.4$ & $20.3$ & $73.4$ & $73.4$ & $91.7$ & $80.3$ & $80.3$ & $94.0$ & $73.4$ & $73.4$ & $92.2$ & $80.3$ & $80.3$ & \bm{$94.3$} & $73.2$ & $73.2$ & $92.0$ & $10.7$ & $10.2$ & $30.9$ \\
\midrule
BBCSports & $K$ & \bm{$5$} & $3.0$ & $10$ & $2$ & $8$ & $9$ & $2$ & $14$ & $10$ & $2$ & $6$ & $12$ & $2$ & $6$ & $11$ & $2$ & $7$ & $8$ & $2$ & $3$ \\
($K_{gt} = 5$) & ARI & \bm{$92.0$} & $26.2$ & $17.3$ & $5.8$ & $22.4$ & $0.3$ & $0.0$ & $0.7$ & $14.1$ & $4.2$ & $11.3$ & $8.9$ & $1.6$ & $7.1$ & $11.1$ & $5.6$ & $15.3$ & $0.8$ & $0.3$ & $1.2$ \\
& Purity & \bm{$97.2$} & $51.6$ & $64.2$ & $36.0$ & $66.8$ & $37.7$ & $36.0$ & $41.2$ & $60.0$ & $36.0$ & $47.2$ & $50.2$ & $36.0$ & $47.1$ & $54.5$ & $36.0$ & $57.8$ & $39.5$ & $36.5$ & $38.7$ \\
\midrule
BBCNews & $K$ & \bm{$6$} & $3.4$ & $14$ & $2$ & $18$ & \bm{$6$} & $2$ & $29$ & $12$ & $2$ & $14$ & $15$ & $2$ & $15$ & $11$ & $2$ & $12$ & $15$ & $2$ & $17$ \\
($K_{gt} = 5$) & ARI & \bm{$82.3$} & $41.8$ & $20.4$ & $4.6$ & $17.9$ & $6.3$ & $3.9$ & $11.5$ & $19.7$ & $3.3$ & $19.6$ & $14.8$ & $2.1$ & $14.8$ & $20.4$ & $2.3$ & $19.6$ & $0.4$ & $-0.1$ & $1.4$ \\
& Purity & \bm{$96.0$} & $60.5$ & $69.8$ & $26.8$ & $71.0$ & $35.2$ & $28.2$ & $55.6$ & $66.9$ & $26.3$ & $66.9$ & $59.3$ & $26.0$ & $59.3$ & $63.0$ & $27.0$ & $61.5$ & $30.5$ & $23.6$ & $30.7$ \\
\midrule
WebKB & $K$ & $9$ & $2.9$ & $9$ & $2$ & $30$ & \bm{$7$} & $2$ & $30$ & $9$ & $2$ & $27$ & $9$ & $2$ & $26$ & $12$ & $2$ & $30$ & $4$ & $2$ & $28$ \\
($K_{gt} = 6$) & ARI & \bm{$29.1$} & $11.3$ & $13.7$ & $0.1$ & $8.6$ & $3.9$ & $0.0$ & $5.8$ & $13.3$ & $0.7$ & $8.1$ & $10.3$ & $-0.4$ & $7.3$ & $10.6$ & $-0.8$ & $10.4$ & $0.1$ & $0.2$ & $-1.7$ \\
& Purity & \bm{$71.0$} & $48.3$ & $59.5$ & $36.3$ & $65.8$ & $39.5$ & $36.3$ & $44.8$ & $59.2$ & $36.3$ & $63.9$ & $54.7$ & $36.3$ & $60.0$ & $55.3$ & $36.3$ & $64.9$ & $36.9$ & $36.9$ & $39.3$ \\
\midrule
Reuters & $K$ & $18$ & $2.0$ & $8$ & $2$ & $30$ & $8$ & $2$ & $30$ & $8$ & $2$ & $30$ & $11$ & $2$ & $29$ & $9$ & $3$ & $29$ & \bm{$6$} & $2$ & $28$ \\
($K_{gt} = 5$) & ARI & $25.6$ & \bm{$43.2$} & $18.9$ & $27.2$ & $11.0$ & $16.3$ & $9.2$ & $12.0$ & $17.8$ & $26.4$ & $11.3$ & $14.7$ & $27.0$ & $10.5$ & $11.2$ & $15.2$ & $5.9$ & $0.2$ & $0.2$ & $0.3$ \\
& Purity & \bm{$96.7$} & $68.8$ & $68.9$ & $64.8$ & $85.1$ & $60.2$ & $48.9$ & $70.0$ & $68.0$ & $64.5$ & $85.8$ & $72.3$ & $64.5$ & $80.8$ & $65.0$ & $63.8$ & $76.9$ & $47.6$ & $47.4$ & $48.4$ \\
\midrule
20NewsG & $K$ & \bm{$22$} & $9.5$ & $11$ & $2$ & $30$ & $6$ & $2$ & $30$ & $11$ & $2$ & $30$ & $12$ & $2$ & $26$ & $11$ & $2$ & $30$ & $5$ & $2$ & $30$ \\
($K_{gt} = 20$) & ARI & \bm{$22.9$} & $2.1$ & $1.5$ & $0.4$ & $2.0$ & $0.0$ & $0.0$ & $0.5$ & $1.8$ & $0.6$ & $2.0$ & $1.2$ & $0.3$ & $1.7$ & $0.8$ & $0.4$ & $1.5$ & $0.0$ & $0.0$ & $0.0$ \\
& Purity & \bm{$41.7$} & $8.7$ & $11.4$ & $7.5$ & $13.8$ & $5.9$ & $5.3$ & $8.3$ & $11.2$ & $8.3$ & $13.6$ & $10.8$ & $7.7$ & $12.7$ & $9.7$ & $7.4$ & $11.3$ & $5.3$ & $5.3$ & $5.6$ \\
\midrule
MouseAtlas & $K$ & $26$ & $3.6$ & $9$ & $2$ & $30$ & $9$ & $2$ & $30$ & \bm{$11$} & $2$ & $29$ & $10$ & $2$ & $30$ & $10$ & $2$ & $30$ & $10$ & $2$ & $30$ \\
($K_{gt} = 11$) & ARI & \bm{$55.0$} & $34.4$ & $39.0$ & $3.2$ & $28.6$ & $1.5$ & $0.7$ & $1.2$ & $37.1$ & $17.5$ & $31.2$ & $34.9$ & $3.8$ & $26.9$ & $24.5$ & $2.5$ & $28.0$ & $0.0$ & $0.0$ & $0.4$ \\
& Purity & $65.3$ & $50.4 $ & $68.0$ & $25.2$ & \bm{$82.5$} & $34.4$ & $21.9$ & $38.0$ & $68.0$ & $43.8$ & $81.8$ & $63.3$ & $27.7$ & $77.6$ & $56.8$ & $25.1$ & $73.8$ & $20.7$ & $20.4$ & $22.0$ \\
\midrule
GeneExp & $K$ & \bm{$6$} & $3.0$ & \bm{$6$} & \bm{$6$} & $13$ & $7$ & \bm{$6$} & $12$ & \bm{$6$} & \bm{$6$} & $12$ & $7$ & \bm{$6$} & $12$ & $7$ & \bm{$6$} & $12$ & $12$ & $2$ & $22$ \\
($K_{gt} = 5$) & ARI & \bm{$87.9$} & $53.1$ & $87.6$ & $87.6$ & $45.9$ & $71.4$ & $87.6$ & $48.7$ & $87.8$ & $87.8$ & $48.7$ & $73.2$ & $87.1$ & $44.6$ & $71.2$ & $87.6$ & $48.4$ & $0.2$ & $0.1$ & $0.7$ \\
& Purity & \bm{$99.6$} & $67.2$ & $99.3$ & $99.3$ & $99.5$ & $99.3$ & $99.3$ & $99.5$ & $99.4$ & $99.4$ & $99.6$ & $99.4$ & $99.4$ & $99.4$ & $99.1$ & $99.3$ & $99.6$ & $38.5$ & $37.7$ & $39.7$ \\
\midrule
HDendritic & $K$ & $29$ & $2.7$ & $6$ & \bm{$4$} & $29$ & $11$ & $2$ & $30$ & $7$ & $3$ & $29$ & $7$ & $5$ & $23$ & $11$ & $2$ & $30$ & $3$ & $2$ & $2$ \\
($K_{gt} = 4$) & ARI & $74.9$ & $42.3$ & $54.9$ & $80.1$ & $28.8$ & \bm{$82.1$} & $-0.1$ & $38.5$ & $55.7$ & $68.9$ & $25.3$ & $49.8$ & $61.4$ & $24.3$ & \bm{$82.1$} & $-0.1$ & $38.5$ & $-0.1$ & $-0.1$ & $-0.1$ \\
& Purity & \bm{$95.8$} & $64.3$ & $92.7$ & $92.0$ & \bm{$95.8$} & $93.2$ & $33.5$ & $95.5$ & $93.6$ & $77.4$ & $95.1$ & $89.4$ & $88.5$ & $94.3$ & $93.2$ & $33.5$ & $95.5$ & $33.5$ & $33.5$ & $33.5$ \\
\bottomrule
\end{tabular}
}
\end{table*}

\end{appendices}

\end{document}